\documentclass{article}

\PassOptionsToPackage{numbers, compress}{natbib}

\usepackage{microtype}
\usepackage{graphicx}
\usepackage{wrapfig}
\usepackage{booktabs}
\usepackage{amsmath}
\usepackage{amsfonts}
\usepackage{amsthm, relsize}
\usepackage{thm-restate}
\usepackage[font=small,labelfont=bf]{caption}
\usepackage{epstopdf}
\usepackage{courier}
\usepackage{placeins,makecell,authblk}
\usepackage{pifont,nccmath,bbm,titlesec}
\newcommand{\cmark}{\ding{51}}
\newcommand{\xmark}{\ding{55}}

\makeatletter
\newcommand*{\centerfloat}{%
  \parindent \z@
  \leftskip \z@ \@plus 1fil \@minus \textwidth
  \rightskip\leftskip
  \parfillskip \z@skip}
\makeatother

\usepackage[margin=1in]{geometry}
    
\usepackage{multirow}
\usepackage{color}
\newcommand{\R}{\mathbb{R}}
\newcommand{\ep}{\epsilon}

\newcommand{\TV}{\textnormal{TV}}
\newcommand{\X}{\mathcal{X}}

\newcommand{\f}[2]{\dfrac{#1}{#2}}
\newcommand{\ff}[2]{\tfrac{#1}{#2}}
\newcommand{\de}{\delta}
\newcommand{\T}{\theta}

\newcommand{\su}[2]{\mathlarger{\sum\limits_{#1}^{#2}}}

\newcommand{\p}{\partial}
\newcommand{\lt}{\left(}
\newcommand{\rt}{\right)}
\newcommand{\Lt}{\left[}
\newcommand{\Rt}{\right]}
\newcommand{\A}{\alpha}

\newcommand{\F}{\mathcal{F}}

\newcommand{\I}[2]{\mathlarger{\int_{#1}^{#2}}}
\newcommand{\G}{\nabla}

\renewcommand{\P}{\mathcal{P}}
\newcommand{\N}{\mathcal{N}}
\newcommand{\ra}{\rightarrow}

\newcommand{\1}{\mathbf{1}}
\providecommand{\norm}[1]{\left\lVert#1\right\rVert}

\newcommand{\defeq}{:=}

\newcommand{\E}{\mathbb{E}}
\DeclareMathOperator*{\argmin}{argmin}
\DeclareMathOperator*{\argmax}{argmax}

\newcommand{\iter}[1]{^{(#1)}}

\newtheorem{lem}{Lemma}

\newtheorem{example}{Example}[section]
\usepackage[ruled]{algorithm2e}
\newcommand{\unaryminus}{\scalebox{0.75}[1.0]{\( - \)}}
\newcommand{\citen}[1]{[\citenum{#1}]}
\newcommand{\dx}[0]{\mathrm{d} x}

\usepackage{subcaption}
\usepackage{hyperref}

\usepackage{cleveref}
\crefname{appsec}{Appendix}{Appendices}
\newcommand{\name}{{\textsc{George}}}

\usepackage[utf8]{inputenc}
\usepackage[T1]{fontenc}
\usepackage{hyperref}
\usepackage{url}
\usepackage{booktabs}
\usepackage{amsfonts}
\usepackage{nicefrac}
\usepackage{microtype}
\usepackage{natbib}

\title{No Subclass Left Behind: Fine-Grained Robustness in \linebreak Coarse-Grained Classification Problems}

\author{%
  Nimit S. Sohoni, Jared A. Dunnmon, Geoffrey Angus, Albert Gu, Christopher R\'e \\
  \emph{Stanford University} \\
  \small nims@stanford.edu, jdunnmon@cs.stanford.edu, gdlangus@cs.stanford.edu, \\
  \small albertgu@stanford.edu, chrismre@cs.stanford.edu
}
\begin{document}

\maketitle

\begin{abstract}
In real-world classification tasks, each class often comprises multiple finer-grained ``subclasses.''
As the subclass labels are frequently unavailable,
models trained using only the coarser-grained class labels often
exhibit highly variable performance across different subclasses.
This phenomenon, known as \emph{hidden stratification}, has
important consequences for models deployed in safety-critical applications such as medicine.
We propose \name, a method to both measure and mitigate hidden stratification
\textit{even when subclass labels are unknown}.
We first observe that unlabeled subclasses are often separable in the feature space of deep neural networks,
and exploit this fact to estimate subclass labels for the training data via clustering techniques.
We then use these approximate subclass labels as a form of noisy supervision
in a distributionally robust optimization objective.
We theoretically characterize the performance of \name~in terms of the worst-case generalization error across any subclass.
We empirically validate \name~on a mix of real-world and benchmark image classification datasets, and
show that our approach boosts worst-case subclass accuracy by up to 22 percentage
points compared to standard
training techniques, without requiring any prior information about the subclasses.
\end{abstract}

\section{Introduction}
In many real-world classification tasks, each labeled class consists of multiple semantically distinct \emph{subclasses} that are unlabeled.
Because models are typically trained to maximize \emph{global} metrics such as average performance,
they often underperform on important subclasses~\citen{yao2011combining, recht2019imagenet}.
This phenomenon---recently termed \textit{hidden stratification}---can
lead to skewed assessments of model quality and result in unexpectedly poor performance when models are deployed~\citen{oakden2019hidden}.
For instance, a medical imaging model trained to classify between benign and abnormal lesions
may achieve high overall performance, yet consistently mislabel a rare but critical abnormal subclass as ``benign'' \citen{Dunnmon2019-rr}.
As another example, a well-known model for classifying chest radiographs was
shown to perform substantially worse at recognizing pneumothorax (collapsed lung)
on the subclass of pneumothorax images without a chest drain--which is worrisome since
chest drains are the common form of \emph{treatment} for the condition,
so the drain-free subclass is in fact the clinically important one \citen{oakden2019hidden}.

Modern robust optimization techniques can improve performance on poorly-performing groups
when the group identities are known \citen{sagawa2020distributionally}. However,
in practice, a key obstacle is that \textit{subclasses are often unlabeled}, or even unidentified.
This makes even detecting such performance gaps---let alone mitigating them---a challenging problem.
Nevertheless, recent empirical evidence~\citen{oakden2019hidden} encouragingly suggests that feature
representations of deep neural networks
often carry information about unlabeled subclasses (e.g., \Cref{fig:isic-clusters}).
Motivated by this observation, we propose a method for addressing hidden stratification, by
 both measuring and improving worst-case
subclass performance in the setting where subclass labels are unavailable.
Our work towards this is organized into four main sections.

\begin{wrapfigure}{R}{0.45\textwidth}
\vspace{-\intextsep}
\begin{center}
\includegraphics[height=1.5in]{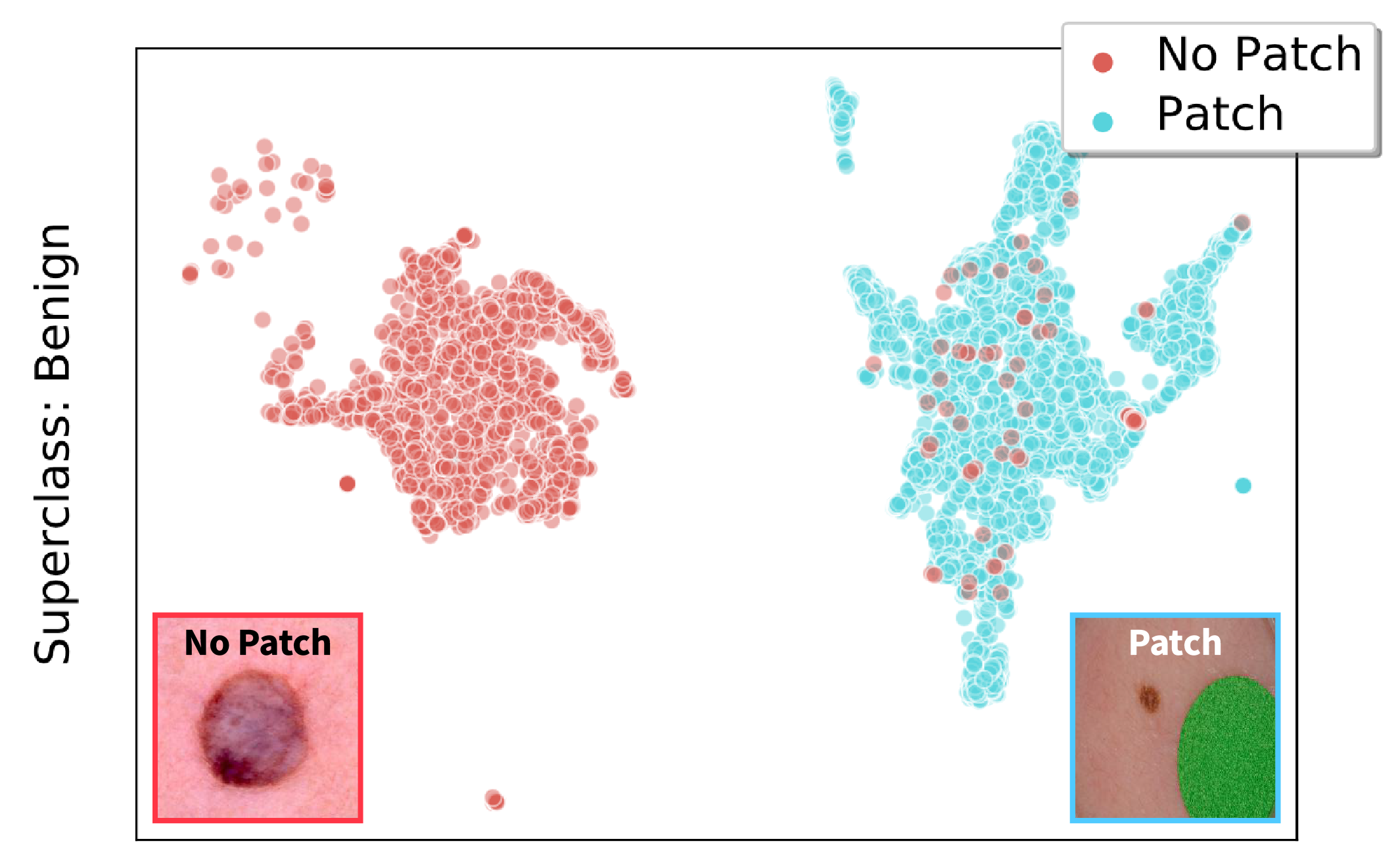}
\caption{Benign class examples in the feature space of a model classifying
skin lesions as benign or malignant. Benign examples containing a
brightly colored patch (blue) and those without a patch (red)
are separable in model feature space,
even though the labels do not specify the presence of patches.}
\label{fig:isic-clusters}
\end{center}
\vspace{-1.5em}
\end{wrapfigure}

First, in Section~\ref{sec:sec2} we propose a simple generative model of the data labeling process.
Using this model, we show that when label annotations are insufficiently fine-grained---as is often the case in real-world datasets---hidden stratification can naturally arise.
For instance, an image classification task might be to classify birds vs. frogs; if labels
are only provided for these broad classes, they may fail to capture visually meaningful finer-grained, intra-class variation
(e.g., ``bird in flight'' versus ``bird in nest'').
We show that in the setting of our generative model, standard training via empirical risk minimization (ERM)
can result in arbitrarily poor performance on underrepresented subclasses.

Second, in Section~\ref{sec:alg-framework} we use insights from this generative model to motivate
\name, a two-step procedure for alleviating hidden stratification by first \emph{estimating} the subclass labels
and then \emph{exploiting} these estimates to train a robust classifier.
To {estimate} subclass labels, we train a standard model on the task, and split each class (or ``superclass,''
for clarity) into estimated subclasses via unsupervised clustering in the model's feature space.
We then {exploit} these estimated subclasses
by training a new model to optimize \emph{worst-case} performance over all estimated subclasses using group distributionally robust optimization (GDRO) \citen{sagawa2020distributionally}.
In this way, our framework allows ML practitioners to automatically
detect poorly-performing subclasses and improve performance on them, without needing to resort to expensive manual relabeling of the data.

Third, in Section~\ref{sec:analysis}
we use our generative framework to prove that---under sufficiently strong conditions on the data distribution
and the quality of the recovered clusters---\name~can reduce the subclass performance gap,
attaining the same asymptotic sample complexity rates as if the true subclass labels were known.

Fourth, in Section~\ref{sec:experiments} we empirically validate the ability of \name~to both \emph{measure} and \emph{mitigate} hidden stratification on
four image classification tasks, comprising
both robustness benchmarks and real-world datasets.
We demonstrate that the first step of \name---training an ERM model and clustering the superclass features---often recovers clusters that align closely with true subclasses.
We evaluate the ability of these clusters to measure the worst-case subclass (i.e., ``robust'') performance:
on average, the gap between worst-case cluster performance and worst-case subclass performance is less than
half the gap between overall and worst-case subclass performance,
indicating that \name~enables more accurate measurement of robust performance.
Next, we show that the second stage of \name---retraining a robust model using cluster assignments as proxy subclass labels---reduces average
worst-case subclass error rates by 22\% on these datasets. 
For comparison, the state-of-the-art ``oracle'' GDRO method that \emph{does} require subclass labels \citep{sagawa2020distributionally}
reduces average worst-case subclass error rates by 51\%.
As an extension, we show that leveraging recent pretrained image embeddings \cite{koleskinov2020bit} for clustering
can substantially further improve the robust performance of \name, in some cases
to nearly match the performance of GDRO trained using the true subclass labels.\footnote{Code for \name\ can be found at \url{https://github.com/HazyResearch/hidden-stratification}.}

\section{Background}
\subsection{Related Work}
\label{sec:related}
Our work builds upon prior work from three main areas: robust optimization,
representation learning, and unsupervised clustering. We provide a more extensive
discussion of related work in Appendix \ref{app:related}.

\textbf{Distributionally Robust Optimization.}
Robustness and fairness is an active research area in machine learning \citen{barocas-hardt-narayanan, hardt2016equality, lipton2018does, kearns2019average}.
\emph{Distributionally robust optimization} (DRO) attempts to guarantee
good performance in the presence of distribution shift, e.g., from adversarial perturbations
\cite{staib2017distributionally, sinha2018certifying} or evaluation on arbitrary subpopulations \citep{duchidistributionally}. 
Because these notions of robustness can be
pessimistic \citen{hu2018does}, others investigate \emph{group}
DRO (GDRO), which optimizes
worst-case performance over a known set of groups (partitions) of the data \citen{hu2018does, sagawa2020distributionally}.
A major obstacle to applying GDRO methods in practice is that group labels are often unavailable;
in our work, we aim to address this issue in the classification setting.

\textbf{Representation Learning \& Clustering.}
Our approach relies on estimating unknown subclass labels by
clustering a feature representation of the data. Techniques for learning
semantically useful image features include 
autoencoder-based methods \citen{mcconville2019n2d, shukla2018semi},
the use of unsupervised auxiliary tasks \citen{asano2020critical, chen2020simple},
and pretraining on massive datasets \citen{koleskinov2020bit}.
Such features may be used for unsupervised identification of classes,
either using clustering techniques \citen{caron2018deep} or an end-to-end approach \citen{ji2019invariant, gansbeke2020learning}.
It has also been observed that when a model is trained on coarse-grained class labels,
the data \emph{within each class} can often be separated into distinct clusters in model feature space 
(e.g., \citen{oakden2019hidden}). While we primarily focus on the latter
approach, we also evaluate the utility of pretrained embeddings as a source of features for clustering.

\subsection{Problem Setup}
\label{sec:setup}
We are given $n$ datapoints $x_1, \dots, x_n \in \X$ and associated \emph{superclass} labels $y_1, \dots, y_n \in \{1,\dots,B\}$.\footnote{We assume $B > 1$ since otherwise the ``classification problem'' is trivial.} In addition, each datapoint $x_i$ is associated with a latent (unobserved) \emph{subclass} label $z_i \in \{1,\dots,C\}$. We assume that $\{1,\dots,C\}$ is partitioned into disjoint nonempty sets $S_1,\dots,S_B$ such that if $z_i \in S_b$, then $y_i = b$; in other words, the subclass label $z_i$ determines the superclass label $y_i$. Let $S_b$ denote the set of all subclasses comprising superclass $b$, and $S(c)$ denote the superclass corresponding to subclass $c$.

Our goal is to classify examples from $\X$ into their correct \emph{superclass}.
Given a function class $\mathcal{F}$, it is typical to seek a classifier $f \in \mathcal{F}$ that maximizes overall population accuracy:
\begin{equation}
\label{eq:erm}
\argmax\limits_{f \in \mathcal{F}} \mathbb{E}_{(x,y)}\left[\1(f(x) = y)\right].
\end{equation}
By contrast, we seek to maximize the \emph{robust accuracy},
defined as the \emph{worst-case} expected accuracy over all subclasses:
\begin{equation}
\label{eq:robust}
\argmax\limits_{f \in \mathcal{F}} \min\limits_{c \in \{1,\dots,C\}} \mathbb{E}_{(x,y) | z = c}\left[\1(f(x) = y) \right].
\end{equation}
Note that $y$ is fixed conditional on the value of $z$. As we cannot directly optimize the population accuracy, we select a surrogate loss function $\ell$ and attempt to minimize this loss over the training data.
For instance, the standard ERM approach to approximate \eqref{eq:erm} minimizes the empirical risk (i.e., training loss) $R(f)$:
\begin{equation}
\argmin\limits_{f \in \mathcal{F}} \left\{ {R}(f) \defeq \tfrac{1}{n}\textstyle\sum\limits_{i=1}^n \ell(f(x_i), y_i) \right\}.
\end{equation}
To approximate \eqref{eq:robust}, if we knew $z_1, ..., z_n$ we could minimize the \emph{worst-case per-subclass training risk} by solving:
{\small
\begin{equation}
\label{eq:robust_obj}
\argmin\limits_{f \in \mathcal{F}} \left\{ {R}_{\text{robust}}(f) \defeq \max\limits_{c \in \{1,\dots,C\}} \ff{1}{n_c} \textstyle\sum\limits_{i=1}^n \mathbf{1}(z_i = c) \ell(f(x_i), y_i) \right\},
\end{equation}}%
where $n_c = \sum_{i=1}^n \mathbf{1}(z_i = c)$ is the number of training examples from subclass $c$.
 ${R}_{\text{robust}}(f)$ is the ``robust loss'' achieved by $f$.
Our goal is to learn a model $\tilde{f} \in \mathcal{F}$ such that
$
{R}_{\text{robust}}(\tilde{f}) - \min\limits_{f \in \mathcal{F}} \lt {R}_{\text{robust}}(f) \rt
$ 
is small with high probability.
When the $z_i$'s are known, Eq.~\eqref{eq:robust_obj} can be tractably optimized using GDRO~\citen{hu2018does, sagawa2020distributionally}.
However, we do \emph{not} assume access to the $z_i$'s;
we seek to approximately minimize $R_{\text{robust}}$ {without knowledge of the subclass labels}.

\section{Modeling Hidden Stratification}
\label{sec:sec2}
In Section~\ref{sec:generative-model}, we introduce a generative model of the data labeling process.
In Section~\ref{sec:erm-failure}, we use this model to explain how hidden stratification can occur,
and show that in the setting of this model ERM can attain arbitrarily poor robust risk compared to GDRO.

\subsection{A Model of the Data Generating and Labeling Process}
\label{sec:generative-model}
In real datasets, individual datapoints are typically described by multiple different attributes,
yet often only a subset of these are captured by the class labels. For example,
a dataset might consist of images labeled ``cat'' or ``dog.''
These coarse class labels may not capture other salient attributes (color, size, breed, etc.);
these attributes can be interpreted as latent variables representing different subclasses.

We model this phenomenon with a hierarchical data generation process.
First, a binary vector $\vec{Z}\,{\in}\,\{\unaryminus 1, +1\}^k$ is sampled from a distribution $\P(\vec{Z})$.\footnote{In this paper, we use $\mathcal{P}$ to denote a distribution and $p$ to denote its density.}
Each entry $Z_i$ is an attribute, while each unique value of $\vec{Z}$ represents a different subclass.
Then, a latent ``feature vector'' $\vec{V} \in \R^k$ is sampled from a distribution conditioned on $\vec{Z}$.
Specifically, when conditioned on $Z_i$, each individual feature $V_i$ is Gaussian and independent of the $Z_j$'s with
$j > i$.
Finally, the datapoint $X \in \mathcal{X}$ is determined by the latent features
$\vec{V}$ via a fixed map $g : \R^k \to \mathcal{X}$.
Meanwhile, the superclass label $Y$ is equal to $h(\vec{Z})$, where $h$ is a fixed discrete-valued function.
In particular, $h$ may only depend on a subset of the $Z_i$ attributes;
the $Z_i$'s which do not influence the label $Y$
correspond to hidden subclasses.
$X, Y$ are observed, while $\vec{V}, \vec{Z}$ are not.
Figure \ref{fig:generative}a illustrates this generative process;
Figure \ref{fig:generative}b presents an analogue on the Waterbirds~dataset \citen{sagawa2020distributionally}.

A key assumption is that the subclasses are ``meaningful'' in some sense, rather than just arbitrary groups of datapoints.
Thus, rather than attempting to enforce good performance on \emph{all possible subsets} of the data,
we assume some meaningful structure on the subclass data distributions.
We model this via the Gaussian assumption on $\P(V_i | \vec{Z})$,
which is similar to that often made for the latent space of GANs \citen{bojanowski2018optimizing}. 
Consequently, the data distribution is a mixture of Gaussians in the ``canonical feature space'' $\vec{V}$,
which
facilitates further theoretical analysis (Section~\ref{sec:analysis}).
Our generative model also bears similarity to that of \citen{hu2018does},
who use a hierarchical data-generation model to analyze the behavior of DRO methods
in the presence of distribution shift.

\begin{figure*}[!t!]
\centerfloat
\begin{subfigure}{0.3\linewidth}
\centering
\includegraphics[height=1 in]{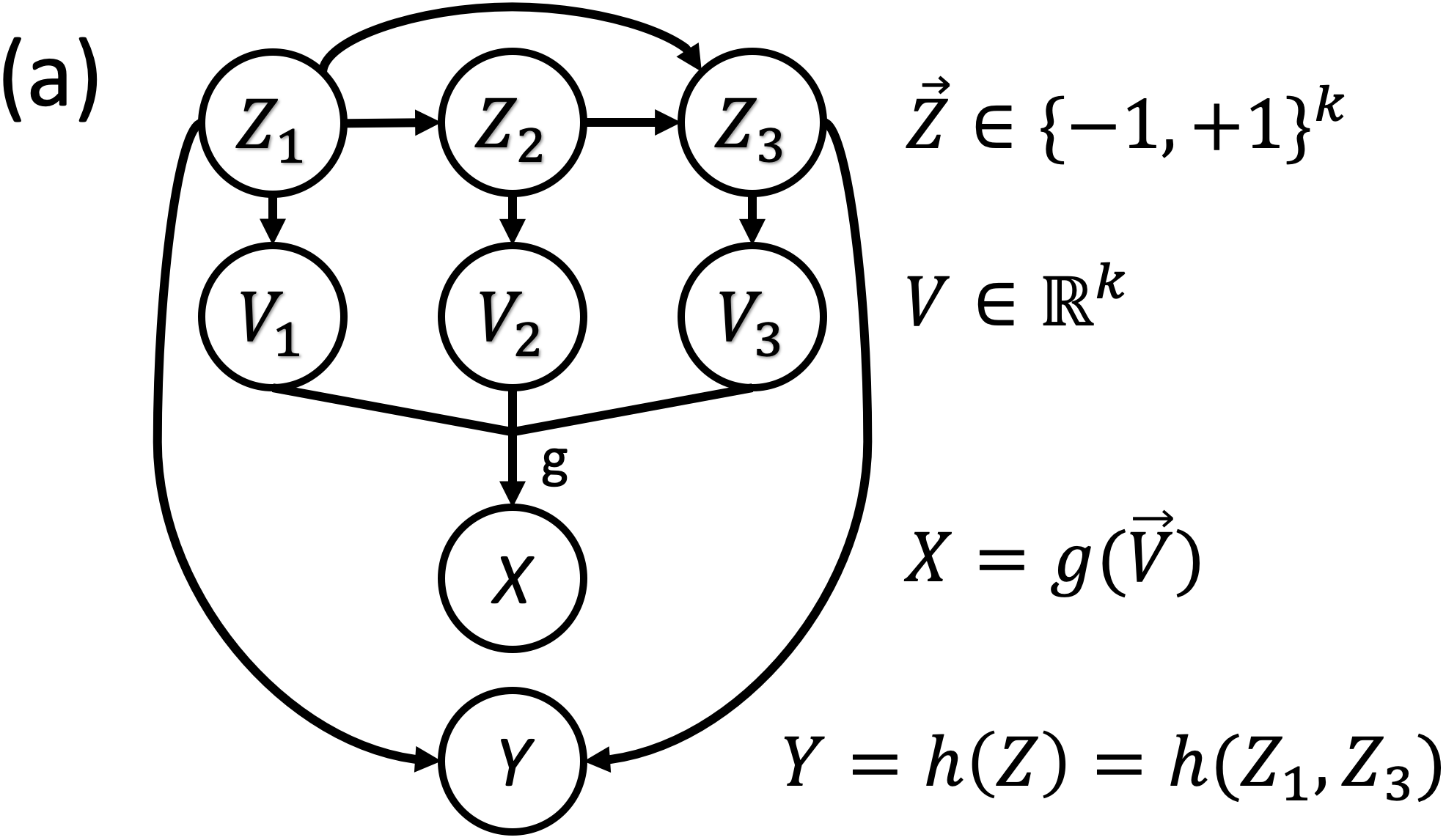}
\end{subfigure}%
\quad
\begin{subfigure}{0.65\linewidth}
\centering
\includegraphics[trim=0mm 5mm 0mm 0mm clip, height=0.95 in]{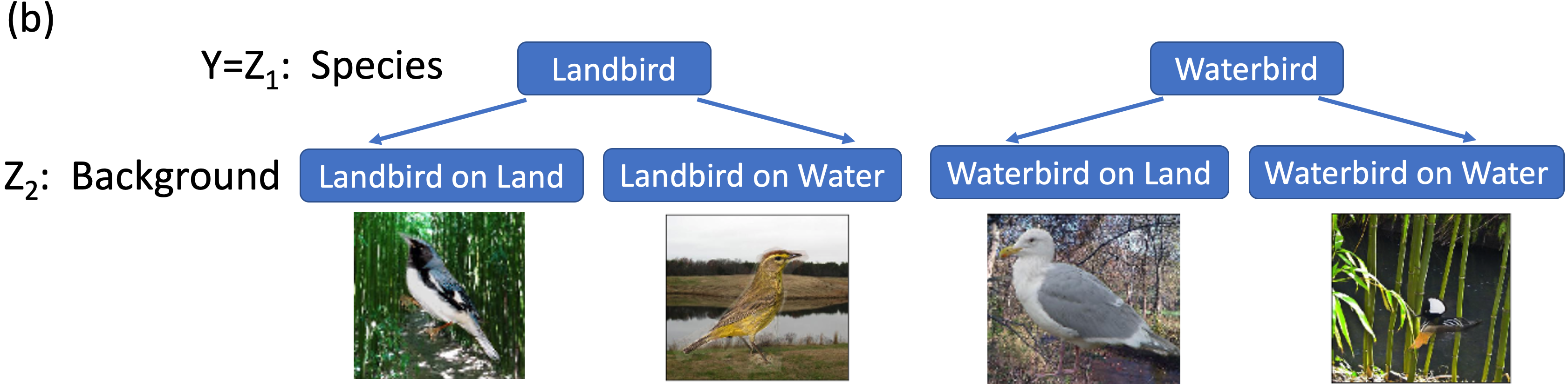}
\end{subfigure}%
\caption{(a) Generative model of hidden stratification: attributes $Z$ determine features $\vec{V}$ and labels $Y$; mapping $g$ transforms $\vec{V}$ to yield observed data $X$. (b) On the Waterbirds dataset \citen{sagawa2020distributionally},
attributes $(Z_1,Z_2)$ denote species and background type respectively; the label $Y$ is the species type.}
\label{fig:generative}
\end{figure*}

\subsection{What Causes Hidden Stratification, and When Can It Be Fixed?}
\label{sec:erm-failure}
We now use our generative model to help understand why hidden stratification
can occur, and present a simple example in which ERM is provably suboptimal
in terms of the robust risk.

\begin{wrapfigure}{R}{0.45\textwidth} 
\vspace{-\intextsep}
\begin{center}
  \includegraphics[height=1.4in]{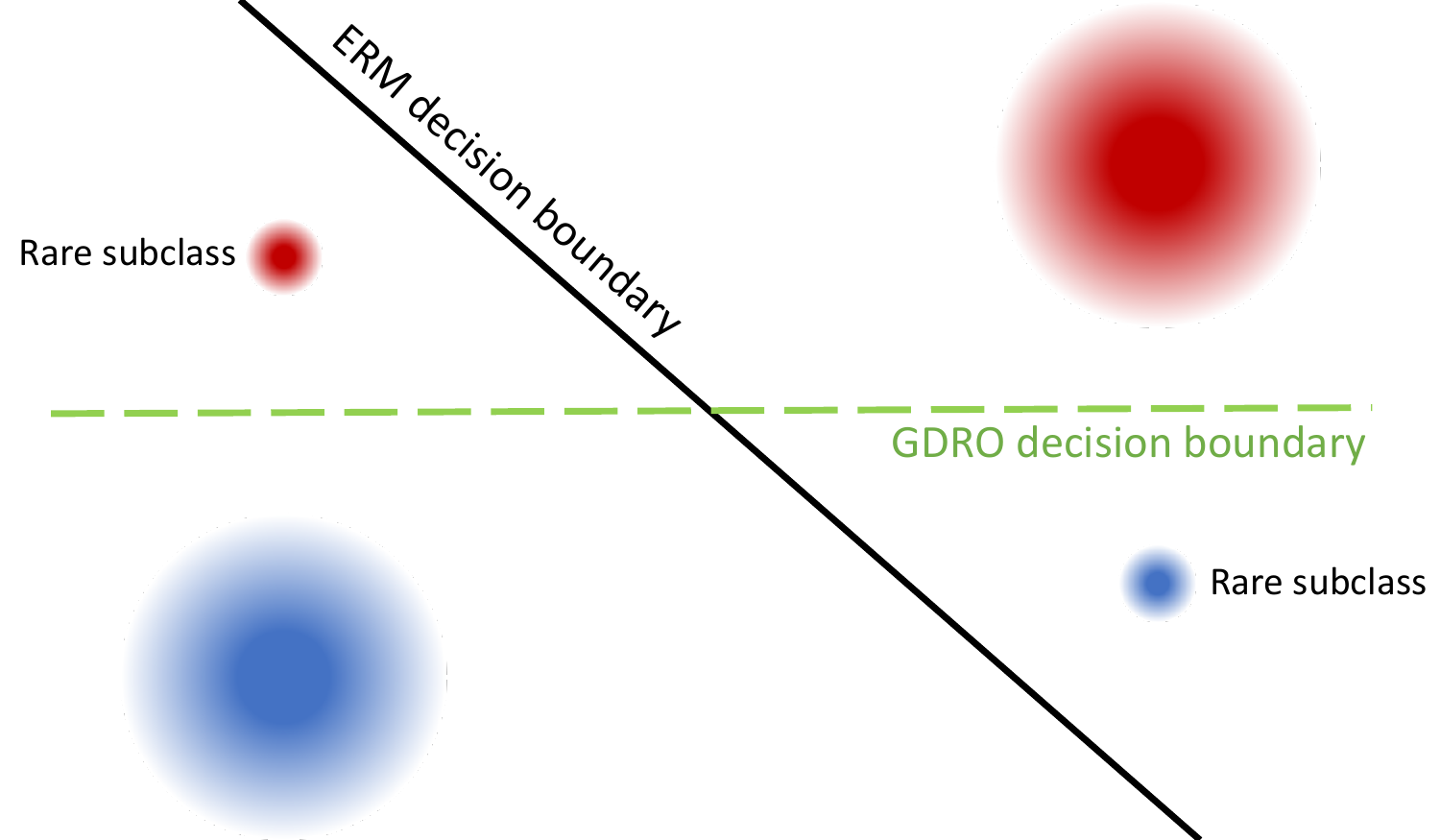}
\caption{
As $\A \downarrow 0$ in Example \ref{example:erm_fail},
top-left \& lower-right subclasses get rarer
and are misclassified by ERM (black boundary),
whereas GDRO learns the optimal robust boundary (green)
to classify red vs. blue superclasses.}
\label{fig:motivation}
\end{center}
\vspace{-3em}
\end{wrapfigure}

We distinguish between two main causes of hidden stratification: \emph{inherent hardness}
and \emph{dataset imbalance}. First,
certain subclasses are ``inherently harder'' to classify because they are more similar
 to other superclasses.
We define the inherent hardness of a task
as the minimum attainable robust error;
inherent hardness thus lower bounds the worst-case subclass error of any model.
See Appendix \ref{app:theory} for more discussion.

Second, imbalance in subclass sizes can cause ERM to underserve rare subclasses,
since it optimizes for average-case performance.
We provide a simple concrete example (\ref{example:erm_fail}) below.
Unlike inherent hardness, robust performance gaps arising from
dataset imbalances \emph{can be resolved} if subclass labels are known,
by using these labels to minimize the objective in Equation \eqref{eq:robust_obj} via GDRO.

\begin{example}
\label{example:erm_fail}
Figure \ref{fig:motivation} depicts an example distribution generated by the model in Section \ref{sec:generative-model}.
In this example, the binary attribute vector $\vec{Z}$ has dimension 2, i.e., $\vec{Z}= (Z_1,Z_2)$, while only $Z_2$ determines the superclass label $Y$, i.e., $Y = Z_2$.
The latent attribute $Z_1$ induces two subclasses in each superclass,
each distributed as a different Gaussian in feature space,
with mixture proportions $\A$ and $1-\A$ respectively.
For linear models with regularized logistic loss, there exists a family of distributions of this form
such that as the proportion $\A$ of the rare subclasses goes to $0$,
the worst-case subclass accuracy of ERM is only $O(\A)$, while that of GDRO is $1-O(\A)$.
(See Appendix \ref{app:prop1} for the specific parameters of the per-subclass distributions in this example
and a proof of the claim.)
\end{example}

Example \ref{example:erm_fail} illustrates that when the dataset is imbalanced---i.e., the distribution of the
underlying attributes $\vec{Z}$ is highly nonuniform---knowledge of subclass labels
 can improve robust performance. We thus ask:
\emph{how well can we estimate subclass labels if they are not provided}?
In the extreme, if two subclasses of a superclass have the same distribution in feature space,
we cannot distinguish them.
However, the model must then perform the same on each subclass, since its prediction
is a fixed function of the features!
Conversely, if one subclass has higher average error,
it must lie ``further across'' the decision boundary,
meaning that the two subclasses must be separable to some degree;
the larger the accuracy gap, the more separable the subclasses are.
We formalize this in Appendix~\ref{app:gap}.

\section{\name: A Framework for Mitigating Hidden Stratification}
\label{sec:alg-framework}
Inspired by the insights of Section \ref{sec:sec2},
we propose \name,
an algorithm to mitigate hidden stratification. A schematic overview of \name~is provided in Figure~\ref{fig:schematic}.

\begin{figure*}[!t!]
\makebox[\textwidth][c]{\includegraphics[height=0.5 in]{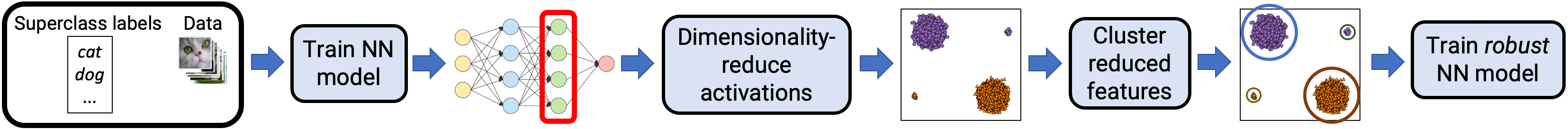}}
\caption{Schematic describing \name. The inputs are the datapoints and superclass labels. First, a model is trained with ERM on the superclass classification task. The activations of the penultimate layer are then dimensionality-reduced, and clustering is applied to the resulting features to obtain estimated subclasses. Finally, a new model is trained using these clusters as groups for GDRO.}
\label{fig:schematic}
\end{figure*}

Under the generative model of Section \ref{sec:generative-model}, each subclass is
described by a different Gaussian in latent feature space. This suggests that a natural approach to
\emph{identify} the subclasses is to transform the data into feature space, and then cluster
the data into estimated subclasses. However, this feature space is unknown.
To obtain a surrogate for this feature space, we leverage the
empirical observation that feature representations of deep neural networks trained on a \emph{superclass}
task can carry information about unlabeled \emph{subclasses} \citen{oakden2019hidden}.
Next, to \emph{improve performance} on these estimated subclasses, we minimize the maximum
\emph{per-cluster} average loss, by using the clusters as groups in the GDRO
objective~\citen{sagawa2020distributionally}. We provide more details below, and pseudocode in Appendix \ref{app:experiments} (Algorithm~\ref{alg:main}).

\subsection{Step 1: Estimating Approximate Subclass Labels}
\label{sec:generate-labels}
In the first step of \name, we train an ERM model on the superclass task
and cluster the feature representations of each superclass to generate proxy subclass labels.
Formally, we train a deep neural network $L \circ f_\theta$ to predict the superclass labels,
where $f_\theta : \mathcal{X} \to \R^d$ is a parametrized ``featurizer'' and $L: \R^d \to \Delta^B$ outputs classification logits.
We then cluster the features output by $f_\theta$ for the data of each
superclass into $k$ clusters, where $k$ is chosen automatically.
To each datapoint $x_i$ in the training and validation sets, we associate its cluster assignment $\tilde{z}_i \in \{1,\dots,k\}$.
We use the $\tilde{z}_i$'s as surrogates for the true subclass labels $z_i$.

\subsubsection{Clustering Details}
In practice, we apply UMAP dimensionality reduction \citep{mcinnes2018umap} before clustering,
as we find it improves results (Appendix \ref{app:experiments}).
Additionally, based on the insight of Section~\ref{sec:erm-failure}
that subclasses with high loss differences are more separable,
we also use the loss component (i.e., the component of the activation vector orthogonal to the decision boundary)
as an alternative representation.

We first tried using standard clustering methods (such as $k$-means and Gaussian mixture model clustering) in our work.
By visual inspection, we found that these methods often failed to capture smaller clusters, even if they were well-separated.
However, missing small clusters like this is problematic for \name, since these small clusters frequently correspond to
rare, low-performing subclasses. Additionally, these methods require specification of $k$.
We apply \emph{over-clustering} (clustering using a larger $k$) to remedy this problem in an efficient manner.
Naive overclustering also has drawbacks as it still requires manual specification of $k$, and if $k$ is set too large,
several clusters can be spurious and result in overly pessimistic and unstable
measurements of robust performance (as we explore in Appendix \ref{sec:fix-k}).
Thus, we develop a fully automated criterion based on the commonly used Silhouette (SIL) criterion \citep{silhouette} to
search for the number of clusters $k$, over-cluster to find smaller clusters that were missed, and filter out the
spurious overclusters. Empirically, our clustering approach significantly improves performance over ``vanilla'' clustering;
we hope that it may be of independent interest as well.
We describe our procedures in more detail in Appendix \ref{app:experiments}.

$k$ and other clustering and dimensionality reduction hyperparameters
are selected automatically based on an unsupervised SIL criterion \citep{silhouette} as described further in Appendix \ref{app:experiments}.

\subsection{Step 2: Exploiting Approximate Subclass Labels}
\label{sec:exploit-labels}
In the second step of \name, we use the GDRO algorithm from \cite{sagawa2020distributionally} and our estimated subclass labels $\tilde{z}_i$ 
to train a new classifier with better worst-case performance on the estimated subclasses.
Given data $\{(x_i, y_i, t_i)\}_{i=1}^n$ and loss function $\ell$, GDRO minimizes
$\max\limits_{t \in \mathcal{T}} \mathop{\E}\limits_{x,y \sim \hat{\P}_t} \Lt \ell((L \circ f_\T)(x), y) \Rt$
with respect to parameters $(L,\theta)$, where $\mathcal{T}$ is the discrete set of
groups and $\hat{\mathcal{P}}_t$ is the empirical distribution of examples from group $t$.
This
coincides with the true objective~\eqref{eq:robust_obj} when the true subclass labels $z_i$ are used as the group labels $t_i$.
In our case, we use the cluster assignments $\tilde{z}_i$ as the group labels instead,
i.e., minimize $\max\limits_{1 \le \tilde{z} \le k} \mathop{\E}\limits_{x,y \sim \hat{\mathcal{P}}_{\tilde{z}}} \Lt \ell(L \circ f_\T)(x), y) \Rt$.\footnote{In
Appendix \ref{app:theory}, we present an extension to the GDRO algorithm of
\citen{sagawa2020distributionally} to handle the case where the group assignments $\tilde{z}_i$ can be probabilistic labels in $\Delta^k$, instead of hard labels in $\{1,\dots,k\}$.}
Similarly, using the clusters fit on the training set, we generate cluster assignments for the validation set points and use these to compute the validation worst-case per-cluster performance. \name~uses this metric, rather than overall validation performance, to select the best model checkpoint over the training trajectory.

\section{Analysis of~\name}
\label{sec:analysis}

We now analyze a simple mixture model data distribution, based on the generative model presented in Section~\ref{sec:generative-model}.
 We show that in this setting, unlike ERM, \name~converges to the optimal robust risk at the same sample complexity rate as GDRO
 when it is able to recover the true latent features $\vec{Z}$ (and when ``soft'' group predictions are used).
Specifically, Example \ref{example:erm_fail}
shows that the robust risk of ERM can be arbitrarily worse than that of
GDRO, for data generated according to the generative model in
\Cref{sec:generative-model}. By contrast,
if the subclass labels estimated by \name~are sufficiently accurate,
then the objective minimized in Step 2 of \name~well approximates
the true GDRO objective \eqref{eq:robust_obj}.
In Theorem~\ref{thm:l3}, we use this to show that, when
each subclass is described by a different Gaussian in feature space,
\name~(with soft group predictions)
achieves the same optimal asymptotic sample complexity rates
as GDRO trained with true subclass labels.
We sketch the argument below; full proofs are deferred to Appendix \ref{app:theory}.

First, suppose we could compute the true data distribution $\mathcal{P}(x,y,z)$.
Our goal is to minimize the maximum per-subclass training loss by solving Eq.~\eqref{eq:robust_obj}.
Even with infinite data, we cannot estimate the \emph{individual} $z_i$'s to arbitrary accuracy,
so we cannot directly compute the objective in \eqref{eq:robust_obj}.
However, we can estimate the per-subclass losses as follows:
for each training example $(x_i, y_i)$, we use $\mathcal{P}$
to compute the probability that it comes from subclass $c$,
and use that to weight the loss corresponding to that example.
In Lemma~\ref{lem:unbiased}, we show that when the training data is randomly sampled from $\P$, this yields an unbiased estimate of the average per-subclass empirical risk.

\begin{restatable}{lem}{unbiasedness}
\label{lem:unbiased}
Let ${R}_c$ be the sample average loss of examples in subclass $c$. Let $w(x, c) \defeq \ff{p(x | z = c)}{{p}(x | y = S(c))}$. Let $\tilde{R}_c$ be the sample average of $w(x_i, c) \ell(f(x_i), y_i)$ over all examples $x_i$ with \emph{superclass} label $y_i = S(c)$.
Then $\tilde{R}_c$ is an unbiased estimate of ${R}_c$, and their difference converges to $0$ at the rate $O({1}/{\sqrt{n}})$.
\end{restatable}

In practice, we do not have access to the true distribution $\P$, so we estimate it with $\hat{\P}$, computed from data.
Thus, the weights $w(x,c)$ are replaced by weights $\hat{w}(x, c)$ estimated from $\hat{\P}$, leading
to an estimate $\hat{R}_c$ of the quantity $\tilde{R}_c$ defined in Lemma~\ref{lem:unbiased}.
Nevertheless, if we can bound the total variation estimation error of $\hat{\P}$,
we can use this to bound the error in this loss estimate, as shown in Lemma~\ref{lem:l2} (Appendix \ref{app:theory}).\footnote{In general, estimation of the data distribution to within small total variation error is a strong requirement, and is difficult to ensure in practice. However, it can be achieved in certain special cases, such as the mixture-of-Gaussian setting in Theorem~\ref{thm:l3}.}
In Theorem~\ref{thm:l3}, we leverage Lemma~\ref{lem:l2} and recent results on learning
Gaussian mixtures~\citen{ashtiani2018near} to show that, when each subclass is described by a different Gaussian,
$\hat{\P}$ can be estimated well enough so that the minimizer of the perturbed robust loss converges to the minimizer of the true robust loss
at the optimal sample complexity rate.

\begin{restatable}{thm}{gaussian}
\label{thm:l3}
Let $\hat{R}_{robust} \defeq \max_c \hat{R}_c$. Suppose $\ell$ and $f$ are Lipschitz, $f$ has bounded parameters,
and $\P(x | z = c)$ is Gaussian and unique for each subclass $c$.
Then, if we estimate $\hat{\P}$ using the algorithm from~\citen{ashtiani2018near}, $\hat{f} \defeq \min\limits_{f \in \F} \hat{R}_{\text{robust}}(f)$
satisfies ${R}_{\text{robust}}(\hat{f}) - \min\limits_{f \in \mathcal{F}} {R}_{\text{robust}}(f) \le \tilde{O}(\sqrt{1/n})$ w.h.p.
\end{restatable}

Theorem~\ref{thm:l3} implies that if each subclass is Gaussian in feature space,
and we have access to this feature space (i.e., we can invert the mapping $g$ from features $\vec{Z}$ to data $X$), then we can cluster the features to estimate $\hat{\P}$,
and the robust generalization performance of the model
that minimizes the resulting perturbed training loss $\hat{R}_{\text{robust}}$ scales
the same as does that of the minimizer of the true robust training loss $R_{\text{robust}}$,
in terms of the amount of data required.
This underscores the importance of recovering a ``good'' feature space;
empirically, we show in
Appendix \ref{app:results} that the choice of model architecture
can indeed dramatically impact the model feature space and thus the ability to recover~subclasses.

\section{Experiments}
\label{sec:experiments}
We empirically validate that~\name~can mitigate hidden stratification across four datasets.
In \Cref{sec:e2e},
we show that
when subclass labels are unavailable, \name~{improves} robust performance
over standard training methods.
In \Cref{sec:clustering-exp}, we analyze the clusters
returned by \name~to understand the reasons for this improvement;
we confirm that \name~identifies clusters that correspond to poorly-performing subclasses,
which enables accurate measurement of robust performance.
In \Cref{sec:val-metric}, we ablate the contributions that \name's robust training objective
and \name's improved measurement of validation robust performance each make to the performance gains of \name.
Finally, in \Cref{sec:bit}, we evaluate the use of recent pretrained image embeddings \citen{koleskinov2020bit}
as a source of features for \name, and find that this further improves performance of \name~on some applications.
Additional details on datasets, model architectures, and experimental procedures are provided in Appendix~\ref{app:experiments}.

\subsection{Datasets}
\label{sec:datasets}
\textbf{Waterbirds.} Waterbirds, a robustness benchmark introduced to evaluate GDRO in \citen{sagawa2020distributionally},
contains images of land-bird and water-bird species on either land or water backgrounds.
The task is to classify images as ``land-bird''  or ``water-bird''; however,
95\% of land (water)-birds are on land (water) backgrounds,
causing ERM to frequently misclassify both land-birds on water and water-birds on land.

\textbf{Undersampled MNIST (U-MNIST).}
We design U-MNIST as a modified version of MNIST \citen{lecun2010mnist}, where the task is
to classify digits as `$<$5' and `$\ge$5' (digits 0-9 are the subclasses).
In addition, we remove 95\% of `8's; due to its rarity,
it is challenging for ERM to perform well on the `8' subclass.

\textbf{CelebA.} CelebA is a common face classification dataset also used as a robustness benchmark in \citen{sagawa2020distributionally}.
The task is to classify faces as ``blond'' or ``not blond.''
Because only 6\% of blond faces are male,
ERM performs poorly on this rare subclass.

\textbf{ISIC.} The ISIC skin cancer dataset \citen{codella2019skin} is a public real-world dataset for classifying skin lesions as ``malignant'' or ``benign.'' 48\% of benign images contain a colored patch.
Of the non-patch examples, 49\% required histopathology (a biopsy) to diagnose. We report AUROC for ISIC, as is standard \citen{rieger2019interpretations}.

\subsection{End-to-End Results}
\label{sec:e2e}
\setlength\tabcolsep{5pt}
\begin{table}[t]
	\centerfloat
  \small
	\begin{tabular}{lcc|cccccccccc}\\
		\toprule
       \multicolumn{1}{c}{\textbf{Method}} &
        \multicolumn{1}{c}{\textbf{Requires}} &
        \multicolumn{1}{c}{\textbf{Metric}} &
      \multicolumn{1}{c}{\textbf{Waterbirds}} &
      \multicolumn{1}{c}{\textbf{\textsc{U-MNIST}}} &
      \multicolumn{2}{c}{\textbf{{ISIC}}} &
       \multicolumn{1}{c}{\textbf{{CelebA}}} & 
      \\
      
      & \multicolumn{1}{c}{\textbf{Subclass Labels?}}  & \multicolumn{1}{c}{\textbf{Type}} & & & \textbf{Non-patch} & \textbf{Histopath.} &
      \\
    \midrule
		ERM & \xmark  & Robust & 63.3($\pm$1.6)  & 93.9($\pm$0.6) 
		& \textbf{.922}($\pm$.003)  & .875($\pm$.005)  & 40.3($\pm$2.3)
		   \\
		& & Overall & 97.3($\pm$0.1) & 98.7($\pm$0.1)
		& \multicolumn{2}{c}{.957($\pm$.002) } & 95.7($\pm$0.0)
		 \\
    \midrule
		\name~(ours) & \xmark & Robust& \textbf{76.2}($\pm$2.0) & \textbf{95.7}($\pm$0.6) 
		& .912($\pm$.005) & \textbf{.876}($\pm$.006) & \textbf{53.7}($\pm$1.3) 
		\\
		& & Overall & 95.7($\pm$0.5)  & 98.1($\pm$0.3)
		&\multicolumn{2}{c}{.927($\pm$.008)} & 94.6($\pm$0.2) 
		 \\
    \midrule
    \midrule
		Subclass-GDRO & \cmark & Robust & 90.7($\pm$0.4) & 96.8($\pm$0.4) 
		 & .923($\pm$.003) & .875($\pm$.004) & 89.3($\pm$0.9)
		 \\
		 & & Overall & 92.7($\pm$0.4)& 98.0($\pm$0.3)
		& \multicolumn{2}{c}{.933($\pm$.005) } & 92.8($\pm$0.1)
		 \\

    \bottomrule
  \end{tabular}
   \caption{
    Robust and overall performance for ERM, \name,
    and subclass-GDRO (i.e., GDRO with true subclass labels).
    Performance metric is accuracy for all datasets but ISIC, which uses AUROC.
     Bolded values are best between ERM and \name, which do not require subclass labels.
     Sub-columns for ISIC represent two different definitions of the ISIC subclasses; see \Cref{sec:clustering-exp}.
  }
  \label{tab:main_results}
\end{table}

We first show that \name~substantially improves the worst-case subclass accuracy, while modestly affecting overall accuracy.
(Recall that we refer to worst-case subclass accuracy as ``robust accuracy'' [Eq.~\eqref{eq:robust}].)
We train models on each dataset in Section \ref{sec:datasets} using (a) ERM, (b) \name, and (c) GDRO with true subclass labels (``Subclass-GDRO''),
and report both robust and overall performance metrics in Table \ref{tab:main_results}.
Compared to ERM, training with \name~improves robust accuracy by up to 22 points,
and substantially reduces the gap between the robust error of the ERM model and
that of the subclass-GDRO model---despite the fact that \name~does not require subclass labels. In
{Appendix \ref{app:results}}, we show that \name~also outperforms other subclass-agnostic baselines,
such as GDRO trained using the \emph{superclasses} as groups.

On Waterbirds, U-MNIST, and CelebA, \name~significantly improves worst-case subclass accuracy over ERM.
On ISIC, all methods perform similarly in terms of both AUROC on the non-patch subclass
and AUROC on the clinically meaningful histopathology subclass.
On CelebA, although \name~improves upon ERM, it substantially underperforms subclass-GDRO.
However, this gap can be closed when improved features are used:
if we cluster pretrained BiT embeddings \citen{koleskinov2020bit} rather than ERM features
and use the resulting cluster assignments for the second stage of \name, the robust accuracy improves to
nearly match that of subclass-GDRO. We describe this experiment in Section \ref{sec:bit}.

In terms of overall performance, ERM generally performs best (as it is designed to optimize for average-case performance),
followed by \name~and then subclass-GDRO. However, this difference is generally much smaller in magnitude than the increase in robust performance.

\subsection{Clustering Results}
\label{sec:clustering-exp}
Step 1 of \name~is to train an ERM model and cluster the data of each superclass in its feature space.
We analyze these clusters to better understand \name's behavior.
First, in \Cref{sec:subclass-recovery} we show that \name~finds clusters that align well with poorly-performing
human-labeled subclasses. This helps explain why the second step of \name,
running GDRO using the cluster assignments as groups, improves performance on these subclasses (as demonstrated in Section~\ref{sec:e2e}).
Next, in \Cref{sec:unlabeled-discovery} we
show that \name~can discover meaningful subclasses
that were not labeled by human annotators.
Finally,  in \Cref{sec:cluster-measure} we show that the worst-case performance measured on the clusters
returned by \name~is a good approximation
of the true robust performance.

\subsubsection{Subclass Recovery}
\label{sec:subclass-recovery}
We evaluate the ability of \name~to identify
clusters that correspond to the true subclasses.
We focus on identification of \emph{poorly-performing} subclasses,
as these determine robust performance.
In Table~\ref{tab:subclass-recovery}, we compute the precision and recall
of the cluster~returned by \name~that most closely aligns with each given subclass.
\emph{Precision} is the fraction of cluster examples with~that subclass label;
\emph{recall} is the fraction of subclass examples assigned to the cluster. 
For each poorly-performing subclass, \name~identifies a cluster with high recall and better-than-random precision. (Interestingly, while the precision and recall are substantially better than random in all cases, they are often still quite far from the optimal value of 1.0, but Step 2 of \name~nevertheless achieves substantial worst-group performance gains.)

We note that the lower recall on ISIC is because the no-patch subclass is often split into multiple clusters;
in fact, this subclass is actually composed of two semantically distinct groups as discussed below.
If these sub-clusters are combined into one, the precision and recall of the resulting cluster at identifying no-patch examples are $> 0.99$ and $> 0.97$ respectively.

\subsubsection{Unlabeled Subclass Discovery}
\label{sec:unlabeled-discovery}
In addition to yielding clusters aligned with human-annotated subclasses,
our procedure can identify semantically meaningful subclasses that were not specified in the human-provided schema.
On U-MNIST, 60\% of trials of \name~partition the ``7'' subclass into two
subclusters, each containing stylistically different images (Figure \ref{fig:cluster-viz}c, Appendix \ref{app:results}).
On ISIC, 70\% of \name~trials reveal distinct benign clusters within the no-patch subclass (see Figure \ref{fig:cluster-viz}g-i).
In these trials, at least 77\% of images in one of these no-patch clusters required histopathology
(biopsy \& pathologist referral),
while such images made up $<$7\% of each other cluster.
In other words, the no-patch subclass split into ``histopathology'' and ``non-histopathology'' clusters, where
the former datapoints were harder for clinicians to classify.
We comment on the real-world importance of the ISIC result in particular.
Naively, the overall AUROC on ISIC obtained using the ERM model is 0.957,
which suggests a high-performing model;
however, our clustering reveals that a large fraction of the benign images contain a ``spurious'' brightly
colored patch, which makes them very easy to classify. The model performs substantially worse on examples
without such a patch, and worse still on ``more difficult'' examples for which a clinician also utilized a
histopathology examination to make a diagnosis. Thus, if deployed in practice with a target sensitivity value
in mind, the appropriate way to set an operating point for this model is in fact cluster-dependent; if a single
operating point were set using the aggregate ROC curve, the true sensitivity on the histopathology subclass
would be substantially lower than intended. This means that even simply \emph{measuring}
hidden stratification via Step 1 of \name~can provide crucial information that would help avoid spurious false negatives at test
time---the worst type of error a medical screening application can make.

\setlength\tabcolsep{3pt}
\begin{table}[h]
	\centerfloat
	\scriptsize
	\begin{tabular*}{\textwidth}{p{0.2\textwidth}|@{\extracolsep{\fill}}c c c c c c}
	\toprule
	\textbf{Task} & \textbf{Subclass} &  \textbf{Subclass Prevalence} &  \textbf{\% of trials} & \textbf{Precision} & \textbf{Recall} \\
	\midrule
	U-MNIST & ``8'' digit & 0.012 & 80 & 0.54 & 0.74  \\ 
	Waterbirds & Water-birds on land & 0.05 & 100 & 0.19 & 0.91  \\
	Waterbirds & Land-birds on water & 0.05 & 100 & 0.33 & 0.93   \\
	ISIC & No-patch & 0.48 & 100 & 0.99 & 0.59  \\
	ISIC & Histopathology & 0.23 & 70 & 0.77 & 0.92  \\
	{CelebA} & blond males & 0.06 & 100 & {0.14} & {0.88}  \\
	CelebA (w/BiT) & blond males & 0.06 & 100 & {0.93} & {0.68} \\
    \bottomrule
  \end{tabular*}
  \caption{
		Alignment of clusters with poorly-performing subclasses on the train set.
		We run Step 1 of \name~over multiple random seeds (i.e., train multiple ERM models and cluster their activations).
		In col.~4, we report the percentage of these trials with a cluster above the given precision and recall thresholds (cols. 5, 6)
		for identifying the subclass in col. 2.
		We report the proportion of training examples from that subclass within its superclass in col. 3.
	  }
\label{tab:subclass-recovery}
\end{table}

\subsubsection{Estimating Robust Accuracy}
\label{sec:cluster-measure}
We show that the clusters returned by \name~enable improved measurement of worst-case subclass performance.
Specifically, we measure the worst-case performance across any cluster returned by \name~(which we call the
``cluster-robust'' performance)
and compare this to the true robust performance and the overall performance.
We present results for both ERM and \name~in Table~\ref{tab:measurement}.
In most cases, the cluster-robust performance is much closer to the true robust performance
than the overall performance is.
On ISIC, cluster-robust performance even yields a better estimate of robust performance on the histopathology subclass
than does performance on the patch/no-patch subclass. 
By comparing cluster-robust performance to overall performance, we
can detect hidden stratification (and estimate its magnitude) without requiring subclass labels.

\setlength\tabcolsep{3pt}\
\begin{table}[h]
  \vspace{-2em}
	\centerfloat
  \small
	\begin{tabular}{lc|cccccccccc}\\
		\toprule
       \multicolumn{1}{c}{\textbf{Method}} &
        \multicolumn{1}{c}{\textbf{Metric}} &
      \multicolumn{1}{c}{\textbf{Waterbirds}} &
      \multicolumn{1}{c}{\textbf{\textsc{U-MNIST}}} &
      \multicolumn{2}{c}{\textbf{{ISIC}}} &
       \multicolumn{1}{c}{\textbf{{CelebA}}} & 
       \\
       & \multicolumn{1}{c}{\textbf{Type}} & & & \textbf{Non-patch} & \textbf{Histopath.} &
       \\
    \midrule
		ERM &  Robust & 63.3($\pm$1.6)  & 93.9($\pm$0.6) 
		& .922($\pm$.003)  & .875($\pm$.005)  & 40.3($\pm$2.3)  
		\\
		& Cluster-Robust & 76.8($\pm$1.4) & 92.3($\pm$2.5)
		& \multicolumn{2}{c}{.893($\pm$.013) } & 56.7($\pm$2.5) 
		 \\
		& Overall & 97.3($\pm$0.1) & 98.2($\pm$0.1)
		& \multicolumn{2}{c}{.957($\pm$.002) } & 95.7($\pm$0.1) 
		 \\
    \midrule
		\name &  Robust& 76.2($\pm$2.0) & 95.7($\pm$0.6)  
		& .912($\pm$.005) &.876($\pm$.006) & 53.7($\pm$1.3) 
		\\
		& Cluster-Robust & 78.3($\pm$1.1)  & 93.5($\pm$1.9) 
		&\multicolumn{2}{c}{.897($\pm$.011)}  & 70.8($\pm$1.1)  
		 \\
		& Overall & 95.7($\pm$0.5)  & 97.9($\pm$0.2)
		&\multicolumn{2}{c}{.927($\pm$.008)} & 94.6($\pm$0.2) 
		\\
    \bottomrule
  \end{tabular}
  \caption{
        Comparison of overall, cluster-robust, and robust performance. (Conventions as in Table \ref{tab:main_results}.)
  }
    \label{tab:measurement}
\end{table}

In addition, improvements in robust performance from \name~compared to ERM
are accompanied by increases in cluster-robust performance;
by comparing the cluster-robust performance of ERM and \name,
we can estimate how much \name~improves hidden stratification.

\subsection{Effects of Validation Metric}
\label{sec:val-metric}
\name's improvement of robust performance has two potential explanations: (1) minimizing the cluster-robust training loss is a better surrogate objective for the true robust performance than minimizing the overall training loss, and (2) selecting the best model checkpoint based on validation cluster-robust performance is better than selecting based on overall validation performance. To decouple these two effects, in Table~\ref{tab:val_metric} we display the test robust performance for ERM and \name~when using the \emph{true} robust validation performance as a checkpointing metric. This change generally improves performance for both methods, but \name~still significantly outperforms ERM on all datasets except ISIC. This shows that, for the goal of maximizing robust performance, the GDRO objective with cluster labels indeed performs better than the ERM objective.

In Table~\ref{tab:val_metric}, we also display the effects of changing the frequencies of subclasses in the validation set. As described in Appendix \ref{sec:data-details}, by default we re-weight the validation and test sets of U-MNIST and Waterbirds so that the effective frequencies of each subclass are the same as they are in the training set.
If we turn off this reweighting, the cluster-robust validation performance is a more accurate measure of the true robust performance, since the frequency of the underperforming subclass increases in the validation set. Using this unreweighted metric to select the best model checkpoint increases the true robust performance of \name~to 83.3\%---an improvement of over \textbf{22} points compared to ERM checkpointed against average accuracy on the same unreweighted validation set.
We stress that the true validation subclass labels are still assumed to be unknown in this experiment.
Having training and validation sets with different distributions is realistic in many situations.

We remark that in typical supervised learning settings, model selection---encompassing both hyperparameter tuning and selection of a checkpoint from the training trajectory---is done with the help of a validation set on which the desired metric of interest can be computed.
By contrast, the setting we study in \name~is more challenging not only due to the absence of training group labels, but also because the absence of \emph{validation} group labels means that this selection metric (worst-group accuracy / AUROC) cannot even be computed exactly on the validation set. The results of this section are encouraging in that they suggest that \name's cluster-robust performance is an acceptable proxy metric for model selection.\footnote{Our hyperparameter tuning procedure is described in Section \ref{sec:methods}; we also do not use validation set group labels to tune any of \name's hyperparameters.} 

\setlength\tabcolsep{3pt}\
\begin{table}[h]
  \vspace{-2em}
	\centerfloat
  \small
	\begin{tabular}{lc|cccccccccc}\\
		\toprule
       \multicolumn{1}{c}{\textbf{Training}} &
        \multicolumn{1}{c}{\textbf{Validation}} &
      \multicolumn{1}{c}{\textbf{Waterbirds}} &
      \multicolumn{1}{c}{\textbf{\textsc{U-MNIST}}} &
      \multicolumn{2}{c}{\textbf{{ISIC}}} &
       \multicolumn{1}{c}{\textbf{{CelebA}}} & 
       \\
      \textbf{Method} & \textbf{Checkpoint Metric} & & & \textbf{Non-patch} & \textbf{Histopath.} &
       \\
    \midrule
		ERM &  Acc. & 63.3($\pm$1.6)  & 93.9($\pm$0.6) 
		& .922($\pm$.003)  & .875($\pm$.005)  & 40.3($\pm$2.3)  
		\\
		& Unw. Acc. & 60.7($\pm$0.7) & 94.2($\pm$0.8)
		& \multicolumn{2}{c}{-} & - 
		\\
		& Robust Acc. & 68.8($\pm$0.9) & 94.5($\pm$0.9)
		& .924($\pm$.003) & .880($\pm$.004) & 46.3($\pm$2.1)
		 \\
    \midrule
		\name & Cluster-Robust Acc. & 76.2($\pm$2.0) & 95.7($\pm$0.6)  
		& .912($\pm$.005) &.876($\pm$.006) & 53.7($\pm$1.3) 
		\\
		& Unw. Cluster-Robust Acc. & 83.3($\pm$1.3)  & 95.7($\pm$0.6) 
		&\multicolumn{2}{c}{ - }  & -  
		 \\
		& Robust Acc. & 83.8($\pm$1.0)  & 96.3($\pm$0.5)
		& .915($\pm$.004) & .877($\pm$.006)  & 54.9($\pm$1.9) 
		\\
    \bottomrule
  \end{tabular}
  \caption{Test robust performance for each method, where the ``best'' model checkpoint over the training trajectory
  is selected according to the listed metric on the validation set.\protect\footnotemark~``Unw.'' stands for unweighted, where we do not reweight the frequencies of different subclasses in the evaluation sets. This only applies to Waterbirds and U-MNIST, as the subclass proportions of ISIC and CelebA are roughly equal across splits. (For ISIC, the checkpoint metrics are AUROC rather than accuracy; in the two subcolumns we define the true robust performance as AUROC for the non-patch or histopathology subclass, respectively.)}
    \label{tab:val_metric}
\end{table}

\subsection{Extension: Leveraging Pretrained Embeddings} 
\label{sec:bit}
\footnotetext{Test \emph{average} performances for ERM/\name~models validated on robust accuracy are reported in Table \ref{tab:val_metric_avg}.}
As an alternative to training an ERM model,
we assess whether recent pretrained image embeddings (BiT \citen{koleskinov2020bit})
can provide better features for Step 1 of \name.
Specifically, we modify Step 1 of \name~to compute BiT embeddings for the datapoints,
cluster the embeddings, and use these cluster assignments as estimated subclass labels in Step 2 of \name.
This modification (\name-BiT) dramatically improves robust accuracy on CelebA to \underline{\textbf{87.3\%}} ($\pm 1.3\%$), nearly matching subclass-GDRO.\footnote{Overall accuracy drops somewhat to 91.5\%.}

The CelebA BiT clusters align much better with the true subclasses (cf. Table~\ref{tab:subclass-recovery}),
which helps explain this improvement. Similarly, cluster-robust accuracy measured using the BiT clusters
is much closer to the true robust accuracy: for the \name-BiT model,
average cluster-robust performance on the BiT clusters is
$83.3\pm 1.3\%$, and for ERM it is $33.9 \pm 2.5\%$ [compared to the true robust accuracy of 40.3\%].

Despite its excellent performance on CelebA, the default \name~implementation outperforms \name-BiT
on the other datasets, suggesting that BiT is not a panacea:
on these datasets, the task-specific information contained in the representation of the trained ERM model
seems to be important for identifying meaningful clusters.
See Appendix \ref{app:bit-details} for additional evaluations and discussion.
Extending Step 1 of \name~to enable automatically selecting between different representations
(e.g., BiT vs. ERM) is a compelling future topic.

\section{Conclusion}
We propose \name, a two-step approach for measuring and mitigating hidden stratification without requiring access to subclass labels.
\name's first step, clustering the features of an ERM model, identifies clusters that
provide useful approximations of worst-case subclass performance.
\name's second step, using these cluster assignments as groups in GDRO,
yields significant improvements in worst-case subclass performance.
We analyze \name~in the context of a simple generative model,
and show that under suitable assumptions \name~achieves the same asymptotic sample complexity rates
as if we had access to true subclass labels.
We empirically validate \name~on four datasets, and find evidence that it can reduce hidden stratification
on real-world machine learning tasks.
Interesting directions for future work include further exploring different ways to learn representations
for the first stage of \name, developing better unsupervised metrics to choose between representations and clustering methods,
and characterizing when ERM learns representations that enable separation of subclasses.

\clearpage

\section*{Acknowledgments}
We thank Arjun Desai, Pang Wei Koh, Shiori Sagawa, Zhaobin Kuang, Karan Goel, Avner May, Esther Rolf, and Yixuan Li for helpful discussions and feedback.

We gratefully acknowledge the support of DARPA under Nos. FA86501827865 (SDH) and FA86501827882 (ASED); NIH under No. U54EB020405 (Mobilize), NSF under Nos. CCF1763315 (Beyond Sparsity), CCF1563078 (Volume to Velocity), and 1937301 (RTML); ONR under No. N000141712266 (Unifying Weak Supervision); the Moore Foundation, NXP, Xilinx, LETI-CEA, Intel, IBM, Microsoft, NEC, Toshiba, TSMC, ARM, Hitachi, BASF, Accenture, Ericsson, Qualcomm, Analog Devices, the Okawa Foundation, American Family Insurance, Google Cloud, Swiss Re, Total, the HAI-AWS Cloud Credits for Research program, the Schlumberger Innovation Fellowship program, and members of the Stanford DAWN project: Facebook, Google, VMWare, and Ant Financial. The U.S. Government is authorized to reproduce and distribute reprints for Governmental purposes notwithstanding any copyright notation thereon. Any opinions, findings, and conclusions or recommendations expressed in this material are those of the authors and do not necessarily reflect the views, policies, or endorsements, either expressed or implied, of DARPA, NIH, ONR, or the U.S. Government.

\nocite{dong2015looking}
\nocite{hoai2013discriminative}
\bibliographystyle{plainnat}
\bibliography{no_subclass_left_behind}

\newpage
\section*{Appendix}
\appendix
\crefalias{section}{appsec}

\section{Extended Related Work}
\label{app:related}
Our work builds on several active threads in the machine learning literature.

\paragraph{Hidden Stratification.}
Our motivating problem is that of hidden stratification, wherein models trained on superclass labels exhibit highly variable performance on unlabeled subclasses \citen{oakden2019hidden}.
This behavior has been observed in a variety of studies spanning both traditional computer vision \citen{yao2011combining, recht2018cifar, dong2014looking, hoai2013discriminative}  and medical machine learning \citen{oakden2019hidden, Dunnmon2019-rr, Gulshan2016-we, Chilamkurthy2018-op}.
Of note is the work of \citen{recht2018cifar}, who propose that the existence of ``distribution shift'' at the subclass level may substantially affect measures of test set performance for image classification models on CIFAR-10.
They use a simple mixture model between an ``easy'' and a ``hard'' subclass to demonstrate how changes that would not be detectable at the superclass level could affect aggregate performance metrics.
\cite{chen2020angular} extend these ideas by developing notions of visual hardness, and suggest that better loss function design would be useful for improving the performance of machine learning models on harder examples. \cite{polyzotis2019slice} studies how to automatically find large, interpretable underperforming data slices in structured datasets.

Our approach is also inspired by the literature on causality and machine learning, and in particular by the common assumption that the data provided for both training and evaluation are independent and identically distributed (IID) \citen{scholkopf2019causality}.
This is often untrue in real-world settings; in particular, classes in real-world datasets are often composed of multiple subclasses,
and the proportions of these subclasses may change between training and evaluation settings---even if the overall class compositions are
the same.
Many of the guarantees from statistical learning theory break down in the presence of such non-IID data \citen{chalupka2015visual}, suggesting that models trained using traditional Empirical Risk Minimization (ERM) are likely to be vulnerable to hidden stratification.
This motivates the use of the maximum (worst-case) per-subclass risk, rather than the overall average risk, as the objective to be optimized.

\paragraph{Neural Representation Clustering.}
The first stage of the technique we propose for addressing hidden stratification (\name) relies heavily on our ability to identify latent subclasses via unsupervised clustering of neural representations learned via ERM.
This has been an area of substantial recent activity in machine learning, and has provided several important conclusions upon which we build in our work.
The work of \citen{mcconville2019n2d} and \citen{shukla2018semi}, for instance, demonstrate the utility of a simple autoencoded representation for performing unsupervised clustering in the feature space of a trained model.
While the purpose of these works is often to show that deep clustering can be competitive with semi-supervised learning techniques, the mechanics of clustering in model feature space explored by these works are important for our present study.
Indeed, we directly leverage the conclusion of \citen{mcconville2019n2d} that Uniform Manifold Approximation and Projection (UMAP) \citen{mcinnes2018umap} works well as a dimensionality reduction technique for deep clustering in the current study.

Further, the fact that work such as \citen{han2019learning} directly uses neural representation clustering to estimate the presence of novel classes in a given dataset provides an empirical basis for our approach, which uses a model trained with ERM to approximately identify unlabeled subclasses within each superclass.
Similarly, \citen{ji2019invariant} demonstrate excellent semi-supervised image classification performance by maximizing mutual information between the class assignments of each pair of images. 
Their work demonstrates not only the utility of a clustering-style objective in image classification, but also suggests that overclustering -- using more clusters than naturally exist in the data -- can be beneficial for clustering deep feature representations in a manner that is helpful for semi-supervised classification.

A related, but different, approach is that of \citep{dong2014looking}, who explicitly attempt to identify subcategories of classes via a graph and SVM-based ``subcategory mining'' framework in order to improve overall task performance. The subcategory mining algorithm is quite complicated and uses manually extracted features (rather than automatically learned features, e.g., from CNNs); in addition, this work is geared towards improving overall performance, rather than ensuring good performance on \emph{all} subcategories. Nevertheless, it is an important piece of prior literature.

\paragraph{Distributionally Robust Optimization.}
The second stage of \name depends on our ability to optimize the worst-case classification loss over existing subgroups.
This formulation draws a clear connection between our work and the literature on fairness in machine learning \citen{barocas-hardt-narayanan}, which is at least partially concerned with ensuring that trained models do not disadvantage a particular group in practice.
While there exist a wide variety of definitions for algorithmic fairness \citen{hardt2016equality, lipton2018does, kearns2019average}, the common idea that models should be optimized such that they respect various notions of fairness is similar to the motivation behind our work.

\emph{Distributionally robust optimization} studies the problem of optimizing for worst-case performance with respect to some ``uncertainty set'' of distributions. A multitude of recent papers have explored optimizing distributionally robust objectives in slightly different contexts.
Most relevant to our work is the study of \citen{sagawa2020distributionally}, who propose the group DRO algorithm for training classifiers with best worst-case subgroup performance (in other words, the ``uncertainty sets'' in this case are the per-group distributions).
Crucially, this algorithm demonstrates improved worst-case subclass performance in cases where triplets ($x,y,g$) are known for every data point, with $x$ is the input data, $y$ is the true label, and $g$ is a true subgroup label.
While \citen{sagawa2020distributionally} present preliminary evidence that group DRO can work well in the presence of noisy $g$, the efficacy of the algorithm in this setting remains functionally unexplored.
We leverage the group DRO algorithm as an optimizer for minimizing the worst-case loss with respect to our approximately identified subclasses.

We also discuss other works on DRO. \citen{duchidistributionally}, for instance, considers the general problem of optimizing the worst-case loss over any possible subdistribution of the data above a specified size; while conceptually important, the goal of optimizing over arbitrary subdistributions of a minimum size is rather pessimistic compared to assuming more structure on these subdistributions (such as in GDRO).
\citep{levy2020large} built upon these ideas to design efficient methods for large-scale DRO, but the uncertainty sets considered are also less structured compared to group DRO and therefore generally give poorer results when the goal is in fact to optimize worst-group performance on a specific set of groups. Concurrently with our work, \cite{wang2020robust} studies how to train models with respect to group-level fairness constraints, given only noisy versions of the groups; they show how to solve this problem efficiently via a DRO formulation and a reweighting approach based on soft group assignments.
While their formulation encompasses several fairness objectives, such as ``equal opportunity'' (equal true positive rates across groups) and ``equal outcome'' (equal positive prediction rates across groups), it does not directly optimize for worst-group performance. Additionally, they assume that the noisy group labels are provided and their marginal distribution is the same as that of the true groups, 
whereas we do not assume any foreknowledge of the (noisy or true) groups.

Similar to the goal of optimizing worst-group performance is optimizing for \emph{group Pareto fairness}, i.e., seeking solutions that are Pareto-efficient in terms of the performances on each group. \cite{balashankar2019what} and \cite{martinez2020minimax} both study this (more general) problem in the case where the group labels are known. In addition, follow-ups to these papers that were concurrent to our work explored this problem in the setting where the group labels are unknown \citep{lahoti2020fairness, martinez2021blind}; however, both of these works focus on simpler structured datasets rather than more challenging settings such as image classification.

Other relevant techniques include invariant risk minimization, which attempts to train classifiers that are optimal across data drawn from a mixture of distributions (i.e., a non-IID setting) \citep{arjovsky2019invariant}; methods from slice-based learning that learn feature representations optimized for ensuring high performance on specific subsets, or ``slices'' of the data \citep{chen2019slice}; mixture-of-experts models, which explicitly handle learning models for multiple different subsets of data \citep{jacobs1991adaptive}; and techniques for building robust classifiers via domain adaptation \citep{wang2019learning}.
While our work is closely related to these directions, a major difference is that we handle the setting where the different subclasses (i.e., groups, environments, slices, etc.) are unidentified.

\paragraph{Representation Learning with Limited or Noisy Labels.} A final research thread that is closely related to the work presented here focuses on deep representation learning in the absence of ground truth labels.
Our methods are similar in spirit to those from weak supervision \citen{mintz2009distant, ratner2019snorkel}, which focuses on training models using noisy labels that are often provided programmatically.
Our work can be seen as analyzing a new form of fine-grained weak supervision for DRO-style objectives, which is drawn from unsupervised clustering of an ERM representation.
Another related line of work is representation learning for few-shot learning \citen{snell2017prototypical}; however, our work fundamentally differs in the sense that we assume no access to ground truth subclass labels.

Other methods aim to automatically learn classes via an iterative approach. An early work of this type is \citen{caron2018deep}, which uses iterative clustering and ERM training to learn highly effective feature representations for image classification. More recently, \citen{gansbeke2020learning} used a self-supervised task to learn semantically meaningful features, and then generate labels using an iteratively refining approach.
Our work differs from these in that we do assume access to ground truth \textit{superclass} labels---which provide much more information than in the fully-unlabeled setting---and use clustering \textit{within each superclass} to generate approximate labels. In addition, our primary end goal is not accurate identification of the subclasses, but ensuring good worst-case performance among all subclasses.

Finally, other works aim to promote a notion of ``diversity'' among feature representations by adding different regularizers. In \citen{xie2017uncorrelation}, such a regularizer was introduced in the context of latent space models, to better capture infrequently observed
patterns and improve model expressiveness for a given size. More recently, \citen{muller2020subclass} introduced a regularizer that aims to promote diversity of the predicted logits. They showed that this method could also lead to estimation of subclasses within a superclass, without requiring subclass labels. However, this work focused on improving \emph{overall} performance, and specifically improvement of knowledge distillation; by contrast, our goal is to improve \emph{robust} performance. Nevertheless, integrating these recent ideas into our work is an interesting avenue for future work, to potentially further improve the feature learning stage.

\section{Experimental Details}
\label{app:experiments}

\subsection{\name~Pseudocode}
We provide pseudocode for \name~in Algorithm \ref{alg:main}, to complement the detailed description of our methodology in \Cref{sec:alg-framework}.\footnote{We note that the final step of \name---training a model to minimize the maximum per-cluster risk---can also be done
when ``soft'' (probabilistic) cluster labels are given instead of hard assignments; see Appendix~\ref{app:gdro-soft}.}
Note that our model class $\mathcal{F}$ (as per the notation in \Cref{sec:setup}) is a class of neural networks, composed of a ``featurizer'' module $f_\theta$ and a ``linear classification head'' $L$ that takes the feature representation to a prediction.

\begin{algorithm}[h]
   \caption{``\name''}
   \label{alg:main}
\SetAlgoLined

   \textbf{Input}: Data and superclass labels $(x,y) = \{(x_i, y_i)\}_{i=1}^n$; loss function $\ell(\cdot, \cdot)$; featurizer class $\mathcal{F}(\theta)$ parameterized by $\theta \in \R^p$, dimensionality reducer $g$ (e.g. UMAP; default: identity) \\
   {\bfseries Optional input:} Pretrained featurizer $f_\theta$ \\
   \BlankLine
   \uIf{featurizer $f_\theta$ provided}{
   \nl \textbf{pass}
   }
   \Else{
   \# train model [featurizer $f_\theta$ \& linear classification head $L$] to minimize empirical risk, and save featurizer \\
   \nl $f_\theta, L \leftarrow \argmin\limits_{\theta' \in \R^p, L'} \left\{ \tfrac{1}{n}\textstyle\sum\limits_{i=1}^n \ell(L' \cdot f_{\theta'}(x_i), y_i) \right\}$
   }
   \# compute feature vectors \\
   \nl $\{v_i\}_{i=1}^n = g(f_\theta(x_i))$\\
   \For{$b=1$ {\bfseries to} $B$}{
   \# cluster features of each superclass \\
   \nl {$\{\hat{z}_i\} \leftarrow$ \sc{Get\_Cluster\_Labels}}($\{v_i : y_i = b\}$)
   }
   \# train final model to minimize maximum per-cluster risk \\
   \nl $f_{\hat{\theta}}, \hat{L} \leftarrow \argmin\limits_{\theta'\in\R^p, L'} \left\{ \max\limits_{c \in \{1,\dots,C\}} \ff{1}{n_c} \textstyle\sum\limits_{i=1}^n \mathbf{1}(\hat{z}_i = c) \ell(L' \circ f_{\theta'}(x_i), y_i) \right\}$\\
   \nl \Return{($f_{\hat{\theta}}, \hat{L}$)}
\end{algorithm}

Note that our framework is not constrained by the specific choice of clustering algorithm, dimensionality reduction algorithm, or robust optimization algorithm. While we use GDRO throughout this work, other training techniques that encourage good robust performance could be swapped in for GDRO during ``Step 2'' of \name.

\begin{figure*}[]
\centering
\begin{subfigure}{0.3\linewidth}
\centering
\includegraphics[trim=0mm 0mm 0mm 0mm, clip, scale=0.3]{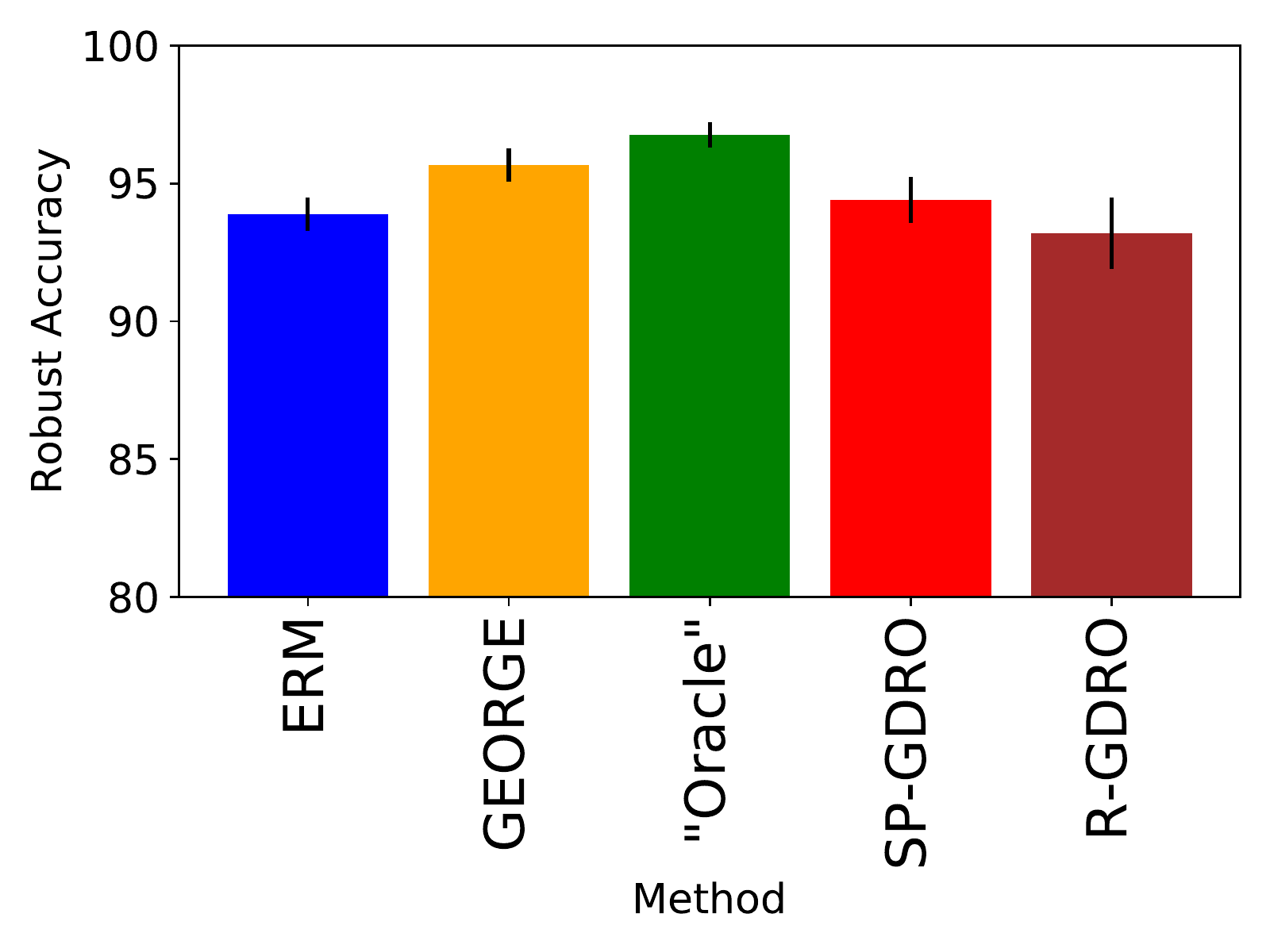}
\label{fig:mnist}
\caption{U-MNIST}
\end{subfigure}\qquad
\begin{subfigure}{0.3\linewidth}
\centering
\includegraphics[trim=0mm 0mm 0mm 0mm, clip, scale=0.3]{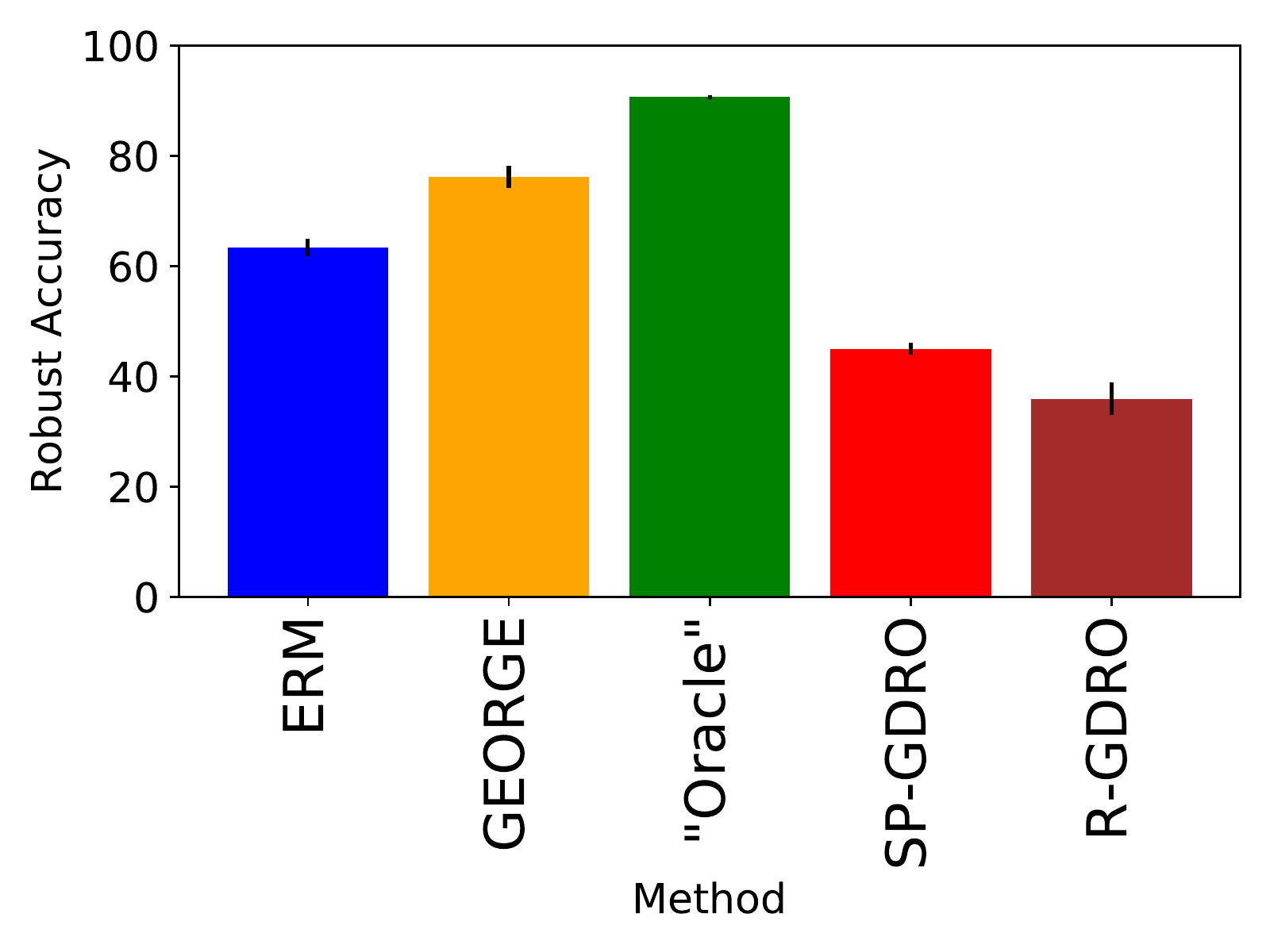}
\label{fig:wb}
\caption{Waterbirds} 
\end{subfigure} \\
\begin{subfigure}{0.3\linewidth}
\centering
\includegraphics[trim=0mm 0mm 0mm 0mm, clip, scale=0.3]{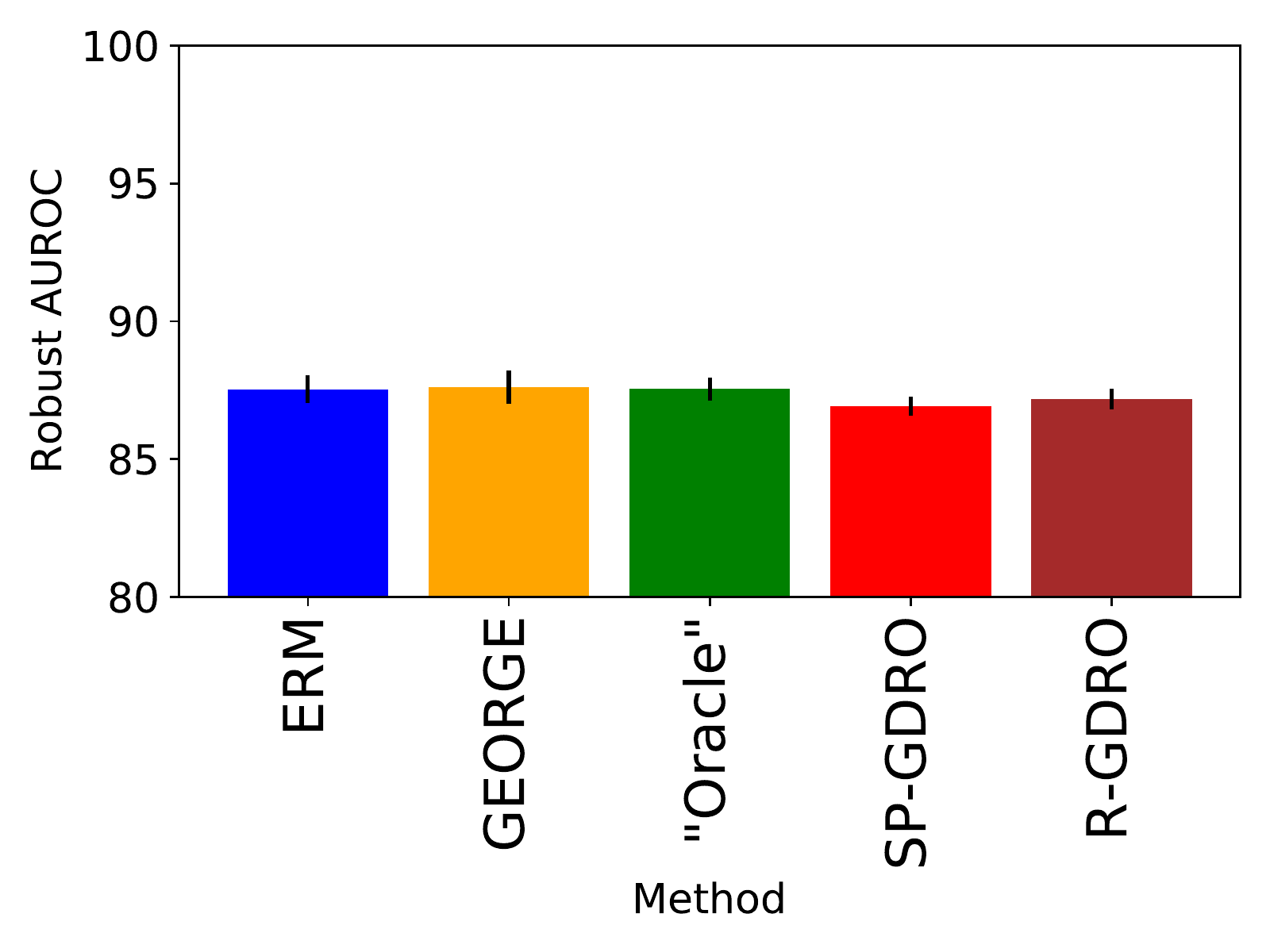}
\label{fig:celeba}
\caption{ISIC (Histopathology)}
\end{subfigure}\qquad
\begin{subfigure}{0.3\linewidth}
\centering
\includegraphics[trim=0mm 0mm 0mm 0mm, clip, scale=0.3]{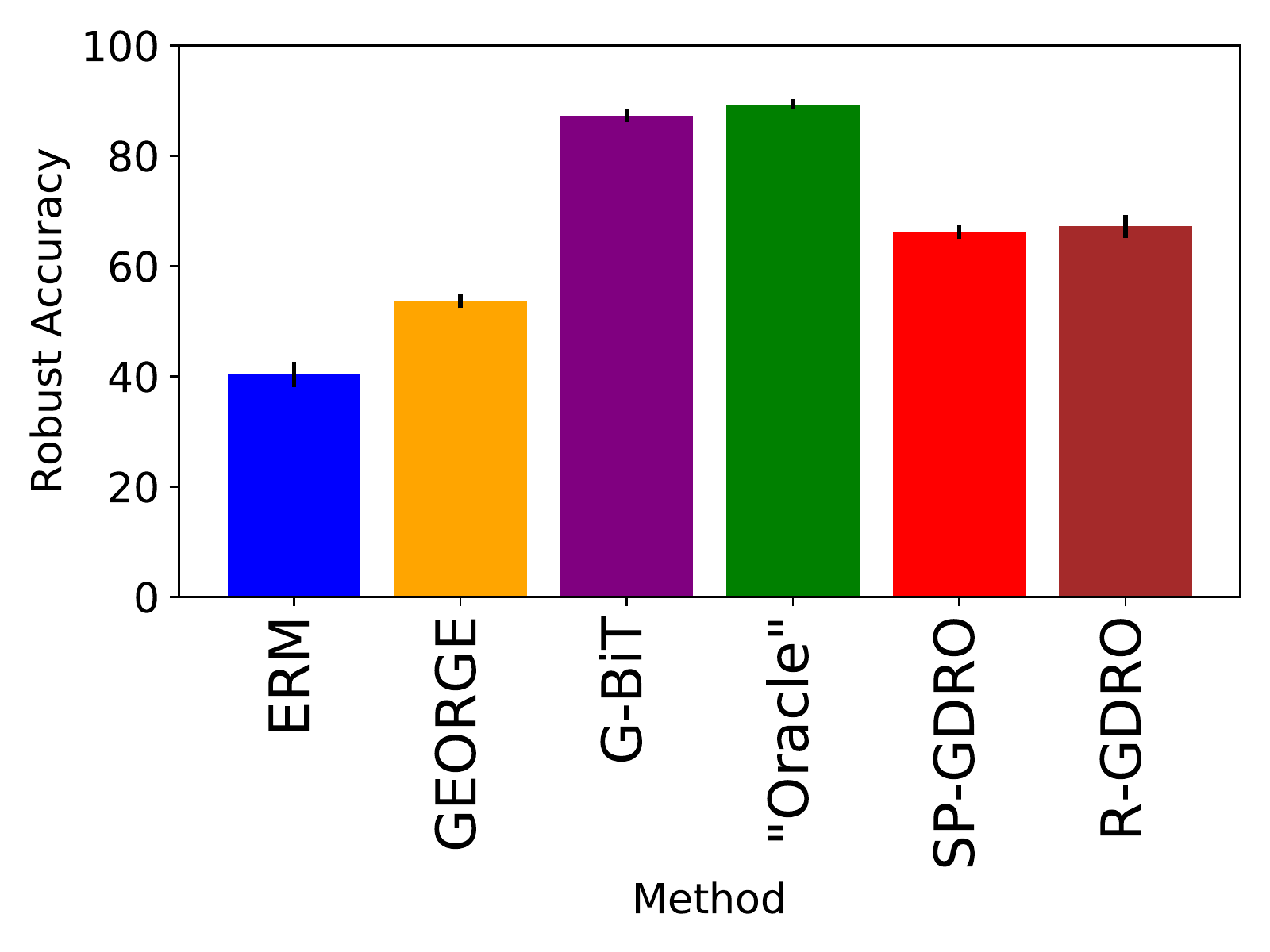}
\label{fig:isic}
\caption{CelebA}
\end{subfigure}
\caption{Worst-group performance of ERM, Superclass-GDRO (SP-GDRO), Random-GDRO (R-GDRO), \name, and Subclass-GDRO (SBC-GDRO).  We also show
\name~with BiT embeddings (G-BiT) for CelebA; however, G-BiT performed significantly worse on the other datasets.}
\label{fig:end-to-end}
\end{figure*}

\subsection{Dataset Details}
\label{sec:data-details}
Below, we describe the datasets used for evaluation in more detail. [We provide PyTorch dataloaders to support each one in our code.]

Each dataset contains labeled subclasses;
although the $\name$ procedure does not use the subclass labels at any point,
we use them to assess how well \name~(a)
can estimate the subclass labels and
(b) can estimate and improve worst-case subclass performance.

We remark that while we evaluate on binary classification tasks in this work, \name~can readily be applied in principle to tasks with any amount of superclasses and subclasses.

\subsubsection{U-MNIST}
Undersampled MNIST (U-MNIST) is a binary dataset that divides data from the standard MNIST dataset (which has 60,000 training points) into two superclasses: numbers less than five, and numbers greater than or equal to five.  Crucially, the ``8'' subclass is subsampled at 5\% of its usual frequency in MNIST. The rarity of the ``8'' subclass makes this task much more challenging than the default MNIST task, in terms of robust performance. We use data drawn from the original MNIST validation set as our test set; we create a separate validation set of 12,000 points by sampling from the MNIST training set, and use the remainder for training.

On the validation (and test) sets, we do not actually undersample the 8's, as this would leave only 50-60 ``8'' examples; instead, we downweight these examples when computing validation/test accuracies and losses, to mimic the rarity that would be induced by actually undersampling but still allow for more stable accuracy measurements.

\subsubsection{Waterbirds}
\label{sec:data-wb}
The Waterbirds dataset (4,795 training points) used in this work was introduced in \cite{sagawa2020distributionally}. Similar to our approach for U-MNIST, Sagawa et al. \citep{sagawa2020distributionally} create more balanced validation and test sets to allow for stable measurements, and downweight the examples from rare subclasses during evaluation (as we do for U-MNIST as well); we follow the same procedure.

\subsubsection{ISIC}
The dataset from the International Skin Imaging Collaboration (ISIC) website is, at time of writing, comprised of 23,906 images and their corresponding metadata \citep{codella2019skin}.  We extract the ISIC dataset directly from the site's image archive, which is accessible through a public API.\footnote{https://isic-archive.com/api/v1/} We only use images whose metadata explicitly describe them as ``benign'' or ``malignant.'' We use these descriptors in order to formulate the problem as a binary classification task that classifies images as either normal or abnormal. Other possible descriptors that exist in the image metadata (which we filter out) include ``indeterminate,'' ``indeterminate/benign,'' ``indeterminate/malignant,'' or no description. We created pre-set training, validation, and test splits from these images by randomly assigning 80\% of examples to the training set, 10\% to the validation set, and 10\% to the test set. 

We derive true subclass information from the image metadata. In particular, we observed that an image belongs in the benign patch subclass if and only if it is an image from the SONIC data repository \citen{SCOPE2016813}. As detailed in Section \ref{sec:clustering-exp}, we are retroactively able to identify the histopathology subclass through analysis of the diagnosis confirmation type of each image. Images in the histopathology subclass were explicitly mentioned as such---other possible diagnosis confirmation types include ``single image expert consensus,'' ``serial imaging showing no change,'' ``confocal microscopy with consensus dermoscopy,'' or no confirmation type.

\subsubsection{CelebA} The CelebA dataset \citen{liu2015faceattributes} is a standard face classification dataset containing over 200,000 examples ($\approx 163,000$ train) of celebrity faces, each annotated with 40 different attributes.  The images contain a wide variety of poses, backgrounds, and other variations.  The task is to classify faces as ``blond'' or ``not blond,'' as in \citen{sagawa2020distributionally}. We use the standard (pre-set) train/validation/test splits for this task.

\subsection{Methods}
\label{sec:methods}
\subsubsection{Result Reporting}
For each dataset, we perform ten separate trials of each method with different random seeds. The exception is CelebA, on which we perform five trials instead due to the larger dataset size.
In all result tables, `$X \pm Y$' intervals represents a 95\% confidence interval, where $X$ is the mean of the per-trial results and $Y$ is the interval half-width, calculated as standard deviation times 1.96 divided by the square root of the number of trials. Similarly, in all plots, error bars denote 95\% confidence intervals computed the same way.

\subsubsection{Baselines}
In addition to ERM, we run two additional baseline methods: superclass-GDRO and random-GDRO. Superclass-GDRO minimizes the maximum loss over each \emph{superclass}, i.e., runs GDRO using the superclasses as groups. Since we assume knowledge of the training superclass labels, this does not require additional information at training time. Random-GDRO runs GDRO using \emph{randomly chosen} groups within each superclass, where the groups are chosen to have the same sizes as the true subclasses. Since we do not assume the subclass sizes are known, this is not a method that would be useful in practice; rather, it helps highlight the difference between running GDRO with labels that do not align well with the true subclasses, and running GDRO with labels that do. Results on each dataset are presented in Figure~\ref{fig:end-to-end}.

\subsubsection{ERM Training Details}
The first stage of \name~is to train a model for each application using ERM. The activations of the resulting model are clustered and used in the second stage of our procedure. Inspired by the results of \citen{sagawa2020distributionally}, we explored using either a standard ERM model or an ERM model with high regularization for this stage, selecting between the two based on the quality of the resulting clustering as measured by the Silhouette score (an unsupervised metric). Below, we detail the ERM hyperparameter settings for each dataset.
\paragraph{U-MNIST.}
Our U-MNIST model is a simple 4-layer CNN, based on a publicly available LeNet5 implementation;\footnote{\url{https://github.com/activatedgeek/LeNet-5}} based on this implementation, we fix the learning rate at 2e-3 and use the Adam optimizer. Each model is trained for 100 epochs. Because the original implementation does not specify a weight decay, we search over weight decay values of $[10^{-3}, 10^{-4}, 10^{-5}]$, and choose the setting with highest average validation accuracy over three trials with different random seeds. Our final hyperparameters are recorded in Table \ref{table:hps}.
\paragraph{Waterbirds.}
Our Waterbirds model uses the \texttt{torchvision} implementation of a 50-layer Residual Network (ResNet-50), initialized with pretrained weights from ImageNet (as done in \citep{sagawa2020distributionally}). For the standard ERM model we use hyperparameters reported by \citep{sagawa2020distributionally}: weight decay of 1e-4, learning rate of 1e-3, SGD with momentum 0.9, and 300 epochs.
For the high-regularization ERM model (which ends up being the one used in Stage 1 of \name), the weight decay and learning rate are 1.0 and 1e-5 respectively (as done in \citep{sagawa2020distributionally} for the high-regularization model).
\paragraph{CelebA.}
Our CelebA model also uses a \texttt{torchvision} pretrained ResNet-50, as done in \citep{sagawa2020distributionally}.
We use the hyperparameters reported by \citep{sagawa2020distributionally}: weight decay of 1e-4, learning rate of 1e-4, SGD with momentum 0.9, and 50 epochs.
For the high-regularization ERM model (which ends up being the one used in Stage 1 of \name), the weight decay and learning rate are 0.1 and 1e-5 respectively (as done in \citep{sagawa2020distributionally} for the high-regularization model).
However, we train on 4 GPUs instead of 1.
(This change does not substantially affect the results; our ERM and subclass-GDRO results are similar to those reported in \citep{sagawa2020distributionally}.)
\paragraph{ISIC.}
Our ISIC model also uses a \texttt{torchvision} pretrained ResNet-50.
Models were trained for 20 epochs using SGD with momentum 0.9 (as done in~\citen{rieger2019interpretations}).
Because these hyperparameters were unavailable in the literature for this architecture and task, we grid searched over weight decay values in [0.01, 0.001, 0.0001] and learning rates in [0.0005, 0.001, 0.005], selecting the values that maximize the overall AUROC on the validation set, averaged over three trials per hyperparameter setting.

Rather than measuring accuracy for ISIC, we use the AUROC (area under the receiver operating characteristic curve), as is standard on this task \citep{rieger2019interpretations} and other medical imaging tasks~\citep{oakden2019hidden}. The specific metric of interest is the worst \emph{per-benign-subclass} AUROC for classifying between that subclass and the malignant superclass (e.g., benign no-patch vs. malignant AUROC). Typically, models designed to attain high AUROC are trained by minimizing the empirical risk as usual. For our \emph{robust} models, we instead minimize $\max\limits_{c \in \text{benign}} \left\{ \f{1}{n_c + n_{\text{malignant}}} \su{x}{} \1(z_i = c \text{ OR } y_i = \text{malignant}) \ell(x_i, y_i; \theta) \right\}$ - in other words, the maximum over all benign subclasses of the ``modified'' empirical risk where all other benign subclasses are ignored. We do this because the worst-case loss over any \emph{benign or malignant} subclass is not necessarily a good proxy for the worst-case per-\emph{benign}-subclass AUROC. Due to the dataset imbalance (many fewer malignant than benign images), standard ERM models attain 100\% accuracy on the benign superclass and much lower accuracy (and higher loss) on the malignant superclass. By contrast, in practice a classification threshold is typically selected corresponding to a target \emph{sensitivity} value.

\subsubsection{Clustering Details}
\label{sec:clustering-details}
We apply a consistent clustering procedure to each dataset, which is designed to encourage discovery of clusters of varied sizes, while still being computationally efficient. We emphasize that while the clustering procedure outlined below yields adequate end-to-end results on our datasets, optimizing this part of the \name~procedure represents a clear avenue for future work. In particular, we use the Silhouette score as a metric to select between feature representations and number of clusters; while this is a serviceable heuristic, it has several flaws (and in the case of BiT embeddings, misleadingly suggests that they are not a suitable representation due to their low Silhouette score).
\begin{enumerate}
\item 
\emph{Dimensionality Reduction}: As recommended by \cite{mcconville2019n2d}, we use UMAP for dimensionality reduction before clustering;
clustering is faster when the data is low-dimensional, and we find that UMAP also typically improves the results. As an alternative to UMAP, we
also use the component of the representation that is orthogonal to the decision boundary, which we refer to as the ``loss component,'' as a single-dimensional representation; this can improve clustering on datasets, especially when performance on certain subclasses is particularly poor (as discussed further in Appendix~\ref{app:gap}).\footnote{We experimented with concatenating the UMAP and loss representations, but found this to reduce performance.} When the loss component is used to identify clusters, we find that applying higher regularization to the initial ERM model further improves clustering quality, as this regularization ``pushes examples further apart'' along the loss direction, and adopt this convention in our experiments.\footnote{The loss component is used for Waterbirds and CelebA (non-BiT version). For both datasets, the weight decay and learning rate used for the high-regularization ERM model are the same as the ones used for the GDRO models on that dataset.}

In each experiment, we select the representation and the number of clusters $k$ based on the parameter setting that achieves the highest average per-cluster Silhouette score. (For all experiments, we set the number of UMAP neighbors to 10 and the minimum distance to 0; further information about these hyperparameters can be found in \citep{mcinnes2018umap}.)

The fact that simply using the ``loss component'' can yield reasonable results is arguably surprising, as this essentially amounts to just picking the examples that the original network got wrong (or closer to wrong than others). Nevertheless, especially on tasks with severe data imbalances and ``spurious features'' (e.g., Waterbirds and CelebA), the rare subclasses do tend to be misclassified at far higher rates, so simply picking the misclassified examples can be a crude but effective heuristic.

\item 
\emph{Global Clustering}: For each superclass, we search over $k \in 2,\dots,10$ to find the clustering that yields the highest average Silhouette score, using the dimensionality reduction procedure identified above.  We similarly perform a search over clustering techniques ($k$-means, GMM, etc.), and find that GMM models achieve high average Silhouette scores most often in our applications. Given that GMM clustering also aligns with our theoretical analysis, we use this approach for all datasets. We refer to this global clustering as $f_{C,G}$.

\item
\emph{Overclustering}: For each superclass, we take the clustering $f_{C,G}$ achieving the highest average Silhouette score, and then split each cluster $c_i$ into $F$ sub-clusters $c_{i1}, \dots, c_{iF}$, where $F$ denotes the ``overclustering factor'' (fixed to 5 for all experiments). For each sub-cluster $c_{ij}$ whose Silhouette score exceeds the Silhouette score of the corresponding points in the original clustering, and which contains at least $s_{min}$ points (for a small threshold value $s_{min}$), the global clustering $f_{C,G}$ is updated to include $c_{ij}$ as a new cluster (and its points are removed from the base cluster $c_i$). The overclustering factor $F$ was coarsely tuned via visual inspection of clustering outputs (without referencing the true subclass labels); the threshold value $s_{min}$ is used to prevent extremely small clusters, as these can lead to instability when training with GDRO and/or highly variable estimates of validation cluster-robust accuracy. (Note: We do not apply overclustering to 1-dimensional representations, as it tends to create strange within-interval splits.)
\end{enumerate}

\subsubsection{Dimensionality Reduction and Clustering: Further Details}
\paragraph{U-MNIST.}
Dimensionality reduction for this dataset used 2 UMAP components and no loss component, as UMAP achieved higher SIL scores.
Our clustering procedure consistently identifies a cluster with a high proportion of the low-frequency ``8'' subclass. As detailed in the main body, we also often observe a small additional cluster with a high concentration of ``7''s written with crosses through the main vertical bar (see Figure \ref{fig:cluster-viz}); performance on this subset is low (below 90\%), which explains why cluster-robust performance actually underestimates the true subclass performance on U-MNIST.
 
\paragraph{Waterbirds.}
Dimensionality reduction for this dataset used only 1 component (the loss component); this significantly outperformed UMAP both in terms of SIL score and final robust performance. We observe that while our procedure does not yield clusters with absolutely high frequencies of the minority classes (as shown in Table~\ref{tab:subclass-recovery}), \name~still identifies clusters with high enough precision (i.e., high enough proportions of the poorly-performing subclasses) such that the second stage of \name~can substantially improve performance on these subclasses.

\paragraph{ISIC.}
Dimensionality reduction for ISIC used 2 UMAP components and no loss component. Patch and non-patch examples lie in different clusters over 99\% of the time. Within the non-patch subclass, on most trials, histopathology examples mostly lie in a different cluster from non-histopathology examples. Similarly, the patch examples often further separate into clusters based on the color of the patch (Figure~\ref{fig:cluster-viz}).

Despite the fact that clustering reveals the non-patch and histopathology subclasses with fairly high fidelity (as also shown in Table~\ref{tab:subclass-recovery}), we do not observe significant improvements in performance on either subset. We hypothesize that this is due to these subsets being ``inherently harder.'' For example, we find that even Subclass-GDRO, which uses the true patch vs. non-patch subclass labels, fails to significantly improve performance on the non-patch subclass compared to ERM, and in fact fails to significantly reduce the training loss on it compared to ERM despite being explicitly trained to do so. This suggests that the issue causing underperformance on these subsets may be due to other factors than the training optimization algorithm (such as model capacity).
 
\paragraph{CelebA.}
Dimensionality reduction for CelebA (without BiT) used only 1 component (the loss component). We observe that clustering does not do a good job of identifying the subclasses of either superclass; thus, it is not surprising that the default version \name~(i.e., without BiT) performs poorly. In fact, \name~performs poorly even compared to the non-ERM baselines. By contrast, \name-BiT does significantly better; the clustering on the (nearly balanced) non-blond superclass attains approximately 95\% accuracy at distinguishing between men and women, and the clustering on the blond superclass also significantly improves over the default version of \name.

\subsubsection{BiT Details}
\label{app:bit-details}
As an alternative to representations from a trained ERM model, we explore the use of BiT embeddings \citen{koleskinov2020bit}, as discussed in Section~\ref{sec:bit}. We use the ResNet-50 version of BiT embeddings; specifically, BiT embeddings are the activations of the penultimate layer of a network pretrained on massive quantities of image data (see \citen{koleskinov2020bit} for more details). The remainder of \name~proceeds the same as usual: the embeddings are clustered and then the cluster assignments are used in the GDRO objective.

For BiT, we experimented with both clustering the BiT embeddings directly (under the hypothesis that the BiT embedding space itself is a good representation), and clustering after dimensionality reduction with UMAP. We found clustering raw embeddings generally performed somewhat better; thus, we show results for clustering the raw embeddings. Due to the high dimensionality of these embeddings (2048-d), we use $k$-means clustering when clustering the BiT embeddings, although the rest of our procedure remains the same.

We find that BiT embeddings significantly improve the end-to-end robust performance results on CelebA; however, they perform worse than the standard version of \name~on all other datasets, indicating that the task-specific information is important for these other tasks to learn a ``good'' representation that can be clustered to find superclasses. Indeed, we find that on these other tasks, the BiT clustering is worse than clustering the activations of the ERM model, in terms of precision and recall at identifying poorly-performing subclasses. [For example, when BiT embeddings are used on MNIST, the ``8''s are never identified as their own cluster.]

Surprisingly, the clustered BiT embeddings uniformly have a much lower Silhouette score than the clustered ERM embeddings, even for CelebA. Thus, our current unsupervised representation and clustering selection technique would not have identified the BiT embeddings as better for CelebA. Improving the representation and clustering selection metric to do a better job at automatically choosing among different representations is an interesting avenue for future work. We note that if a small validation set with \emph{subclass} labels is available, such a set could be used to select between different clusterings by measuring the degree of overlap of the clusters with the true subclasses, as well as used to measure which representation and clustering technique eventually leads to the best validation robust accuracy; however, in general we do not assume any prior knowledge about the subclasses in this work.

\subsubsection{GDRO Training Details}
In the final step of \name, we train a new model (with the same architecture) using the group DRO approach of \cite{sagawa2020distributionally} with weak subclass labels provided by our cluster assignments, and compare to GDRO models trained using (a) superclass labels only (b) random subclass labels and (c) human-annotated subclass labels. Below, we describe the hyperparameter search procedure for each such model and dataset. Unless otherwise stated, all other hyperparameters (batch size, momentum, \# epochs, etc.) are the same as those for ERM.

\paragraph{U-MNIST.}
In the case of U-MNIST, we ran a hyperparameter search over weight decay in [1e-3, 1e-4, 1e-5], and $C$ (the group size adjustment parameter from \cite{sagawa2020distributionally}) in [0, 1, 2].
We find performance to be fairly insensitive to the hyperparameters, so choose weight decay of 1e-5 and $C = 0$ for simplicity and consistency with ERM.

\paragraph{Waterbirds.}
For Waterbirds, we use hyperparameters provided by \cite{sagawa2020distributionally}, so no additional hyperparameter tuning is required.  These hyperparameters are  presented in Table \ref{table:hps}.

\paragraph{CelebA.}
For CelebA, we again use hyperparameters provided by \cite{sagawa2020distributionally}, so no additional tuning is required.  These are presented in Table \ref{table:hps}.

\paragraph{ISIC.}
Each type of ISIC model is hyperparameter searched over the same space as the original ERM model, in addition to searching over group size adjustment parameter $C$ in [0, 1, 2]. We found performance to be fairly insensitive to both. Hyperparameters with highest validation performance were used in the final runs, and are reported in Table \ref{table:hps}.

\subsection{Hyperparameters}
In Table \ref{table:hps}, we present the selected hyperparameters for the final runs of each dataset and method.
{
\begin{table*}[!t]
\small
	\begin{tabular}{|c|c|c|c|c|c|c|}
		\hline
		Dataset & Training Procedure & Epochs & Learning Rate & Batch Size & Weight Decay & Group Adj. Parameter \\
		\hline
		U-MNIST & ERM & 100 & 2e-3 & 128 & 1e-5 & - \\
		U-MNIST & Random-GDRO & 100 & 2e-3& 128 & 1e-5 & 0\\
		U-MNIST & Superclass-GDRO & 100 & 2e-3& 128& 1e-5  & 0\\
		U-MNIST & \name & 100 & 2e-3& 128  & 1e-5 & 0\\
		U-MNIST & Subclass-GDRO & 100 & 2e-3& 128  & 1e-5 & 0\\
		\hline
		Waterbirds & ERM & 300 & 1e-3 & 128  & 1e-4  & -\\
		Waterbirds & Random-GDRO& 300 & 1e-5 & 128  & 1 & 2 \\
		Waterbirds & Superclass-GDRO& 300 & 1e-5  & 128 & 1 & 2 \\
		Waterbirds & \name& 300 & 1e-5  & 128 & 1 & 2 \\
		Waterbirds & Subclass-GDRO& 300 & 1e-5  & 128 & 1  & 2 \\
		\hline
		ISIC & ERM& 20  & 1e-3 & 16& 1e-3  & -\\
		ISIC & Random-GDRO& 20  & 1e-3 & 16& 1e-3  & 2 \\
		ISIC & Superclass-GDRO& 20 & 1e-3 & 16& 1e-3  & 1 \\
		ISIC & \name & 20  & 5e-4 & 16& 1e-3 & 1 \\
		ISIC & Subclass-GDRO& 20  & 5e-4 & 16& 1e-3  & 2 \\
		\hline
		CelebA & ERM& 50 & 1e-4 & 128 & 1e-4  & -\\
		CelebA & Random-GDRO& 50 & 1e-5& 128  & 0.1 & 3 \\
		CelebA & Superclass-GDRO& 50 & 1e-5& 128  & 0.1 & 3 \\
		CelebA & \name& 50 & 1e-5 & 128 & 0.1  & 3 \\
		CelebA & Subclass-GDRO& 50 & 1e-5 & 128 & 0.1 & 3 \\
		\hline
	\end{tabular}
	\caption{Final hyperparameters used in experiments. (Note that for each dataset,
	all \name~runs use the same hyperparameters regardless of whether they use BiT or ERM
	embeddings.)}
	\label{table:hps}
	
\end{table*}
}

\section{Additional Experimental Results}
\label{app:results}
In this section, we provide additional ablation experiments.

\subsection{Visualizing Clusters}
In Figures \ref{fig:cluster-viz} and \ref{fig:celeba-viz}, we visualize the representations returned by $\name$, as well as the clusters it finds and representative examples from each cluster.

\begin{figure*}[!t!]
\centering
\includegraphics[width=\linewidth]{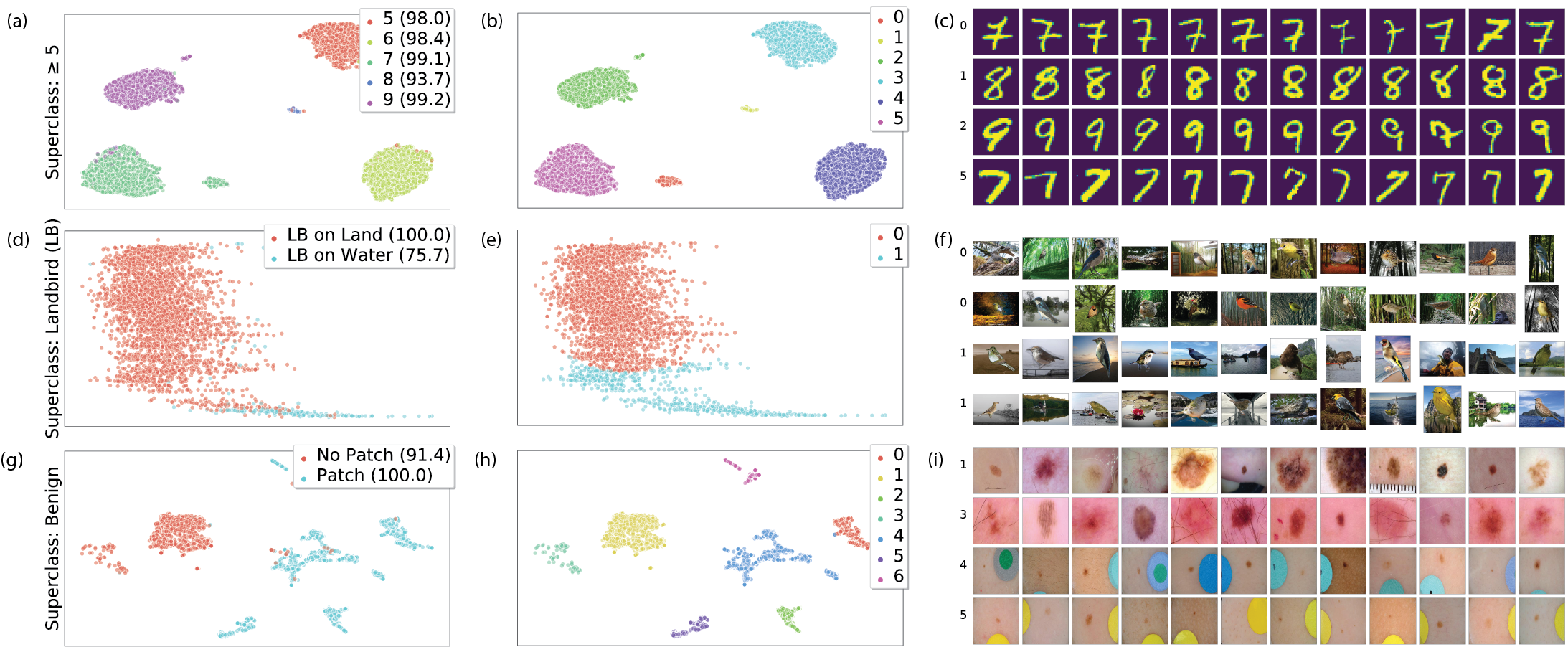}
\caption{True subclasses in the ``feature space'' of a trained ERM model. Left panel
legend colors points by their true subclass, and displays the validation accuracies that the model attains on each subclass.
Middle panel colors points by the cluster index that \name~assigns them. Right panel displays randomly selected examples from each cluster.
Datasets: U-MNIST (row 1), Waterbirds (row 2), and ISIC (row 3). Note that for Waterbirds, the vertical axis is the ``loss component'' and the horizontal axis is the UMAP component.
}
\label{fig:cluster-viz}
\end{figure*}

\begin{figure*}[!t!]
\centering
\includegraphics[width=\linewidth]{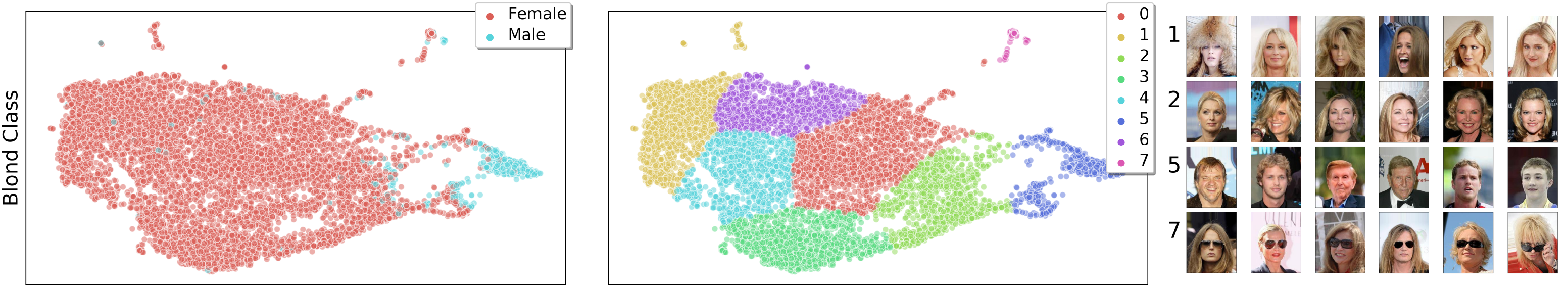}
\caption{True subclasses in \textbf{BiT} embedding space for CelebA,
clusters (middle), and examples from selected clusters (right). 
}
\label{fig:celeba-viz}
\end{figure*}

\subsection{Comparing Cluster-Robust Performance and True Robust Performance}
In addition to the results of Table \ref{tab:measurement} which show that the cluster-robust performance
is a good approximation for the true robust performance, we find that the cluster-robust performance
typically tracks closely with the true robust performance \emph{throughout training}
(with the exception of CelebA without BiT clusters). For example,
Figure~\ref{fig:wb-train-plot} plots the validation cluster-robust accuracy and validation true robust accuracy
from a randomly selected training run on Waterbirds. Both metrics are quite close to each other throughout training
(while the overall accuracy is significantly higher).
\begin{figure}[t]
\centering
\begin{subfigure}{0.45\linewidth}
\centering
\includegraphics[trim=0mm 0mm 0mm 0mm, clip, height=1.75in]{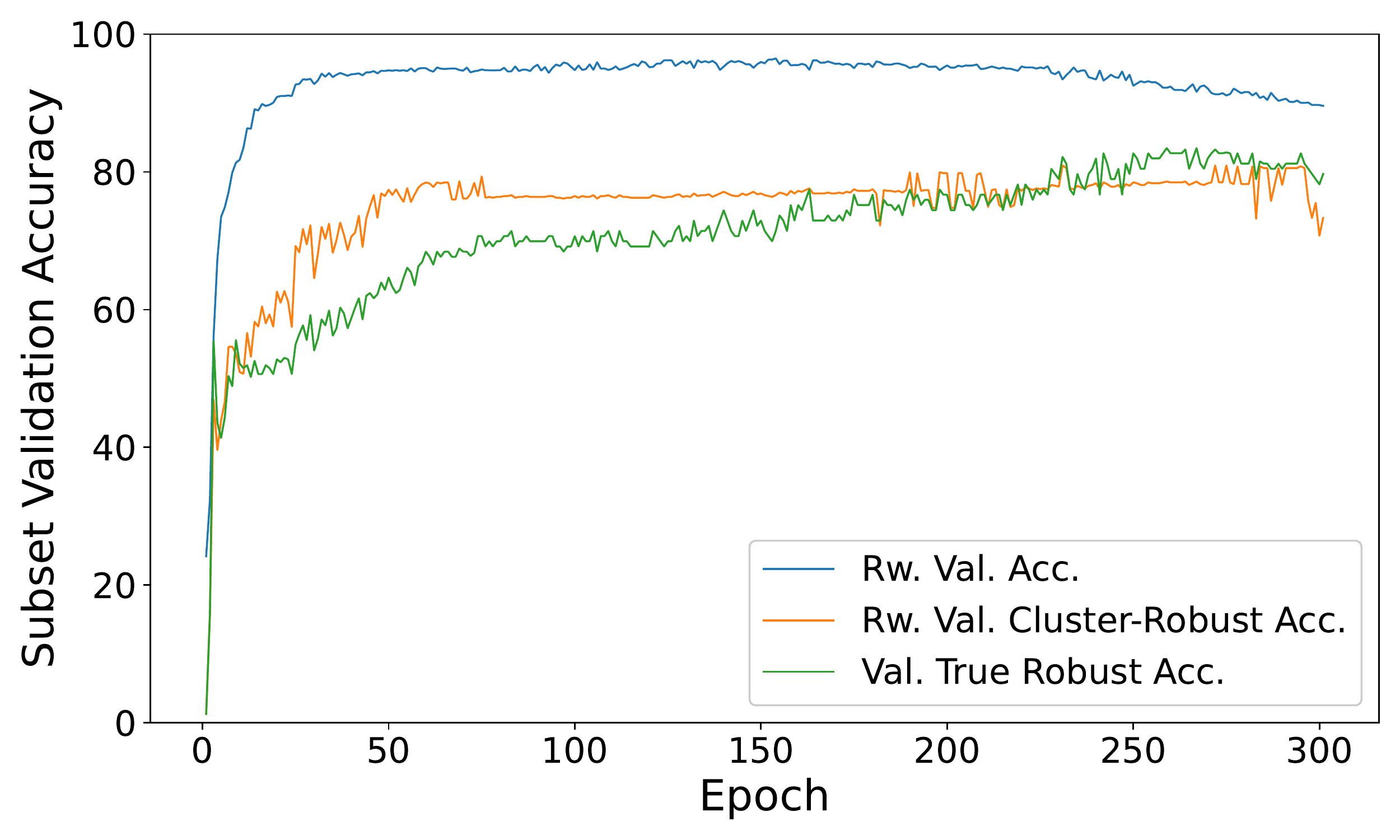}
\label{fig:mnist}
\end{subfigure}\qquad
\begin{subfigure}{0.45\linewidth}
\centering
\includegraphics[trim=0mm 0mm 0mm 0mm, clip, height=1.75in]{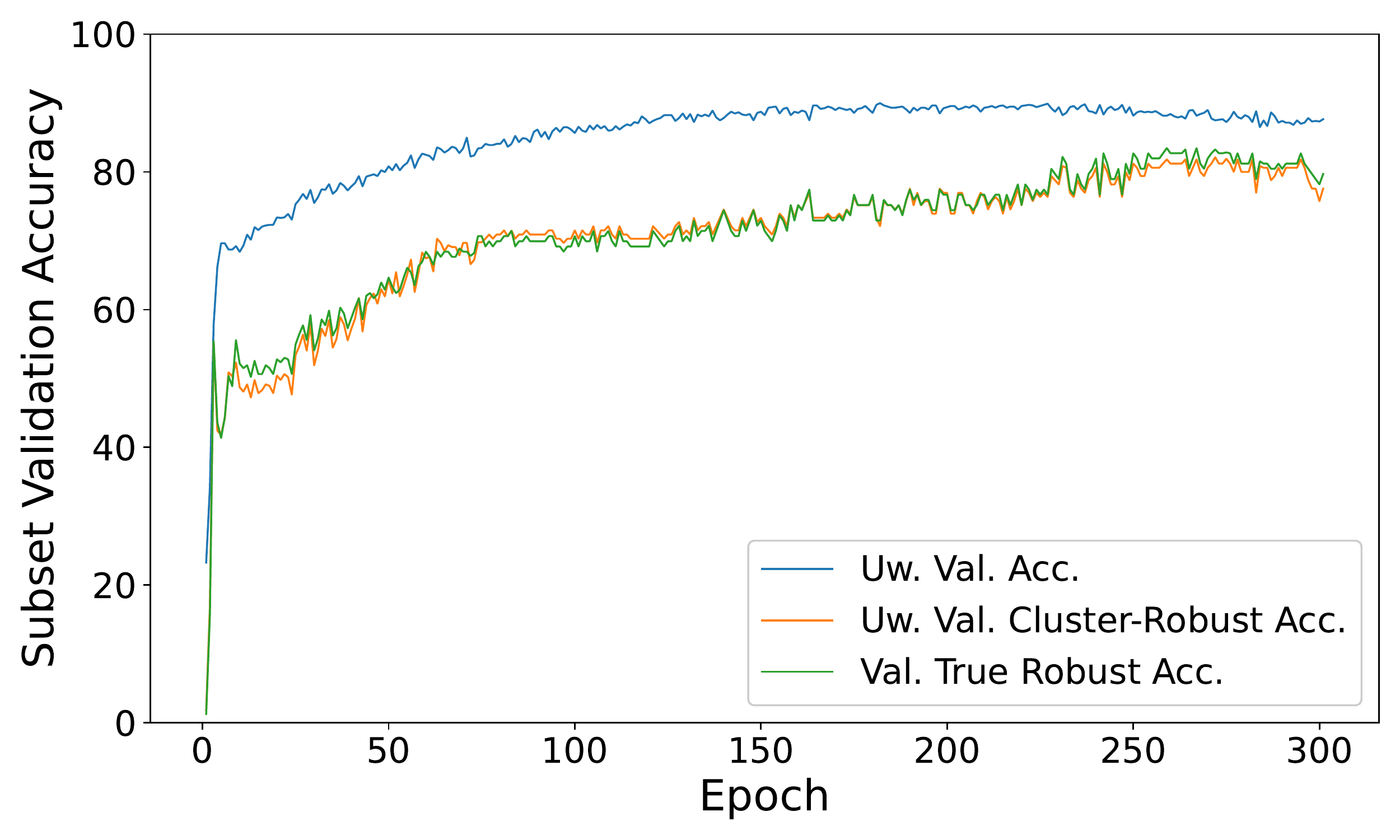}
\label{fig:wb}
\end{subfigure}
\caption{Overall accuracy, worst-case cluster accuracy, and worst-case true subclass accuracy on the validation set
during a randomly selected training run of \name~on Waterbirds [(a) using reweighting for validation set, and (b) not using reweighting]. The worst-case cluster accuracy closely tracks the worst-case true subclass accuracy (especially when reweighting is not applied, whereas the overall accuracy is significantly higher due to the overrepresentation of the ``easier'' subclasses.}
\label{fig:wb-train-plot}
\end{figure}

\subsection{Runtime}
In Table \ref{table:runtime}, we present runtimes for the standard version of \name~broken down by stage.

As the default implementation of \name~involves first training an ERM model, dimensionality-reducing and then clustering its activations,
and then training a ``robust'' model, the total runtime is roughly 2-3$\times$ long as that of simply training an ERM model.
For \name-BiT, no ERM model is trained (and we do not apply dimensionality reduction), so the runtime is just the runtime of the
clustering stage plus the runtime of training the GDRO model (``Step 2''). On our datasets, the total runtime of \name-BiT is less
than 1.5x times the runtime of \name. For instance, on CelebA with BiT embeddings (the largest and most expensive dataset), the
entire clustering stage (including the time taken to compute the BiT embeddings of the datapoints) takes 46 minutes, while the 
time taken to train the ERM model is roughly 2.5 hours. (Clustering the BiT embeddings is more expensive because they are 2048-dimensional.)

Note that the runtime of typical clustering algorithms
scales superlinearly in the number of datapoints;
while the clustering runtime is usually less than the training time for the datasets
we evaluate on, a remedy for larger datasets could be to only use a random subset of the data for clustering (which typically does not
significantly worsen the cluster quality). In addition, we did not attempt to optimize the dimensionality reduction and clustering routines themselves.
As we search over $k$ from 2 to 10 for each superclass, and then overcluster, this is 20 different clusterings in total, along with computing the Silhouette score for each one (which is also expensive as it involves computing pairwise distances). If we instead fixed $k$ (for instance), the total clustering runtime would be less than 7 minutes even for CelebA with BiT embeddings.

The runtime of \name~can be substantially reduced by training the second (robust) model for fewer epochs. On Waterbirds and CelebA,
we can recover over 70\% of the worst-case performance improvement of \name~even
when we limit the total runtime to 1.3$\times$ that of ERM, simply by training for fewer epochs in the second stage. On U-MNIST, if we 
additionally adjust the LR decay schedule so that decay occurs before the end of the shortened training, and fix $k$ to 5 to avoid
the expensive search over $k$ as described above, we can achieve this as well. (On ISIC, the ERM model itself already attains
nearly the same robust AUROC on the histopathology subclass, and higher on the non-patch subclass, than the \name~model.)

\begin{table*}[h]
\centering
	\begin{tabular}{cccc}
		\hline
               Dataset & ERM total runtime & Clustering total runtime & GDRO total runtime \\
                \hline
               U-MNIST & 9m & 6m & 10m \\
               ISIC & 31m & 2m  & 32m \\
               Waterbirds & 90m & 1m & 91m \\
               CelebA & 177m & 17m & 179m \\
                \hline
        \end{tabular}
       \caption{Average runtimes for different stages of \name~(standard implementation, without BiT).
        All runtimes are reported on a machine with 8 CPUs and a single NVIDIA V100 GPU,
        except for CelebA which was run on a machine with 32 CPUs and four NVIDIA V100 GPUs.
       (\name-BiT differs only in the clustering runtime.)}
        \label{table:runtime}
	
\end{table*}

\subsection{Label Noise}
We ran experiments in which a fixed percentage of the data of each subclass was randomly given an incorrect superclass label.
With a minor modification (discarding small clusters), \name~empirically works well in the presence of label noise when the total number of corrupted labels in each superclass is less than the size of the smallest subclass.
Up to this noise threshold, \name~attains +3 points robust accuracy on MNIST and +4 points robust AUROC on ISIC compared to ERM. However, ensuring subgroup-level robustness if there is a larger group of ``wrong'' examples is difficult because differentiating ``real'' subclasses from noise becomes challenging. Thus, we do not consider applying label noise to Waterbirds as the smallest subclass (water-birds on land) is only 1\% of the data; similarly, the smallest subclass on CelebA (blond males) is only 3\% of the data.

In fact, our clustering approach can even be used to help identify incorrectly labeled training examples. First, if a small ``gold'' set of correctly labeled examples is available, the clustering found on the training data could be evaluated on this gold set; clusters consisting of mostly incorrectly labeled training examples should have very few members in the gold set. If such a ``gold'' set is not available, the clusters still allow for much more rapid inspection of the data for incorrect labels, since a few representative examples from each cluster can be inspected instead of a brute-force search through all the training images for incorrectly labeled images. Finally, if one has prior knowledge of the frequency of the rarest subclass in the training data, one can simply discard training examples belonging to poorly-performing clusters smaller than this threshold, treating them as incorrectly labeled.

\subsection{Fixing $k$}
\label{sec:fix-k}
 If the number of clusters $k$ is held fixed (rather than automatically chosen based on Silhouette score), robust performance tends to initially improve with $k$, before decreasing as large values of $k$ cause fragmented clusters that are less meaningful. For example, robust accuracies on U-MNIST using 2, 5, 10, 25, and 100 clusters per superclass are 95.0\%, 96.3\%, 95.9\%, 94.4\%, 90.8\% respectively. We also observe similar trends on the other datasets.

\subsection{Effect of Model Choice on Subclass Recovery}
\label{sec:architecture}
\begin{figure}
\centering
 \includegraphics[height=1.5in]{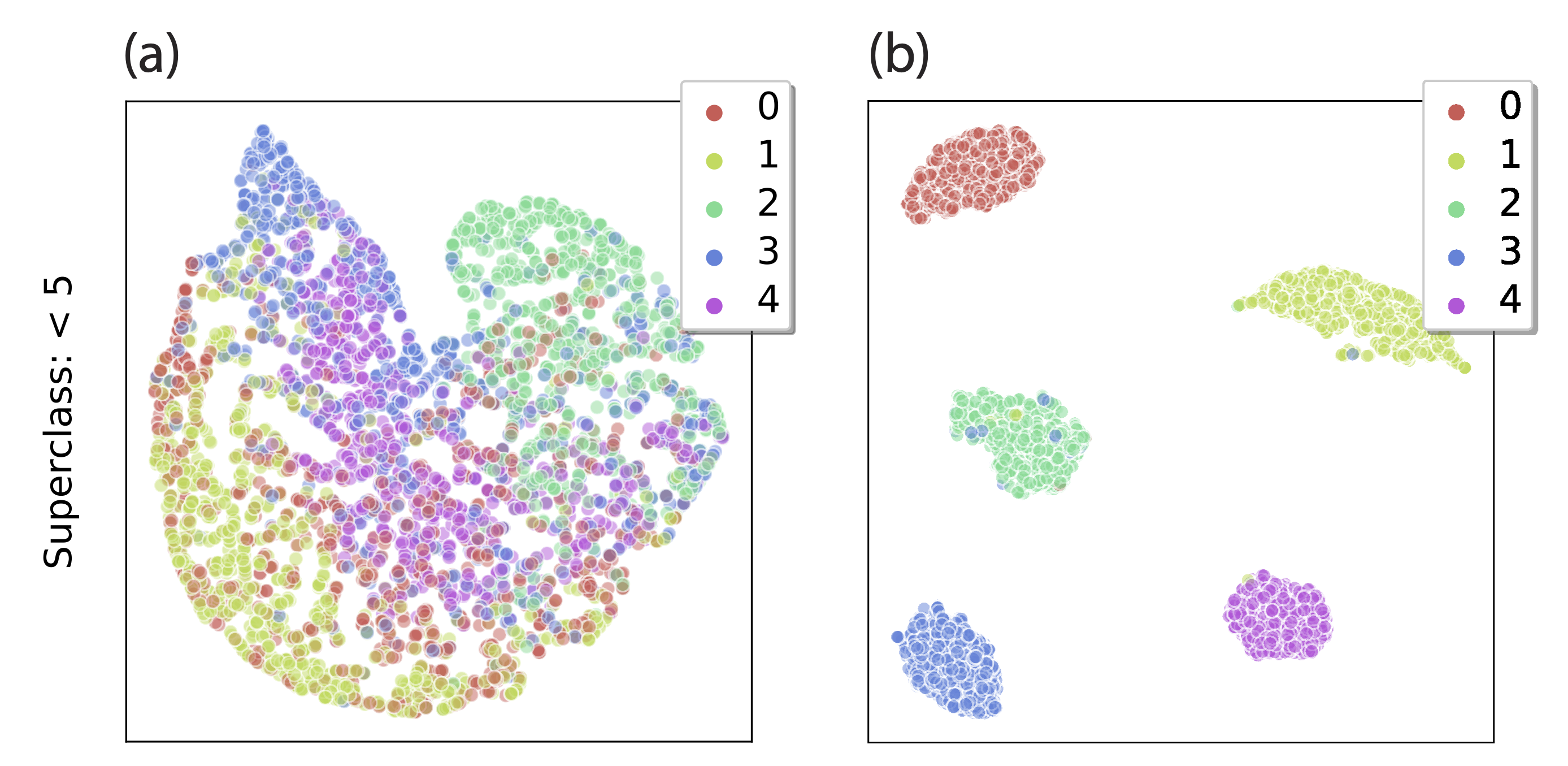}
\caption{LeNet5 vs. LeNet 300-100 activations on U-MNIST with true subclass labels (superclass ``$<$ 5'').
A simple convolutional network (LeNet5, right) separates the true subclasses well in feature space,
while a network consisting only of fully-connected layers (LeNet 300-100, left) does not.}
\label{fig:mnist-architecture}
\end{figure}
As suggested in \Cref{sec:alg-framework}, choosing an appropriate model class $\mathcal{F}$ for the featurizer $f_\theta$ is important.
In particular, $\mathcal{F}$ should ideally contain the inverse of the true generative function $g$,
in order to recover the latent features $\vec{V}$ from the data $X$.
 We demonstrate the importance of model architecture on the ability to separate subclasses
in the model feature space
by comparing the feature representations of two simple networks on a superclass of the U-MNIST dataset
(described in Section~\ref{sec:datasets}).
Figure \ref{fig:mnist-architecture} shows that the choice of model family can strongly affect the learned feature
representation of the initial model and its ability to provide useful information about the subclass.
On this dataset, the feature space of a simple fully connected network (Figure \ref{fig:mnist-architecture}a)
yields substantially less separation between the known subclasses than 
does that of a simple convolutional network (Figure \ref{fig:mnist-architecture}b),
which displays clusters that clearly correspond to semantically meaningful subclasses.

\subsection{Additional Classification Metrics}
In Table~\ref{tab:addl-metrics}, we compare \name~and ERM in terms of both per-subclass averaged accuracy (SCAA) and average precision on the test set.
As expected, \name~slightly decreases average precision, as it trades off some average-case performance for better worst-case performance,
and \name~typically increases per-subclass averaged accuracy (except on U-MNIST, where there is a very slight decrease), due to the fact that it significantly improves performance on poorly-performing subclasses
while only slightly decreasing performance on other subclasses.
\begin{table}[h]
\centering
\small
	\begin{tabular}{lcccccccc}\\
	\toprule 
       \textbf{Method} & \textbf{Metric} & \textbf{Waterbirds} & \textbf{U-MNIST}
       & \textbf{CelebA (BiT)} \\
    \midrule
        \name & SCAA &  86.7 & 97.9 & 90.6
    \\
    & AP & .967 & .998 & .835
		 \\
    \midrule
    ERM & SCAA & 83.8 & 98.2 & 80.5
		\\
		& AP & .984 & .999 & .912
    \\
    \bottomrule
  \end{tabular}
  \caption{Per-subclass averaged accuracy (SCAA) is the mean of the accuracies on each subclass.
  AP denotes the average precision score (which has a maximum of 1).}
    \label{tab:addl-metrics}

\end{table}

In Table \ref{tab:val_metric_avg}, to complement Table \ref{tab:val_metric} we report the robust test accuracies and average test accuracies for models trained with ERM or \name, but where the model checkpoint is selected using the true validation robust accuracy.
\setlength\tabcolsep{3pt}\
\begin{table}[h]
  \vspace{-1em}
	\centerfloat
  \small
	\begin{tabular}{lc|cccccccccc}\\
		\toprule
       \multicolumn{1}{c}{\textbf{Training Method\,\,}} &
        \multicolumn{1}{c}{\textbf{\,Test Metric\,\,}} &
      \multicolumn{1}{c}{\textbf{Waterbirds}} &
      \multicolumn{1}{c}{\textbf{\textsc{U-MNIST}}} &
       \multicolumn{1}{c}{\textbf{{CelebA}}} & 
       \\
    \midrule
		ERM & Robust Acc. & 68.8($\pm$0.9) & 94.5($\pm$0.9) & 46.3($\pm$2.1)
		\\
		& Overall  Acc. & 97.2($\pm$0.1)  & 98.0($\pm$0.2) & 95.1($\pm$0.3)
		 \\
    \midrule
		\name & Robust Acc. & 83.8($\pm$1.0)  & 96.3($\pm$0.5) & 54.9($\pm$1.9)
		\\
		& Overall Acc. & 93.9($\pm$0.8) & 97.9($\pm$0.2)   & 94.5($\pm 0.2$)	
			\\
    \bottomrule
  \end{tabular}
  \caption{Test average performance for each method, where the ``best'' model checkpoint over the training trajectory
  is selected according to the validation robust accuracy (or AUROC in the case of ISIC).}
    \label{tab:val_metric_avg}
\end{table}

\subsection{Empirical Validation of Lemma \ref{lem:unbiased}}
Recall that Lemma \ref{lem:unbiased} says that if we know the true data distribution, we can estimate the per-subclass loss $R_c$ by the quantity $\tilde{R}_c$, a reweighted average of the losses in superclass $S(c)$ based on the ratio of the posterior likelihood of a point given subclass $c$ to its posterior likelihood given superclass $S(c)$; Lemma \ref{lem:unbiased} bounds their difference in terms of the number of datapoints $n$.
In Figure~\ref{fig:lemma1}, we empirically validate Lemma \ref{lem:unbiased} on a synthetic mixture-of-Gaussian example (in which the true data distribution is indeed known). We generate data in dimension $d = 3$ with two subclasses each containing six subclasses each, and compute $R_c$ and $\tilde{R}_c$ for varying numbers of samples $n$ in order to observe the scaling with $n$. We average results over 20 trials; in each trial, new per-subclass distributions are randomly sampled and then new datapoints are sampled. Results are shown in Figure~\ref{fig:lemma1}. As can be observed from the log-log plot, the slope of the line corresponding to the simulated $|\tilde{R}_c - R_c|$ value is very close to that of the predicted rate; the best fit line has a coefficient of $-0.5065$, corresponding to a $O(n^{-0.5065})$ rate, essentially matching Lemma \ref{lem:unbiased}'s predicted $O(n^{-0.5})$ rate.
\begin{figure*}[!t!]
\centering
\includegraphics[height=2in]{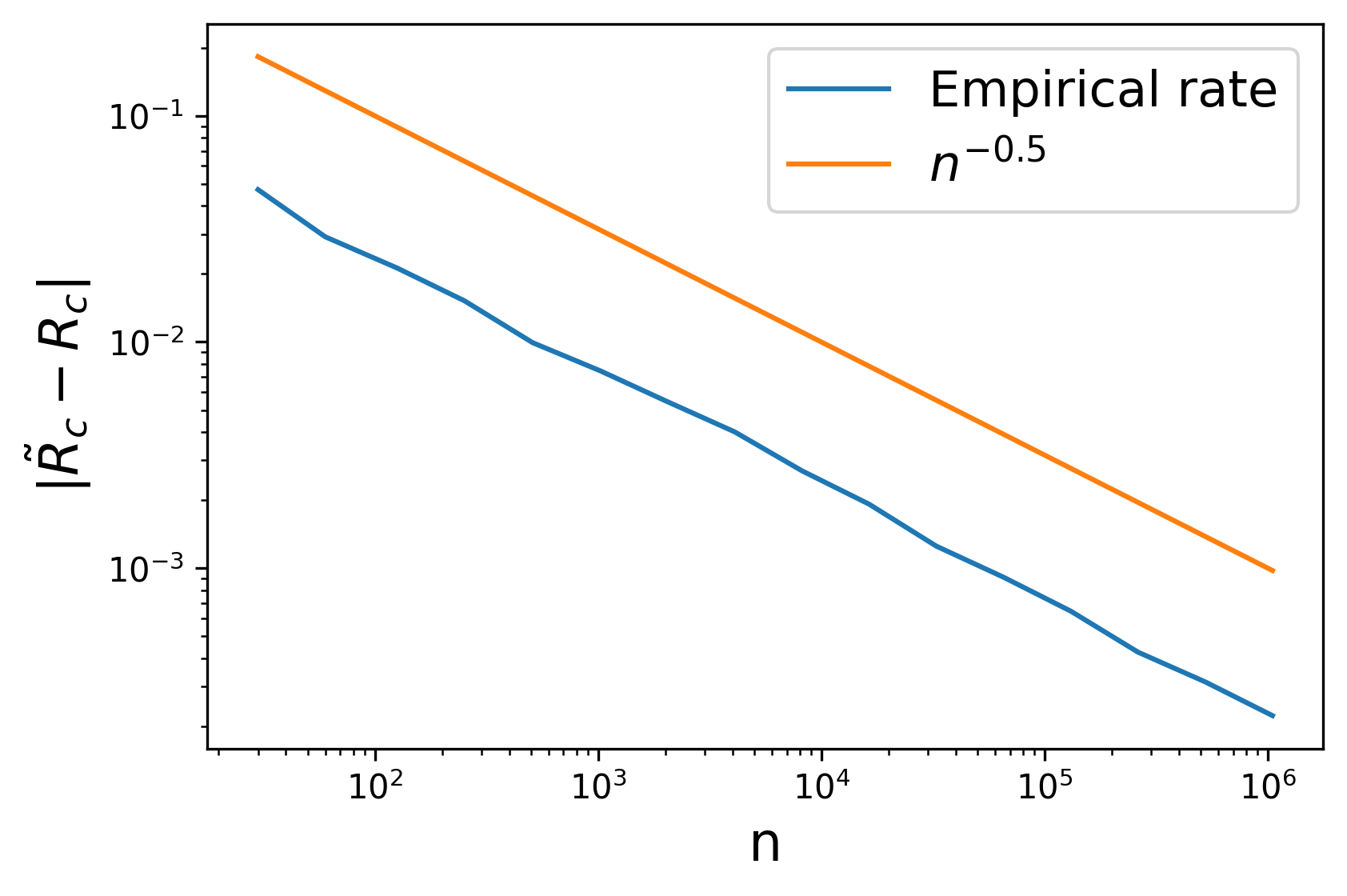}
\caption{Comparison of theoretical and simulated convergence of $\tilde{R}_c - R_c$.}
\label{fig:lemma1}
\end{figure*}

\section{Derivations and Proofs}
\label{app:theory}
\subsection{Analysis of Example \ref{example:erm_fail}}
\label{app:prop1}
We restate Example \ref{example:erm_fail} from \Cref{sec:erm-failure} below:
\paragraph{Example \ref{example:erm_fail}.}
The binary attribute vector $\vec{Z}$ has dimension 2, i.e., $\vec{Z}= (Z_1,Z_2)$, while only $Z_2$ determines the superclass label $Y$, i.e., $Y = Z_2$.
The latent attribute $Z_1$ induces two subclasses in each superclass,
each distributed as a different Gaussian in feature space,
with mixture proportions $\A$ and $1-\A$ respectively.
For linear models with regularized logistic loss, as the proportion $\A$ of the rare subclasses goes to $0$,
the worst-case subclass accuracy of ERM is only $O(\A)$, while that of GDRO is $1-O(\A)$. 

\begin{proof}

Specifically, we consider the following distribution setup: $\vec{Z} \,{\in}\, \{\unaryminus 1, +1\}^2$, 
with ${p(\vec{Z} = (\unaryminus 1, \unaryminus1))} = {p(\vec{Z} = (+1,+1))} = \ff{1-\A}{2}$,
$p(\vec{Z} = (\unaryminus 1,+1)) = p(\vec{Z} = (+1, \unaryminus 1)) = \A/2$,
and $p(V_1 | Z_1) = \N(4Z_1, \A^2)$, $p(V_2 | Z_1,Z_2) = \N(Z_1+3Z_2, \A^2)$,
and the label $Y = h(Z_1,Z_2)$ simply equals $Z_2$. We assume the observed data
$X = (V_1, V_2)$, i.e., the observed data is the same as the ``underlying features'' $\vec{V}$.

Thus, the superclass $Y = -1$ is made up of a ``big'' subclass with distribution $\N((-4, -4), \A^2 \textbf{I})$ and relative
mixture weight $1-\A$ [corresponding to $\vec{Z} = (-1, -1)$], and a ``small'' subclass with distribution $\N((+4, -2), \A^2 \textbf{I})$ and relative mixture weight
$\A$ [corresponding to $\vec{Z} = (+1, -1)$], where $\textbf{I}$ denotes the $2 \times 2$ identity matrix.
The superclass $Y = +1$ is made up of a ``big'' subclass with distribution $\N((+4, +4), \A^2 \textbf{I})$ and relative
mixture weight $1-\A$ [corresponding to $\vec{Z} = (+1,  +1)$], and a ``small'' subclass with distribution $\N((-4, +2), \A^2 \textbf{I})$ and relative mixture weight
$\A$ [corresponding to $\vec{Z} = (-1, +1)$].

For notational simplicity in the following analysis, we will henceforth rename the label $Y = -1$ as $Y = 0$.
The prediction of the logistic regression model on a given sample $(x_1,x_2)$ is $\sigma(w_1 x_1 + w_2 x_2)= \sigma(w^Tx)$, where $\sigma(x) \defeq \log(\ff{1}{1+e^{-x}})$ denotes the sigmoid function and $w_1,w_2$ are the weights of the model. The decision boundary is the line $w^Tx = 0$; examples with $w^Tx < 0$ are classified as $Y = 0$, else they are classified as $Y = 1$.
[For simplicity of exposition, we assume there is no bias term,
and assume that we regularize the norm of the classifier so that $\norm{w}_2 \le R$ for some constant $R$, as changing the parameter norm does not change the decision boundary. Note that neither assumption is necessary, but they serve to simplify the analysis.]

\noindent The logistic loss is the negative log-likelihood, which is $-\su{i}{} \big(y_i \log(\ff{1}{1 + e^{-w^Tx}}) + {(1-y_i)} \log(\ff{1}{1 - e^{-w^Tx}})\big)
= \su{i}{} \big(y_i \log(1 + e^{-w^Tx}) + {(1-y_i)} \log(1 - e^{-w^Tx})\big)$.
Note that by symmetry, the loss on the two subclasses with $Z_1 = Z_2$ is the same, as is the loss on the two subclasses with $Z_1 \ne Z_2$. Therefore, we focus on the class $Y = 1$. The expected average loss on the $Y = 1$ superclass is $\E_{x | y = 1}[-\log(\ff{1}{1 + e^{-w^Tx}})] =
p(Z_1 = 1 | Z_2 = 1)\cdot \\\E_{x | (z_1,z_2) = (1,1)}[\log(1 + e^{-w^Tx})] + p(Z_1 = -1 | Z_2 = 1) \cdot\E_{x | (z_1,z_2) = (-1,1)}[\log(1 + e^{-w^Tx})]$.

\noindent By 1-Lipschitz continuity of the logistic loss and Jensen's inequality,
\begin{align*}
\left| \E_{x | y = 1, z = 1}[\log(1 + e^{-w^Tx})] - \E_{x | y = 1, z = 1}[\log(1 + e^{-w^T(4, 4)})]\right| &\le \\
\E_{x | y = 1, z = 1}[|w^Tx - w^T(4,4)|] &= \\
{\E_{x | y = 1, z = 1}[|w_1 (x_1-4)| + |w_2(x_2-4)|]} &= 
(|w_1| + |w_2|) \E[|b|],\end{align*}
where $b \sim N(0, \A^2)$.
$\E[|b|] = \A\sqrt{2/\pi}$; so, the loss on the $Z_1 = 1$ subclass is bounded in the range
$\log(1 + e^{-w^T(4, 4)}) \pm \A\sqrt{2/\pi}\cdot \norm{w}_2$.

Similarly, the loss on the $Z_1 = -1$ subclass is bounded in the range $\log(1 + e^{-w^T(-4, 2)}) \pm \A\sqrt{2/\pi}\cdot \norm{w}_2$. So, the total loss is bounded in $(1-\A)\log(1 + e^{-w^T(4, 4)}) + \A\log(1 + e^{-w^T(-4, 2)}) \pm  \A\sqrt{2/\pi}\cdot \norm{w}_2$. When $\A$ is sufficiently small, the first term is $\Theta(1)$, while the latter two are $O(\A)$ (under the assumption that $\norm{w}_2$ is bounded). For a fixed value of $\norm{w}_2$, the first term is minimized when $w / \norm{w}_2 = (\ff{1}{\sqrt{2}}, \ff{1}{\sqrt{2}})$, so that $w^T(4, 4)$ is as large as possible. A $\Theta(\A)$-scale perturbation to the direction $w /\norm{w}_2$ results in an increase of $\Theta(\A)$ to the quantity $(1-\A)\log(1 + e^{-w^T(4, 4)})$. Thus, whenever $\A$ is sufficiently small, $w /\norm{w}_2$ must be $ (\ff{1}{\sqrt{2}}, \ff{1}{\sqrt{2}}) + O(\A)$ in order to minimize the loss subject to the $\norm{w}_2 \le R$ constraint. In other words, the regularized ERM solution converges to $(w_1, w_2) = (\ff{1}{\sqrt{2}}, \ff{1}{\sqrt{2}})$ as $\A \downarrow 0$.

For the $Z_1 = -1$ subclass, $w^Tx$ is a normal random variable with mean $-4w_1 + 2w_2$ and variance $\A \norm{w}_2^2$. When $\A$ is sufficiently small and $w / \norm{w}_2 = (\ff{1}{\sqrt{2}}, \ff{1}{\sqrt{2}}) + O(\A)$, the quantity $-4w_1 + 2w_2$ is negative with magnitude $O(1)$---and thus, since examples with $w^Tx < 0$ are classified as $Y = 0$, this means that for sufficiently small $\A$ the fraction of the subclass $Z_1 = -1$ classified correctly as $Y = 1$ is only $O(\A)$.

By contrast, the GDRO solution minimizes the maximum per-subclass loss. Since each subclass has the same covariance
$\A^2 \textbf{I}$, the GDRO decision boundary is the line that separates the superclass means and has maximum distance to any subclass mean. After normalization to have $\norm{w}_2 = 1$, this is the line $(-\ff{1}{\sqrt{5}}, \ff{4}{\sqrt{5}})$; the true solution will be some multiple of this (depending on $\A$ and $R$), giving rise to the same boundary. As $\A \downarrow 0$, the accuracy of this decision boundary is $1-O(\A)$, since the variance of each subclass is $O(\A^2 \textbf{I})$.
\end{proof}

\subsection{Proofs from Section~\ref{sec:analysis}}
\label{app:analysis-proofs}
\subsubsection{Proof of Lemma \ref{lem:unbiased}}
\unbiasedness*
\begin{proof}
Define $\#_{y=k} \defeq \su{i=1}{n} \1(y_i = k)$ and $\#_{z=c} \defeq \su{i=1}{n} \1(z_i = c)$.
Using this notation, \\$R_c = \f{1}{\#_{z=c}} \su{i=1}{n} \1(z_i = c) \ell(f(x_i), S(c))$,
and $\tilde{R}_c = \f{1}{\#_{y=S(c)}} \su{i=1}{n}  \f{p(x_i | z_i = c)}{p(x_i | y_i = S(c))}  \1(y_i = S(c))\ell(f(x_i), S(c))$.
\\ First, observe that $\E[\tilde{R}_c] = \E[R_c]$:
the expectation of each term in the summation defining $\E[\tilde{R}_c]$ is
\\
$\E_{x \sim \mathcal{P}(\cdot | y = S(c))} \Lt \f{p(x | z = c)}{p(x | y = S(c))} \ell(f(x), S(c)) \Rt = {\I{\R^d}{}\f{p(x | z = c)}{p(x | y = S(c))} \ell(f(x), S(c)) p(x | y = S(c)) \, \dx} =
\\
\I{\R^d}{} p(x | z = c)\ell(f(x), S(c)) \, \dx = \E_{x \sim \mathcal{P}(\cdot | z = c)}[\ell(f(x), S(c))] = \E[R_c]$, and so $\E[\tilde{R}_c] = \E[R_c]$.

Now, note that
\begin{align*}
0 &\le \f{p(x | z = c)}{p(x | y = S(c))} = \f{\ff{p(x, z=c)}{p(z=c)}}{\ff{p(x, y=S(c))}{p(y=S(c))}} = \f{\ff{p(x, y=S(c), z=c)}{p(y=S(c), z=c)}}{\ff{p(x, y=S(c))}{p(y=S(c))}} \\
&=
\f{p(x, y=S(c), z=c)}{p(x, y=S(c))} \cdot \f{p(y=S(c))}{p(y=S(c), z=c)}
\\
&\le
\f{p(y=S(c))}{p(y=S(c), z=c)} 
= \f{1}{p(z = c | y = S(c))} \le \f{1}{\pi_{\min}},
\end{align*}
where $\pi_{\min}$ denotes the minimum true subclass proportion (i.e., $\pi_{\min} = \min_c p(z = c)$).

Thus, assuming $\ell$ is bounded, each term in the summation defining $\tilde{R}_c$ is bounded, and has mean $\E[R_c]$ as argued above; applying Hoeffding's inequality (to bound the probability that a sum of bounded random variables deviates from its mean by more than a specified amount) yields that $\left|\tilde{R}_c - \E[R_c] \right| \le O\lt \ff{1}{\sqrt{n}}\rt$ with high probability. Similarly, applying Hoeffding's inequality to $R_c$ yields that $\left|R_c - \E[R_c] \right| \le O\lt \ff{1}{\sqrt{n}}\rt$ with high probability.
\end{proof}

\subsubsection{Proof of Theorem~\ref{thm:l3}}
We restate Theorem~\ref{thm:l3} below:
\gaussian*

Recall that we define $\hat{R}_c$ to be the same as $\tilde{R}_c$ except with weights $\hat{w}(x, c)$ computed from $\hat{\P}$.
More precisely, $\hat{R}_c \defeq \f{1}{\#_{y=S(c)}} \su{i=1, y_i = S(c)}{n} \hat{w}(x,c) \ell(f(x_i), S(c))$, where
$\hat{w}(x, c) \defeq \f{\hat{p}(x | z = c)}{\hat{p}(x | y = S(c))}$.
[In Appendix~\ref{app:gdro-soft}, we provide an efficient algorithm to minimize $\hat{R}_{\text{robust}} = \max_c \hat{R}_c$.]

Our strategy to prove Theorem~\ref{thm:l3} will be as follows.
First, we prove a general statement (that holds regardless of the form of the true data distribution)
that relates the total variation (TV) between the true
per-subclass distributions and the estimated per-subclass distributions
to the difference between $\tilde{R}_c$ and the ``perturbed'' loss $\hat{R}_c$
(Lemma~\ref{lem:l2}).
Next, we bound the total variation between the true and estimated per-subclass distributions
in the mixture-of-Gaussians case, using the Gaussian mixture learning algorithm from \citen{ashtiani2018near}.
Finally, we use standard uniform convergence-type results to yield the final high probability bound on the robust risk of the
returned model $\hat{f}$.

\setcounter{secnumdepth}{4}

\titleformat{\paragraph}
{\normalfont\normalsize\bfseries}{\theparagraph}{1em}{}
\titlespacing*{\paragraph}
{0pt}{3.25ex plus 1ex minus .2ex}{1.5ex plus .2ex}

\paragraph{Total variation estimation error to error in loss}
First, we show that within each superclass, if the per-subclass distributions are estimated well, then the per-subclass estimated risks will be close to the true per-subclass risks.
\begin{restatable}{lem}{tv}
\label{lem:l2}
Let $\pi_{\min}$ be the minimum true subclass proportion and $\hat{\pi}_{\min}$ be the minimum estimated subclass proportion. Suppose $\ell$ is globally bounded by $M$.
Suppose that we have estimated superclass-conditional distributions $\hat{\P}(x | y)$ such that
for all superclasses $b \in [B]$,
$\TV({\P}(x | y = b), \hat{\P}(x | y = b)) \le \ep$.
Additionally suppose we have estimated subclass-conditional distributions $\hat{\P}(x | z)$
such that for all subclasses $c \in [C]$,
$\TV({\P}(x | z = c), \hat{\P}(x | z = c)) \le \ep$.
Then, $|\hat{R}_{c} - \tilde{R}_c| \le \ff{3M}{\hat{\pi}_{\min}} \cdot \ep + O(\ff{1}{\sqrt{n}})$ with high probability.
\end{restatable}

\begin{proof}
First, we bound $|\hat{w}(x, c) - w(x,c)|$ using the triangle inequality:
\begin{align*}
& \,|\hat{w}(x, c) - w(x,c)| \\
\le &\left| \ff{\hat{p}(x | z = c)}{\hat{p}(x | y = S(c))} - \ff{{p}(x | z = c)}{{p}(x | y = S(c))} \right| \\
= &\,\ff{1}{{p}(x | y = S(c))} \Big| \ff{p(x | y = S(c)) }{\hat{p}(x | y = S(c))} \cdot {\hat{p}(x | z = c)} - {p(x | z = c)} \Big| \\
\le &\,\ff{1}{{p}(x | y = S(c))} \lt \left| \ff{p(x | y = S(c)) }{\hat{p}(x | y = S(c))} \cdot \hat{p}(x | z = c) - \hat{p}(x | z = c) \right| + \left|\hat{p}(x | z = c) - p(x | z = c)\right| \rt \\
= &\,\ff{1}{{p}(x | y = S(c))} \lt \hat{p}(x | z = c) \left| \ff{p(x | y = S(c)) }{\hat{p}(x | y = S(c))} - 1 \right| + \left|\hat{p}(x | z = c) - p(x | z = c)\right| \rt
\end{align*}
\noindent By definition, $\hat{p}(x | y = b) = \su{c \in S_b}{} \hat{p}(z = c | y = b) \hat{p}(x | z = c)$, so $\hat{p}(x | z = c) \le \f{\hat{p}(x | y = S(c))}{\hat{p}(z = c | y = S(c))}$. Thus,
\begin{align*}
\hat{p}(x | z = c) \left| \ff{p(x | y = S(c)) }{\hat{p}(x | y = S(c))} - 1 \right|  &\le
\ff{\hat{p}(x | y = S(c))}{\hat{p}(z = c | y = S(c))} \left| \ff{p(x | y = S(c)) }{\hat{p}(x | y = S(c))} - 1 \right| \\
&=
\ff{1}{\hat{p}(z = c | y = S(c))} \left|\hat{p}(x | y = S(c)) - p(x | y = S(c))  \right| \\
&\le \ff{1}{\hat{\pi}_{\text{min}}} \left|\hat{p}(x | y = S(c)) - p(x | y = S(c))  \right|,
\end{align*}
as by definition $\hat{p}(z=c | y = S(c)) \ge \hat{p}(z=c) \ge \hat{\pi}_{\min}$.
So,
\begin{flalign*}
&\E_{x \sim \P(\cdot | y = S(c))} \Lt \left|\hat{w}(x, c) - w(x,c)\right| \Rt \le \\
& \E_{x \sim \P(\cdot | y = S(c))} \left[\ff{1}{{p}(x | y = S(c))} \cdot \lt \ff{1}{\hat{\pi}_{\min}} |\hat{p}(x | y = S(c)) - p(x | y = S(c))| + |\hat{p}(x | z = c) - p(x | z = c)| \rt \right] = \\
&
{
\I{\R^d}{} \f{\ff{1}{\hat{\pi}_{\min}} \lt |\hat{p}(x | y = S(c)) - p(x | y = S(c)) | +
|\hat{p}(x | z = c) - p(x | z = c)| \rt}{{p}(x | y = S(c))} \cdot {p}(x | y = S(c)) \, \dx \le
}
\\
&\f{1}{\hat{\pi}_{\min}}
\I{\R^d}{} |\hat{p}(x | y = S(c)) - p(x | y = S(c)) | \, \dx +
\I{\R^d}{} |\hat{p}(x | z = c) - p(x | z = c)| \, \dx \le \\
&\f{2\ep}{\hat{\pi}_{\min}} + 2\ep \le \f{3\ep}{\hat{\pi}_{\min}},
\end{flalign*}
since $\I{\R^d}{}  \left|\hat{p}(x | y = S(c)) - p(x | y = S(c))\right| \, \dx = {2TV(\hat{\P}(x | y = (c)), \hat{\P}(x | y = S(c)))} \le 2\ep$ by assumption, and similarly $\I{\R^d}{}  \left|\hat{p}(x | z = c) - p(x | z = c)\right| \, \dx = {2TV(\hat{\P}(x | z = c), \hat{\P}(x | z = c))} \le 2\ep$. (The final inequality above uses the fact that there are at least  two subclasses (i.e., $C \ge 2$), so $\hat{\pi}_{\min} \le 1/2$.)

Now, $|\hat{R}_{c} - R_c| \le \ff{1}{\#_{y= S(c)}} \su{i: \, y_i = S(c)}{} |\hat{w}(x_i, c_i)-w(x_i, c_i)| \ell(f(x_i), S(c))$. The expectation of each term in the summation is $\le \ff{3M}{\hat{\pi}_{\text{min}}} \ep$, since the loss is globally bounded by $M$.
Finally, applying Hoeffding's inequality yields that $|\hat{R}_{c} - \tilde{R}_c| \le \ff{3M}{\hat{\pi}_{\text{min}}} \ep + O\lt \ff{1}{\sqrt{n}}\rt$ with high probability. (Note that this result holds for a \emph{fixed} prediction function $f$.)

\end{proof}

\subsubsection*{Total variation in estimated per-subclass distributions: Gaussian case}
\cite{ashtiani2018near} provides an algorithm for estimating mixtures of Gaussians;
they show that $\tilde{O}(kd^2 / \ep^2)$ samples are sufficient to learn a mixture of $k$ $d$-dimensional Gaussians to within error $\ep$ in
total variation. Concretely, given $n$ samples from a mixture-of-Gaussian distribution $\P$ and given the true number of mixture components $k$,
the algorithm in \citep{ashtiani2018near} returns a $k$-component mixture-of-Gaussian distribution $\hat{\P}$
such that $TV(\P, \hat{\P}) \le \tilde{O}(\sqrt{1 / n})$.
To prove Theorem \ref{thm:l3}, we use this result and bound the overall total variation error in terms of the maximum per-component total variation error. The proof depends on the key lemma (Lemma \ref{lem:tv_mixture}) stated below (whose proof appears at the end of this section).
We then apply Lemma~\ref{lem:l2} to translate this total variation error bound to a bound on the robust loss of the end model.

In order to apply Lemma~\ref{lem:l2}, we first need to relate the total variation error $\ep$ between the mixtures to the
total variation error between the individual mixture components; we show that
when $\ep$ is small enough, then the total variation error between
corresponding mixture components is ${O}(\ep)$ as well.
We state this formally in Lemma \ref{lem:tv_mixture} (proved later in this section).

\begin{restatable}{lem}{tvmixture}
\label{lem:tv_mixture}
Let $\P$ and $\hat{\P}$ be two $k$-component Gaussian mixtures, and suppose the $k$ components of $\P$, denoted by $p_1, \dots, p_k$, are \emph{distinct} Gaussian distributions and all have nonzero mixture weights $m_1, \dots, m_k$. Similarly denote the $k$ components of $\hat{\P}$ by $\hat{p}_1, \dots, \hat{p}_k$, with mixture weights $\hat{m}_1, \dots, \hat{m}_k$.
There exists a constant $c(\P)$ depending only on the parameters of $\P$ such that for all sufficiently small $\ep > 0$, whenever $TV(\P, \hat{\P}) \le \ep$ there exists some permutation $P : [k] \ra [k]$ such that $\max\limits_{c \in [k]} \, TV(p_c, \hat{p}_{P(c)}) \le c(\P)\cdot \ep$.
\end{restatable}

In addition, we use the following standard result from learning theory \citep{229t}
to relate the minimizer of the estimated robust training loss $\hat{R}_{\text{robust}}$
to the minimizer of the true robust training loss $R_{\text{robust}}$.

\begin{lem}
\label{lem:learning}
Suppose $g(\cdot, \cdot) \in [-B, B]$. Let $f(\theta) \defeq \E_{(x,y) \sim P} [g(x,y;\theta)]$ and let $\hat{f}(\theta) \defeq \f{1}{n} \su{i=1}{n} g(x_i, y_i; \theta)$, where $\{(x_i,y_i)\}_{i=1}^n$ are IID samples from $P$. Suppose $g(\cdot, \cdot \theta)$ is $L$-Lipschitz w.r.t. $\theta$, so $f(\theta), \hat{f}(\theta)$ are $L$-Lipschitz. Then with probability $\ge 1 - O(e^{-p})$, we have
\begin{equation} \label{eq:uniform_convergence} \forall \theta \text{ s.t.} \norm{\theta} \le R, \quad |\hat{f}(\theta) - f(\theta)| \le O\lt B\sqrt{\f{p \log(nLR)}{n}}\rt\end{equation}
\end{lem}

\subsubsection*{Theorem \ref{thm:l3} Proof}
Using the preceding lemmas, we will now prove Theorem \ref{thm:l3}.

\begin{proof}
Note that $\hat{\P}(x | y = b)$ is itself a mixture of Gaussians for each superclass $b$
(it is simply a mixture of the Gaussians corresponding to the subclasses in $S_b$).
Thus, for each $b \in [B]$ we estimate $\hat{\P}(x | y = b)$
using the Gaussian mixture learning algorithm from \citep{ashtiani2018near}. 
With high probability, this returns an $|S_b|$-component mixture-of-Gaussian distribution
$\hat{\P}(x | y = b)$ such that $\TV({\P(x | y = b)}, \hat{\P}(x | y = b)) \le \tilde{O}(\sqrt{1 / n})$.
By Lemma~\ref{lem:tv_mixture}, if $n$ is large enough, then this means that
there exists a permutation $P: [S_b] \ra [S_b]$ such that for each subclass $c \in S_b$,
$\TV(\P(x | z = c), \hat{\P}(x | z = P(c))) \le \tilde{O}(1/\sqrt{n})$.
So, applying Lemma~\ref{lem:l2}, we have that for a fixed prediction function $f$,
$|\hat{R}_{P(c)} - \tilde{R}_c| \le \tilde{O}(1/\sqrt{n})$, for all subclasses $c \in S_b$ (and for all superclasses $b \in [B]$).
By Lemma~\ref{lem:unbiased} and triangle inequality, $|\hat{R}_{P(c)} - {R}_c|$ is $\tilde{O}(1/\sqrt{n})$ as well.
So, we have that (for fixed $f$) for all subclasses $c \in [C]$, $|\hat{R}_{P(c)} - {R}_c|$ is $\tilde{O}(1/\sqrt{n})$ with high probability.
Thus, for any given $f$, $|\hat{R}_{\text{robust}}(f) - R_{\text{robust}}(f)| =
\left|\max\limits_{b \in [B]} \max\limits_{c \in S_b} \hat{R}_{c} - \max\limits_{b \in [B]} \max\limits_{c \in S_b} R_{c}\right| = \tilde{O}(1/\sqrt{n})$ with high probability.

Let $R^*_{\text{robust}}(f)$ denote the true \emph{population} robust loss. Then Lemma~\ref{lem:learning} says that $|R_{\text{robust}}(f) - R^*_{\text{robust}}(f)|$ with high probability for all $f$ in the hypothesis class $\mathcal{F}$.
Similarly, we can use an analogous uniform convergence result to show that $|\hat{R}_{\text{robust}}(f) - R_{\text{robust}}(f)| \le \tilde{O}(1/\sqrt{n})$ for all $f$ in the hypothesis class with high probability. Thus, by triangle inequality and union bound,
$|\hat{R}_{\text{robust}}(f) - R_{\text{robust}}^*(f)| \le \tilde{O}(1/\sqrt{n})$ for all $f$ in the hypothesis class with high probability, and in particular this holds for the minimizer $\hat{f}$ of $\hat{R}_{\text{robust}}$.

Thus, under the given assumptions,
the excess robust generalization risk (i.e., worst-case subclass generalization risk) of the \name~model $\hat{f}$ is $\tilde{O}(1/\sqrt{n})$ [which is near-optimal in terms of sample complexity, since $\Omega(1/\sqrt{n})$ is a generic worst-case lower bound even if the subclass labels are known].

Note that a technical requirement of the above argument is that the samples we use to estimate $\hat{\P}$ should be independent from those we use to compute the robust loss; for this to hold, we may randomly sample half of the examples to learn the distribution $\hat{\P}$ (and its mixture components), and then use the other half to minimize the robust loss.
This does not change the asymptotic dependence on the number of samples $n$.
[In practice, however, we use all examples in both phases, to get the most out of the data.]
\end{proof}

\subsubsection*{Lemma~\ref{lem:tv_mixture} Proof}
\newcommand{\bigsig}{\mathbf{\Sigma}}
Before we prove Lemma~\ref{lem:tv_mixture}, we first provide a simple lemma bounding the total variation distance of two Gaussians in terms of the Euclidean distance between their parameters, directly based on the results from \citep{devroye2018total}.
\begin{lem}
\label{lem:tv_simple}
Let $p$ be a $d$-dimensional Gaussian with mean $\mu$ and full-rank covariance matrix $\bigsig \in \R^{d \times d}$. Let $p'$ be another Gaussian with mean $\mu'$ and covariance $\bigsig'$. Then there exists a constant $c(\mu,\bigsig)$ [i.e., depending only on the parameters of $p$] such that for all sufficiently small $\ep > 0$, whenever $\norm{\mu-\mu'}_2 \le \ep$ and $\norm{\bigsig- \bigsig'}_F \le \ep$ it is the case that $TV(p, p') \le c(\mu,\bigsig) \cdot \ep$.
\end{lem}
\begin{proof}
The one-dimensional case is shown in Theorem 1.3 of \citen{devroye2018total}. The higher-dimensional case follows from Theorems 1.1 and 1.2 of \citen{devroye2018total}. Note that the constant $c$ does not depend on $\ep$, although it may depend on $d$.
\end{proof}

\noindent For convenience, we restate Lemma~\ref{lem:tv_mixture} below.
\tvmixture*

\begin{proof}
The case $k = 1$ is vacuous (since the ``mixture'' is simply a single Gaussian); so, suppose $k > 1$.
Denote the mixture weights of the true distribution $\P$ as $m_1, \dots, m_k$, and the mean and covariance parameters of the individual distributions in $\P$ as $\mu_1, \dots, \mu_k, \bigsig_1, \dots, \bigsig_k$.
In other words, for $x \in \R^d$, $\P(x) = \su{i=1}{k} m_i \N_{\mu_i, \bigsig_i}(x)$, where $m_i \in (0, 1)$ and $\N_{\mu, \bigsig}$ denotes the normal density with mean $\mu$ and covariance $\bigsig$.
Similarly, denote the mixture weights of the estimated distribution $\hat{\P}$ by $\hat{m}_1, \dots, \hat{m}_k$, and the mean and covariance parameters of $\hat{\P}$ by $\hat{\mu}_1, \dots, \hat{\mu}_k, \hat{\bigsig}_1, \dots, \hat{\bigsig}_k$.
For simplicity assume the covariance matrices $\bigsig_i$ of each component in the true distribution are strictly positive definite (although this is not required).

Define $q({m}_1', \dots, {m}_k', {\mu}_1', \dots, {\mu}_k', {\bigsig}_1', \dots, {\bigsig}_k') = \I{\R^d}{} \left| \su{i=1}{k} m_i \N_{\mu_i, \bigsig_i}(x) - \su{i=1}{k} {m}_i' \N_{{\mu}_i', {\bigsig}_i'}(x) \right| \dx$. The domain of $q$ is constrained to have ${m}_i' \in [0,1]$ for all $1 \le i \le k$ and $\su{i=1}{k} {m}_i' = 1$, as well as to have ${\bigsig}_i'$ be SPD. By definition, $q(\hat{m}_1, \dots, \hat{m}_k, \hat{\mu}_1, \dots, \hat{\mu}_k, \hat{\bigsig}_1, \dots, \hat{\bigsig}_k)$ is simply twice the total variation between $\P$ and $\hat{\P}$.

Note that, since we assumed the mixture components are unique and $m_i \ne 0$ for all $i$, the only global minima of $q$ [where $q$ evaluates to $0$, which means that $\P$ and $\hat{\P}$ are the same distribution]
are where $(m_{\pi(i)}', \mu_{\pi(i)}', \bigsig_{\pi(i)}') = (m_i, \mu_i, \bigsig_i)$ for all $1 \le i \le k$, for some permutation $\pi$---in other words, when the two distributions have the exact same mixture components and mixture weights up to permutation.
Note that $q$ is continuous on its domain.
Further, it is not hard to see that the $\ep$-sublevel sets of $q$ are compact for sufficiently small $\ep$,
and therefore $\lim\limits_{\ep \ra 0} \{({m}_1', \dots, {\mu}_1', \dots, {\bigsig}_1', \dots):q({m}_1', \dots, {\mu}_1', \dots, {\bigsig}_1', \dots) \le \ep\}$ is exactly the set of global minima of $q$.
Thus, for a fixed distribution $\P$, as $\ep \ra 0$, the set of points such that $q({m}_1', \dots, {\mu}_1', \dots, {\bigsig}_1', \dots) \le \ep$ is contained in the union of sets of $\infty$-norm radius $O(\de(\ep))$ around each of the global minima of $q$, where $\de(\ep) \ra 0$ as $\ep \ra 0$.
In other words, when $\ep$ is sufficiently small then the set of all Gaussian mixtures $\P'$ with $TV(\P, \P') \le \ep$ is the set of all mixtures $\P'$ whose parameters $\{m_i', \mu_i', \bigsig_i'\}$ are $O(\de(\ep))$-close to those of the true distribution $\P$, up to permutation.
In particular, if $TV(\P, \P') \le \ep$, then for each individual Gaussian component $\N_{\mu_i,\bigsig_i}$ in $\P$, there exists a component $\N_{\mu_j',\bigsig_j'}$ in $\P'$ whose parameters are $O(\de(\ep))$-close to it, i.e., $\max\left\{|m_i - m_j|, \norm{\mu_i-\mu_j'}_{\infty}, \norm{\bigsig_i'-\bigsig_j'}_{\infty}\right\} \le O(\de(\ep))$.

We now argue that $\lim\limits_{\ep \ra 0} \ff{\de(\ep)}{\ep}$ must be a constant (i.e., that $\de(\ep)$ is $\Theta(\ep)$ as $\ep \ra 0$) in order for the total variation between the two mixtures to be $\le \ep$.
We do so by Taylor expanding a set of quantities whose magnitudes lower bound the total variation between $\P$ and $\P'$, and showing that these quantities are locally linear in the parameter differences between $\P'$ and $\P$ when these differences are sufficiently small.

By assumption,
$2 TV(\P, \P') = q({m}_1', \dots, {m}_k', {\mu}_1', \dots, {\mu}_k', {\bigsig}_1', \dots, {\bigsig}_k') = \\
\I{\R^d}{} \left| \su{i=1}{k} m_i \N_{\mu_i, \bigsig_i}(x) - \su{i=1}{k} m_i'\N_{\mu_i',\bigsig_i'}(x)\right| \dx \le 2\ep$.
Notice that
$2 TV(\P, \P') = \\
\sup\limits_{S \subseteq \mathfrak{M}^d} \I{S}{} \left| \su{i=1}{k} m_i \N_{\mu_i, \bigsig_i}(x) - \su{i=1}{k} m_i'\N_{\mu_i',\bigsig_i'}(x)\right| \dx
\ge
\sup\limits_{S \subseteq  \mathfrak{M}^d} \left| \I{S}{} \lt \su{i=1}{k} m_i \N_{\mu_i, \bigsig_i}(x) - \su{i=1}{k} m_i'\N_{\mu_i',\bigsig_i'}(x) \rt \dx \right|$, \\ where $ \mathfrak{M}^d$ denotes the collection of all measurable subsets of $\R^d$.

\noindent Suppose first that $d = 1$ (so $\bigsig_i = \sigma_i^2$, a scalar). Then, $\sup\limits_{S \subseteq \mathfrak{M}}  \left| \I{S}{} \lt \su{i=1}{k} m_i \N_{\mu_i, \sigma_i^2}(x) - \su{i=1}{k} m_i'\N_{\mu_i',\sigma_i^{2\prime}}(x) \rt \dx \right|
 \\
\ge \sup\limits_{c_j \in \R} \left|\I{-\infty}{c_j} \lt \su{i=1}{k} m_i \N_{\mu_i, \sigma_i^2}(x) - \su{i=1}{k} m_i'\N_{\mu_i',\sigma_i^{2\prime}}(x) \rt  \dx \right|
=
\sup\limits_{c_j \in \R} |h(\P', c_j)|$
where we define
$h(\P'; c_j) =  \\ h(m_1', \dots, m_k', \mu_1', \dots, \mu_k', \sigma_1^{2\prime}, \dots, \sigma_k^{2\prime}; c_j) \defeq \su{i=1}{k} m_i \I{-\infty}{c_j} \N_{\mu_i, \sigma_i^2}(x)\, \dx - \su{i=1}{k} m_i' \I{-\infty}{c_j} \N_{\mu_i',\sigma_i^{2\prime}}(x) \, \dx$.
So, $|h(\P'; c_j)|$ is a lower bound on twice the total variation between $\P$ and $\P'$, for any value of $c_j$.
Let $\vec{v} \in \R^{3k}$ denote the vector of parameters $(m_1,\dots,m_k,\mu_1,\dots,\mu_k,\sigma_1^2,\dots,\sigma_k^2)$,
and similarly for $\vec{v}'$.
For ease of notation we write $h(\vec{v}'; c_j)$ interchangeably with both $h(\P'; c_j)$ and $h(m_1', \dots, m_k', \mu_1', \dots, \mu_k', \sigma_1^{2\prime}, \dots, \sigma_k^{2\prime}; c_j)$.

Assume $TV(\P, \P') \le \ep$, so as argued before $\norm{\vec{v}' - \pi(\vec{v})}_2$ is $O(\de(\ep))$ for some permutation $\pi$ and some function $\de$ with $\lim\limits_{x \ra 0} \de(x) = 0$. [More precisely, $\lim\limits_{\ep \ra 0} \,\,\max\limits_{\vec{v}' : TV(\P, \P') \le \ep} \,\, \min\limits_{\pi : \pi(k)} \norm{\vec{v}'-\pi(\vec{v})}_2 = 0$.] Without loss of generality, we will henceforth simply write $\vec{v}$ in place of $\pi(\vec{v})$. For notational simplicity, let's write $\de \defeq \norm{\vec{v}' - \vec{v}}_2$. As $h$ is smooth around $\vec{v}$, we can Taylor expand $h$ about the point $\vec{v} = \vec{v}'$ [i.e., $\{m_1' = m_1, ..., \mu_1' = \mu_1,  ..., \sigma_1^{2\prime} = \sigma_1^2, ..., \sigma_k^{2\prime} = \sigma_k^2\}$] to get
\newcommand{\erfc}{\text{erfc}}
\begin{align*}
h(m_1', \dots, m_k', \mu_1', \dots, \mu_k', \sigma_1^{2\prime}, \dots, \sigma_k^{2\prime}; c_j) = h(\vec{v}') &= \\
h(\vec{v}; c_j) + \G h(\vec{v}; c_j)^T (\vec{v}' - \vec{v}) + O(\norm{\vec{v}' - \vec{v}}_2^2)~.
\end{align*}
Note that $h(\vec{v}; c_j) = 0$ (since $\vec{v}$ are the true parameters). \newline
$\G h(\vec{v}; c_j)^T = (\ff{\p}{\p m_1'} h(\vec{v}; c_j), \dots, \ff{\p}{\p \mu_1'} h(\vec{v}; c_j), \dots, \ff{\p}{\p \sigma_1'} h(\vec{v}; c_j), \dots)$.
\newline
$\f{\p}{\p m_i'} h(\vec{v}; c_j) = \I{-\infty}{c_j} p_{\mu_i, \sigma_i^2}(x) \, \dx = \ff{1}{2} \erfc\lt \ff{\mu_i-c_j}{\sqrt{2}\sigma_i}\rt$. [$\sigma_i$ denotes the positive root of $\sigma_i^2$.]
\newline
$\f{\p}{\p \mu_i'} h(\vec{v}; c_j) = -\f{m_i e^{-(c_j-\mu_i)^2 / (2\sigma_i^2)}}{\sqrt{2\pi}\sigma_i}$.
\newline
$\f{\p}{\p \sigma_i^{2\prime}} h(\vec{v}; c_j) = \f{m_i e^{-(c_j-\mu_i)^2 / (2\sigma_i^2)}(\mu_i- c_j)}{2\sqrt{2\pi}\sigma_i^3}$.

Define $f_i(x) = \ff{1}{2} \erfc\lt \ff{\mu_i-x}{\sqrt{2}\sigma_i}\rt$ for $1 \le i \le k$, $-\ff{m_i e^{-(x-\mu_i)^2 / (2\sigma_i^2)}}{\sqrt{2\pi}\sigma_i}$ for $k+1 \le i \le 2k$, and $\ff{m_i e^{-(x-\mu_i)^2 / (2\sigma_i^2)}(\mu_i- x)}{2\sqrt{2\pi}\sigma_i^3}$ for $2k+1 \le i \le 3k$. So $\G h(\vec{v}; c_j)^T = (f_1(c_j), \dots, f_{3k}(c_j))^T$.

We have $|\G h(\vec{v}; c_j)^T\de + O(\norm{\de}_2^2)| \le 2\ep$ for all $c_j \in \R$.
Now, we claim that it is possible to select $c_1, \dots, c_{3k} \in \R$ such that the $3k \times 3k$ matrix with rows $\G h(\vec{v}; c_j)^T$ for $1 \le j \le 3k$ is nonsingular.

\emph{Proof}:
Suppose that there is a nonzero vector $\vec{w} \in \R^{3k}$ such that $(f_1(x), \dots, f_{3k}(x))^T w = 0$ for all $x \in \R$. This means that the function $w_1 f_1(x) + \dots + w_{3k} f_{3k}(x)$ is identically 0. But this is impossible unless $\vec{w} = 0$, as the $f_i$'s form a linearly independent set of functions (which can easily be seen by looking at their asymptotic behavior). Thus, there is no nonzero vector that is orthogonal to every vector in the set $\bigcup\limits_{x\in\R} \{(f_1(x), \dots, f_{3k}(x))\}$. In particular, this means that we can find $c_1, \dots, c_{3k} \in \R$ such that the matrix with rows $(f_1(c_j), \dots, f_{3k}(c_j))$ is nonsingular (i.e., has linearly independent rows).

\newcommand{\biga}{\mathbf{A}}
Call this matrix $\biga$. So $\norm{\biga\de + \eta}_\infty \le 2\ep = 2\ep \norm{\mathbbm{1}}_\infty$ where $\eta \in \R^{3k}$ is such that $\norm{\eta}_{\infty}$ is $O(\norm{\de}_2^2) = O(\norm{\de}_\infty^2)$, and $\mathbbm{1}$ denotes the vector of all ones in $\R^{3k}$. So
$\norm{\de}_\infty - \norm{\biga^{-1}\eta}_\infty \le \norm{\de + \biga^{-1}\eta}_{\infty} \le 2\ep \norm{\biga^{-1}}_\infty \norm{\mathbbm{1}}_\infty = 2\ep\norm{\biga^{-1}}_{\infty}$,
where $\norm{\biga^{-1}}_{\infty}$ is the induced $\infty$-norm of $\biga^{-1}$, and the first inequality follows from the reverse triangle inequality.
Note that $\norm{\biga^{-1}\eta}_\infty \le \norm{\biga^{-1}}_\infty \norm{\eta}_\infty \le O(\norm{\de}_\infty^2)$ since $\biga^{-1}$ is defined independently of $\de$.

So, $\norm{\de}_\infty - O(\norm{\de}_\infty^2) \le 2\ep \norm{\biga^{-1}}_{\infty}$, and thus $\norm{\de}_\infty$ [which is, by definition, the maximum error in any parameter $m_1,\dots,m_k,\mu_1,\dots,\mu_k,\sigma_1^2,\dots,\sigma_k^2$ up to permutation] is $O(\ep)$.
But then, the total variation between each pair of mixture components $\N_{\mu_i,\sigma_i^2}$ and $\N_{\mu_i',\sigma_i^{2\prime}}$ is also $O(\ep)$, by Lemma~\ref{lem:tv_simple} and norm equivalence.

Thus, when $d = 1$,
if the total variation between the two Gaussian mixtures $\P$ and $\hat{\P}$ is $O(\ep)$, the total variation between each mixture component must also be $O(\ep)$
[where the big-O notation suppresses all parameters that depend on the true distribution $\P$], as desired. (Recall that total variation is always in $[0, 1]$.)

Now suppose $d > 1$. Similarly to before,
we have
\begin{align*}
2\ep &\ge 2 TV(\P, \P') 
\ge
\max\limits_{S \subseteq \mathfrak{M}^d}  \left| \I{S}{} \lt \su{i=1}{k} m_i \N_{\mu_i, \bigsig_i}(x) - \su{i=1}{k} m_i'\N_{\mu_i',\bigsig_i'}(x) \rt \dx \right|
\\ &\ge
\max\limits_{\vec{c}_j \in \R^d} \left|\I{-\infty}{c_{j1}} \I{-\infty}{c_{j2}} \dots \I{-\infty}{c_{jd}} \lt \su{i=1}{k} m_i \N_{\mu_i, \bigsig_i}(x) - \su{i=1}{k} m_i'\N_{\mu_i',\bigsig_i'}(x) \rt \dx \right|
\\ &=
\max\limits_{\vec{c}_j \in \R^d} |h(\P', \vec{c}_j)|,
\end{align*}
where we define $h(\P'; \vec{c}_j)$ as $\su{i=1}{k} m_i \I{-\infty}{c_{j1}} \I{-\infty}{c_{j2}} \dots \I{-\infty}{c_{jd}} \N_{\mu_i, \bigsig_i}(x) \, \dx - \su{i=1}{k} m_i' \I{-\infty}{c_{j1}} \I{-\infty}{c_{j2}} \dots \I{-\infty}{c_{jd}} \N_{\mu_i',\bigsig_i'}(x) \, \dx$.
Again, we equivalently denote this by $h(\vec{v}'; c_j)$, where $\vec{v}' \in \R^{k(d+1) + kd(d+1)/2} \defeq \\ (m_1', ..., \text{vec}(\mu_1'), ..., \text{vec}(\bigsig_1'), ...)$ denotes the parameters collected into a single vector.
As before, we Taylor expand about the point $\vec{v}' = \vec{v}$ to get that $|\G h(\vec{v}; \vec{c}_j)^T \de + O(\norm{\de}_\infty^2)| \le 2\ep$, where $\de \defeq \vec{v}'-\vec{v}$.
\noindent
Now, we compute the entries of $\G h(\vec{v}; \vec{c}_j)$: \newline
$\f{\p}{\p m_i'} h(\vec{v}; \vec{c}_j) = \I{-\infty}{c_{j1}} \I{-\infty}{c_{j2}} \dots \I{-\infty}{c_{jd}} \N_{\mu_i, \bigsig_i}(x) \, \dx = \\
\f{1}{(2\pi)^{d/2} \sqrt{|\det(\bigsig_i)|}} \I{-\infty}{c_{j1}} \I{-\infty}{c_{j2}} \dots \I{-\infty}{c_{jd}} e^{-(x-\mu_i)^T \bigsig_i^{-1} (x-\mu_i) /2} \, \dx$.
\newline
${\f{\p}{\p (\mu_i')_a} h(\vec{v}; \vec{c}_j) = \f{m_i}{(2\pi)^{d/2} \sqrt{|\det(\bigsig_i)|}} \I{-\infty}{c_{j1}} \I{-\infty}{c_{j2}} \dots \I{-\infty}{c_{jd}} (\bigsig_i^{-1})_a \cdot (x-\mu_i) e^{-(x-\mu_i)^T \bigsig_i^{-1} (x-\mu_i) /2} \, \dx}$, \\
where $(\mu_i)_a$ denotes the $a^{th}$ entry of the vector $\mu_i$ and $(\bigsig_i^{-1})_a$ is the $a^{th}$ row of the matrix $\bigsig_i^{-1}$.
\newline
$\f{\p}{\p (\bigsig_i')_{ab}} h(\vec{v}; \vec{c}_j) = -\f{m_i (\bigsig_i')_{ab}^{-1}}{2(2\pi)^{d/2} \sqrt{|\det(\bigsig_i)|}} \I{-\infty}{c_{j1}} \I{-\infty}{c_{j2}} \dots \I{-\infty}{c_{jd}} e^{-(x-\mu_i)^T \bigsig_i^{-1} (x-\mu_i) /2} \, \dx + \\
{\f{m_i}{(2\pi)^{d/2} \sqrt{|\det(\bigsig_i)|}} \I{-\infty}{c_{j1}} \I{-\infty}{c_{j2}} \dots \I{-\infty}{c_{jd}} [(\bigsig_i^{-1}(x-\mu_i)) (\bigsig_i^{-1}(x-\mu_i))^T]_{ab} \, e^{-(x-\mu_i)^T \bigsig_i^{-1} (x-\mu_i) /2} \, \dx}$,
where $M_{ab}$ \\ denotes the $(a,b)$ entry of the matrix $M$ (note that all matrices involved in this expression are symmetric).

Once again, this set of $k + kd + kd(d+1)/2$ partial derivatives, considered as functions of $\vec{c}_j$, comprise a linearly independent set of functions,
since the $(\mu_i, \bigsig_i)$ pairs are unique. The remainder of the proof proceeds analogously to the $d = 1$ case.
\end{proof}

\subsection{Subclass Performance Gaps Enable Distinguishing Between Subclasses}
\label{app:gap}
In this section, we give simple intuition for why a performance gap between two subclasses of a superclass implies that
it is possible to discriminate between the two subclasses in feature space to a certain extent.

Suppose the setting is binary classification,
and one of the superclasses has two subclasses with equal proportions in the dataset.
Suppose we have access to a model whose training accuracy on one subclass is $x$,
while its training accuracy on the other subclass is $y$, where $1 \ge x > y \ge 0$.

Of the correctly classified examples, $\ff{x}{x+y} > \ff{1}{2}$ fraction of them are from the first subclass;
similarly, of the incorrectly classified examples, $\ff{1-y}{2-x-y} > \ff{1}{2}$ fraction of them are from the second subclass.

This means that if we form ``proxy subclasses'' by simply splitting the superclass into the correctly classified training examples and incorrectly classified training examples, the resulting groups can in fact be a good approximation of the true subclasses!
This is illustrated in Figure~\ref{fig:accuracy-diff}.
For instance, suppose $x = 0.9$ and $y = 0.6$. Then $\ff{x}{x+y} = 0.6$ and $\ff{1-y}{2-x-y} = 0.8$ - so, 60\% of the examples in the first group are from subclass 1, and 80\% of those in the second group are from subclass 2, which is much better than randomly guessing the true subclasses (in which the concentration of each subclass in each guessed group will approach 50\% as $n \ra \infty$).
In the extreme case, if one subclass has accuracy $1$ and the other has accuracy $0$, then the superclass decision boundary separates them perfectly (no matter their proportions).

Combined with other information, this helps explain why looking at the way each example is classified (such as the loss of the example or related error metrics) can be helpful to discriminate between the subclasses.

\begin{figure}
\centering
\includegraphics[height=1.5in]{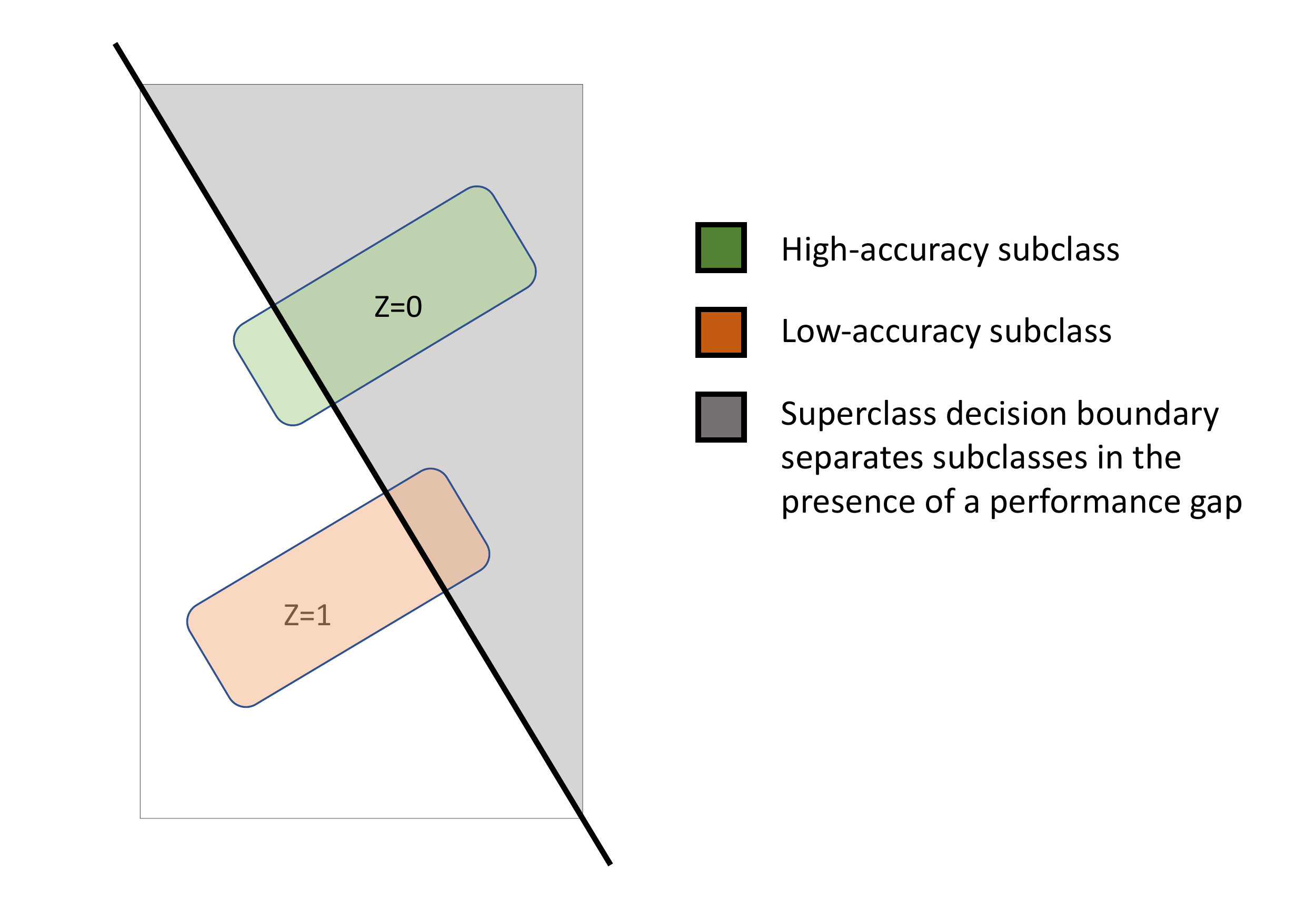}
\caption{
A performance gap between subclasses within the same superclass
implies a corresponding degree of separation in feature space.
Green and red are true subclasses for the superclass which the model predicts as the gray region;
the decision boundary for the superclass classification task
also approximately separates the subclasses.}

\label{fig:accuracy-diff}
\end{figure}

\subsection{Inherent Hardness}
We define the ``inherent hardness'' of a (task, function class) pair as the minimum attainable robust error, i.e.,
$$\argmax\limits_{f \in \mathcal{F}} \min\limits_{c \in \{1,\dots,C\}} \mathbb{E}_{(x,y) | z = c}\left[\1(f(x) = y) \right],$$
where the function class is denoted by $\mathcal{F}$. (This can be thought of as the ``Bayes robust risk.'')
We allow the function $f$ to be stochastic: i.e., for a given
input $x$, it may output a fixed probability distribution over the possible labels,
in which case we define $\1(f(x) = y)$ as the probability assigned by $f$ to the label $y$, given input $x$.
By definition, the inherent hardness lower bounds the robust error attained by any classifier in $\mathcal{F}$,
regardless of how it is trained or how much data is available.
The only way to improve robust performance is therefore
to either make the model class $\mathcal{F}$ more expressive (i.e., include more functions in $\mathcal{F}$)
or to collect new data such that the covariates $x$ include more information that can be used to distinguish between different classes.
(Of course, both of these changes would be expected to improve overall performance as well, if sufficient data is available.)
Thus, addressing hidden stratification effects caused by ``inherent hardness'' is beyond the scope of this work.
A simple example of an ``inherently hard'' task (i.e., a task with nonzero ``inherent hardness'') is shown in Figure~\ref{fig:hardness};
no classifier can get perfect accuracy on every subclass, because the two superclasses overlap
and thus it is impossible to distinguish between them in the region of overlap.
Nevertheless, it is possible to attain perfect accuracy on \emph{some} subclasses in this example,
meaning that there will still be performance gaps between the subclasses.

\begin{figure}
\centering
 \includegraphics[height=1.5in]{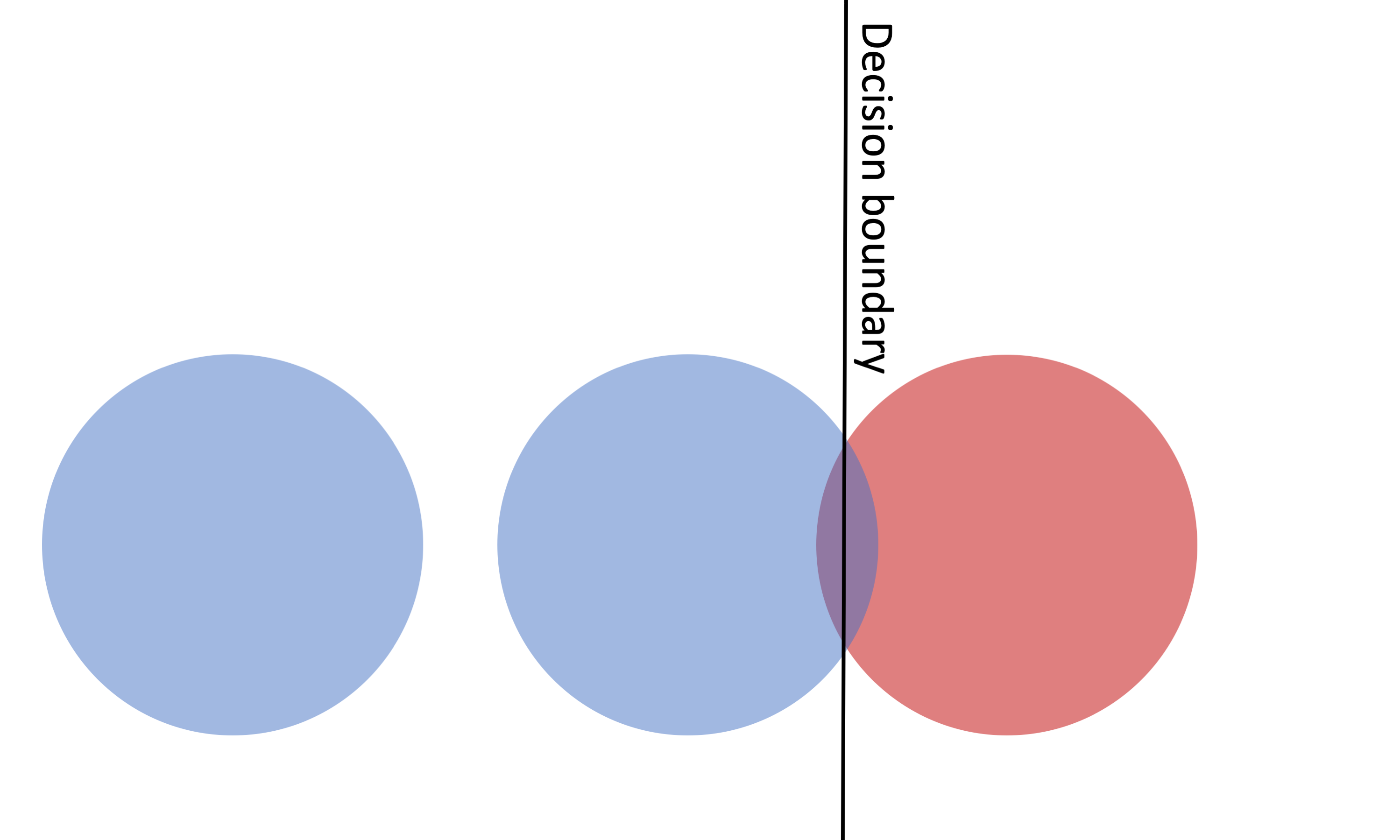}
\caption{``Inherent hardness'': the red and blue superclasses overlap,
making it impossible to distinguish between them with perfect overall accuracy.
The blue superclass has two subclasses; on the leftmost subclass,
the classifier can attain perfect accuracy.}
\label{fig:hardness}
\end{figure}

\subsection{GDRO with Soft Group Assignments}
\label{app:gdro-soft}
As shown above in Appendix \ref{app:analysis-proofs}, we can minimize $\max\limits_{c \in [C]} \E_{x \sim \hat{\P}_{S(c)}}[\hat{w}(x, c) \ell(x, S(c); \theta)]$ as a surrogate for
$\max\limits_{c \in [C]} \E_{x \sim {\P}_{c}}[\ell(x, S(c); \theta)]$. Here, $\hat{\P}_{S(c)}$ is the empirical distribution of training examples of superclass $c$ (with density $\hat{p}_{S(c)} = \hat{p}(x | y = S(c))$), $\hat{w}(x,c)$ is shorthand for $\f{\hat{p}(x | z = c)}{\hat{p}(x | y = S(c))}$,
and $\ell(x, S(c); \T)$ is shorthand for $\ell(f_\T(x), S(c))$, where $f_\T$ is a classifier parameterized by $\T$.
If we define the density $A_c(x, z) = \hat{p}_{S(c)}(x) \1(z = c)$, then $\max\limits_{c \in [C]} \E_{x \sim \hat{\P}_{S(c)}}[\hat{w}(x, c) \ell(x, S(c); \theta)] = \max\limits_{c \in [C]} \E_{(x,z) \sim A_c}[\hat{w}(x, z) \ell(x, S(z); \theta)]$.

If we now define $\tilde{\ell}(x,z; \theta) \defeq \hat{w}(x,z) \ell(x, S(z); \theta)$, we see that this falls directly within the group DRO framework of  \citen{sagawa2020distributionally}. We thus obtain Algorithm~\ref{alg:opt}, which is a minor modification of Algorithm 1 of \citen{sagawa2020distributionally}.

\begin{algorithm}[h]
\caption{Modified Group DRO}
\label{alg:opt}
\SetAlgoLined
\textbf{Input}: Step sizes $\eta_q, \eta_\theta$; empirical per-superclass distributions $\hat{\P}_b$ for each superclass $b\in [B]$ \\
 \nl Initialize $\theta\iter{0}$ and $q\iter{0}$ \\
 \nl \For{$t=1,\dots,T$}{
 \nl $ c \sim \text{Uniform}(1,\dots,c)$\\
 \nl $x \sim \hat{\P}_{S(c)}$\\
 \nl $ q^\prime \gets q\iter{t-1}$\\
 \nl $ q^\prime_c \gets q^\prime_c \cdot \exp \lt \eta_q \cdot \hat{w}(x,c) \cdot \ell(x,S(c); \theta\iter{t-1}) \rt$\\
 \nl $ q\iter{t} \gets q^\prime / \sum_c q^\prime_c$\\
 \nl $\theta\iter{t} \gets \theta\iter{t-1} - \eta_\theta \cdot q_c\iter{t} \cdot \hat{w}(x,c) \cdot \nabla_\theta \ell(x,S(c); \theta\iter{t-1})$\\
 }
\end{algorithm}

The weights $\hat{w}(x, c)$ correspond to ``soft labels'' indicating the probability a particular example came from a particular superclass; notice that that $\E_{(x,z) \sim A_c}[\hat{w}(x, z) \ell(x, S(z); \theta)]$ depends on every training example in the \emph{superclass} $S(c)$, so each training example is used in multiple terms in the maximization.

Finally, note that if the assumptions (informally: nonnegativity, convexity, Lipschitz continuity, and boundedness) of Proposition 2 in \citen{sagawa2020distributionally} hold for the modified loss $\tilde{l}(x, z; \theta)$, then the convergence guarantees carry over as well, since Algorithm \ref{alg:opt} is a specific instantiation of Algorithm 1 from \citen{sagawa2020distributionally}. So, under these assumptions, the convergence rate of Algorithm \ref{alg:opt} is $O(1/\sqrt{T})$, where $T$ is the number of iterations. [Specifically, the average iterate after $T$ iterations achieves a robust loss that is $O(1/\sqrt{T})$ greater than the minimum of the robust loss.]

In preliminary experiments, we found hard clustering to work better than the ``soft clustering'' approach described in this section; as it also has the advantage of simplicity, all final experiments were performed with hard clustering.

\clearpage 
\FloatBarrier

\end{document}